\documentclass[10pt]{article} %

\usepackage[preprint]{tmlr}

\usepackage{amsmath,amssymb,amsthm,mathtools}
\usepackage{bm}
\usepackage{bbm}
\usepackage{enumitem}
\usepackage{microtype}
\usepackage{booktabs}
\usepackage{adjustbox}
\usepackage{multirow}
\usepackage{subcaption}
\usepackage{threeparttable}
\usepackage{graphicx}
\usepackage{tikz}
\usepackage{pgfplots}
\pgfplotsset{compat=1.18}
\definecolor{C0}{HTML}{1B6CA8}
\definecolor{C1}{HTML}{E1701A}
\definecolor{C2}{HTML}{2C8B57}
\definecolor{C3}{HTML}{A02C2C}
\definecolor{C4}{HTML}{6A4C93}
\definecolor{C5}{HTML}{8C6D31}
\definecolor{C6}{HTML}{B5379B}
\definecolor{C7}{HTML}{444444}
\usepgfplotslibrary{groupplots}
\usepackage{xcolor}
\definecolor{correctioncolor}{rgb}{0,0,0} %

\newcommand{\correction}[1]{\textcolor{correctioncolor}{#1}}
\usepackage{lscape}
\usepackage{pifont}
\usetikzlibrary{arrows.meta, calc, positioning, fit,
  decorations.pathreplacing, shapes.geometric, patterns}
\usepackage{algorithm}
\usepackage[noend]{algpseudocode}
\usepackage{float}

\usepackage{hyperref}
\usepackage{url}
\usepackage[capitalize,nameinlink]{cleveref}

\newcommand{\codeurl}{\url{https://anonymous.4open.science/r/dp-rfedproj-83B0/}}
\crefname{assumption}{assumption}{assumptions}
\Crefname{assumption}{Assumption}{Assumptions}

\definecolor{crdp}{HTML}{0077BB}
\definecolor{chuang}{HTML}{CC3311}
\definecolor{chybrid}{HTML}{009988}
\definecolor{crho1}{HTML}{EE6677}
\definecolor{crho2}{HTML}{228833}
\definecolor{crho5}{HTML}{4477AA}
\definecolor{crho10}{HTML}{CCBB44}
\definecolor{crho20}{HTML}{AA3377}

\theoremstyle{plain}
\newtheorem{theorem}{Theorem}
\newtheorem{lemma}[theorem]{Lemma}
\newtheorem{corollary}[theorem]{Corollary}
\newtheorem{proposition}[theorem]{Proposition}
\theoremstyle{definition}
\newtheorem{definition}[theorem]{Definition}
\newtheorem{assumption}{Assumption}
\theoremstyle{remark}
\newtheorem{remark}[theorem]{Remark}

\newcommand{\R}{\mathbb{R}}

\newcommand{\M}{\mathcal{M}}
\newcommand{\E}{\mathbb{E}}

\DeclareMathOperator{\grad}{grad}
\DeclareMathOperator{\conv}{conv}
\DeclareMathOperator{\sens}{sens}
\newcommand{\opnorm}[1]{\|#1\|_{\mathrm{op}}}
\newcommand{\norm}[1]{\left\lVert #1 \right\rVert}

\newcommand{\dimE}{D}
\newcommand{\sdp}{\sigma_{\mathrm{dp}}}
\newcommand{\clip}{\mathrm{clip}}
\newcommand{\N}{\mathcal{N}}
\newcommand{\Geff}{G_{\mathrm{eff}}}
\newcommand{\LM}{L_{\mathcal{M}}}

\newcommand{\Ldp}{L_{P}}

\newcommand{\Sym}{\mathrm{Sym}}
\newcommand{\dM}{d_{\mathcal{M}}}

\newcommand{\Idp}{I_{D}}

\newcommand{\trace}{\mathrm{tr}}
\newcommand{\gradC}{\grad^{C}\!}            %
\newcommand{\hgradC}{\widehat{\grad}^{C}\!} %
\newcommand{\Cqr}{C_{\mathrm{qr}}}
\newcommand{\EtG}{\E_t^{\mathcal{G}}}       %

\newcommand{\sighet}{\sigma_{\mathrm{het}}}
\pgfplotsset{table/search path={figures}}

\usetikzlibrary{arrows.meta, calc, positioning, fit, backgrounds}

\tikzset{
  netbox/.style={
    draw=black, semithick, rounded corners=2pt, align=center, fill=white,
    font=\scriptsize, minimum height=9mm, minimum width=15mm, inner sep=3pt
  },
  netdata/.style={netbox, dashed},
  netarrow/.style={-{Latex[length=2mm]}, thick, draw=black},
  netshapelabel/.style={midway, above, font=\tiny, yshift=5mm},
  netnote/.style={font=\scriptsize, align=left},
}

\newcommand{\EEGNetGeneralFig}{%
\begin{tikzpicture}[node distance=9mm and 11mm]
  \node[netdata] (inp) {Input\\\textbf{Trial}};
  \node[netbox, right=of inp] (block1)
    {\textbf{Block 1}\\Temporal \&\\spatial conv.};
  \node[netbox, right=of block1] (block2)
    {\textbf{Block 2}\\Separable\\convolution};
  \node[netbox, right=of block2] (head)
    {\textbf{Head}\\flatten\\$\to$ linear};
  \node[netdata, right=of head] (out) {Output\\$\hat{y}$};

  \draw[netarrow] (inp) -- node[netshapelabel]
    {$(B,1,n_{\text{chan}},\mathcal{T})$} (block1);
  \draw[netarrow] (block1) -- node[netshapelabel]
    {$(B,F_1\!\times\! D,1,\mathcal{T}/p_1)$} (block2);
  \draw[netarrow] (block2) -- node[netshapelabel]
    {$(B,F_2,1,\mathcal{T}/(p_1p_2))$} (head);
  \draw[netarrow] (head) -- node[netshapelabel]
    {$(B,n_{\text{cl}})$} (out);

  \node[netnote, below=7mm of block2] {%
    $F_1{=}8$ (temporal filters),
    $D{=}2$ (depth multiplier),
    $F_2{=}F_1\!\times\! D{=}16$,
    $p_1{=}4$, $p_2{=}8$ (pooling factors).\\
    $K_t=\max(\lfloor 0.5 f_s \rceil,32)$ (temporal, ${\approx}500$\,ms),
    $K_s=\max(\lfloor f_s/(2p_1) \rceil,8)$.%
  };
\end{tikzpicture}%
}

\newcommand{\EEGNetFigureWithNotes}{%
\begin{figure}[h]
  \centering
  \adjustbox{max width=\linewidth}{\EEGNetGeneralFig}
\caption{EEGNet pipeline. Shapes above each arrow, with $B$ the batch size and
$\mathcal{T}$ the number of time samples.}
  \label{fig:eegnet}
\end{figure}
}

\newcommand{\SPDNetDiagram}{%
\begin{tikzpicture}[node distance=9mm and 11mm,
  box/.style={netbox}, data/.style={netdata}]
  \node[data] (input) {Input\\[2pt]\textbf{S.C.M.}};
  \node[box, right=of input] (bimap) {\textbf{BiMap}\\$W^{\top}\!\cdot W$};
  \node[box, right=of bimap] (re) {\textbf{ReEig}\\$(\varepsilon_{\mathrm{R}})$};
  \node[box, right=of re] (logeig) {\textbf{LogEig}\\$\log(\cdot)$};
  \node[box, right=of logeig] (vec) {\textbf{Vech}\\half-vec.};
  \node[box, right=of vec] (fc) {\textbf{FC}};
  \node[data, right=of fc] (output) {Output\\$\hat{y}$};

  \draw[netarrow] (input) -- node[netshapelabel]
    {$(B,n_{\text{chan}},n_{\text{chan}})$} (bimap);
  \draw[netarrow] (bimap) -- node[netshapelabel] {$(B,d,d)$} (re);
  \draw[netarrow] (re) -- node[netshapelabel] {$(B,d,d)$} (logeig);
  \draw[netarrow] (logeig) -- node[netshapelabel] {$(B,d,d)$} (vec);
  \draw[netarrow] (vec) -- node[netshapelabel] {$(B,d(d{+}1)/2)$} (fc);
  \draw[netarrow] (fc) -- node[netshapelabel] {$(B,n_{\text{cl}})$} (output);
\end{tikzpicture}%
}

\newcommand{\SPDNetFigureWithNotes}{%
\begin{figure}[h]
  \centering
  \adjustbox{max width=\linewidth}{\SPDNetDiagram}
\caption{SPDNet pipeline. Shapes above each arrow, with $B$ the batch size and $d$ the
dimension after BiMap.}
  \label{fig:spdnet}
\end{figure}
}

\begin{document}

\title{Projection-Based Federated Learning on Compact Submanifolds with Geometry-Independent Differential Privacy}

\author{%
  \name Thibault Pautrel\thanks{Corresponding author.} \email thibault.pautrel@centralesupelec.fr \\
  \addr Universit\'e Paris-Saclay, CentraleSup\'elec, L2S\\
        Gif-sur-Yvette, France
  \AND
  \name Florent Bouchard \email florent.bouchard@cnrs.fr \\
  \addr Universit\'e Paris-Saclay, CNRS, CentraleSup\'elec, L2S\\
        Gif-sur-Yvette, France
  \AND
  \name Guillaume Ginolhac \email guillaume.ginolhac@univ-smb.fr \\
  \addr Universit\'e Savoie Mont Blanc, LISTIC\\
        Annecy, France
  \AND
  \name Ammar Mian \email ammar.mian@univ-smb.fr \\
  \addr Universit\'e Savoie Mont Blanc, LISTIC\\
        Annecy, France
    \AND
  \name Aur\'elien Bellet \email aurelien.bellet@inria.fr \\
  \addr PreMeDICaL team, Inria, Idesp, Inserm, Universit\'e de Montpellier\\
        Montpellier, France
}

\maketitle

\begin{abstract}
Federated learning trains a shared model across clients without exchanging their data,
yet the updates they transmit still leak information about individual records.
Enforcing differential privacy is harder still when the model parameters lie on a
Riemannian manifold. Existing methods rebuild the noise added for privacy in the geometry of each new
iterate, aggregate through intrinsic maps that often lack a closed form, and prove
convergence only in restricted regimes. We introduce \textsc{DP-RFedProj}, a record-level
differentially private federated learning method for compact submanifolds that
aggregates through closed-form projections and privatises gradients in a fixed ambient
geometry. Each client may run any stochastic first-order optimiser, take several local
steps, and participate only intermittently. An explicit record-level $f$-DP guarantee holds for the entire training
transcript, independently of the manifold, the aggregation scheme and the local
optimiser. We establish convergence of DP-RFedProj when each client runs Riemannian SGD locally and the server aggregates by projected averaging, and show that privacy enters the rate through a term scaling with the manifold’s intrinsic dimension rather than the ambient one. Synthetic experiments on federated PCA over the Grassmann manifold
corroborate these findings. On EEG motor-imagery classification with
SPDNet, whose weights lie on the Stiefel manifold, the method outperforms a Euclidean
baseline at every privacy budget considered. In both experimental settings, our accounting requires five to six times less noise than
that of the closest prior method for the same budget, which translates into higher utility.
\end{abstract}
\section{Introduction}
\subsection{Setting and motivation}

Federated learning (FL) enables a set of clients, each holding its own local data, to
collaboratively train a shared model under the coordination of a central server,
without exchanging raw data~\citep{mcmahan2017communication,konevcny2016federated}. It
is well suited to settings where data are scarce or subject-specific and where
centralised pooling is precluded by confidentiality or regulatory constraints, as in
medical imaging~\citep{sheller2020federated} or satellite
imagery~\citep{ben2024federated}. The standard protocol proceeds over a sequence of communication rounds. In each
round, the server broadcasts the current global model to the clients, each client
trains it locally on its own data, and the resulting updates are transmitted back
and aggregated into a new global model that seeds the next round. When the updates
are Euclidean parameters, the canonical aggregation scheme is
FedAvg~\citep{mcmahan2017communication}, which forms the next global model by
arithmetically averaging the client updates.

In a range of applications, however, the model parameters are constrained to a Riemannian
manifold. Such structure for instance arises when signals are classified through their symmetric
positive definite (SPD) spatial covariance matrices, as in EEG brain--computer interfaces
(BCI)~\citep{barachant2012multiclass} or image recognition~\citep{li2017second}, and
it underpins architectures such as SPDNet~\citep{huang2017riemannian,brooks2019riemannian},
whose orthonormal weight layers lie on the Stiefel manifold. Riemannian federated learning
adapts the protocol to this geometry: clients optimise locally with Riemannian gradient
methods, and the server aggregates parameters constrained to the manifold. Aggregation is
the central difficulty. The arithmetic mean of the client parameters may fall off the manifold, and even when it stays on it, it can ignore the underlying geometry. The
natural replacement is the Riemannian (Fr\'echet) mean, which rarely admits a closed form
and is instead approximated through the manifold's geometric maps, as in the classical
Karcher flow~\citep{karcher2014riemannian}. 

Federated learning is often motivated by privacy, but keeping data on-device does not by itself protect it. We consider an honest-but-curious server: it follows the training
protocol but observes all messages exchanged by the clients. The updates a client transmits are
computed from its records and leak information about them: reconstruction attacks recover
training examples from shared gradients~\citep{zhu2019deep}, and membership-inference
attacks reveal whether a particular record was used~\citep{shokri2017membership}.
Randomisation severs that link, since adding calibrated noise makes the output nearly as likely with or without any single record.  Differential privacy
(DP)~\citep{dwork2006calibrating,dwork2014algorithmic} makes this precise, bounding how
much the output distribution can change when a single record is altered. It comes at two
granularities: record-level DP makes any single training example undetectable, while
client-level DP conceals a client's entire dataset and thus whether that client took part
at all~\citep{geyer2017differentially,mcmahan2018general}. \correction{We target the former.
A record enters each local update through the gradient contributions it induces, and the
standard approach to protecting it is to bound these contributions by clipping before
masking them with noise.}

To certify privacy, one must ensure that an adversary cannot tell whether any particular
record was used, even after observing all the messages exchanged throughout training.
In the Euclidean setting, privacy accounting has tightened across successive frameworks:
from the original $(\varepsilon,\delta)$-DP bound~\citep{dwork2014algorithmic}, loose under
composition, through its R\'enyi-DP refinement~\citep{mironov2017renyi,balle2018privacy},
which composes sharply but needs care under subsampling, to the $f$-DP / Gaussian-DP
calculus of \citet{dong2022gaussian}, which tightens both in closed form.

\correction{On a Riemannian manifold, the geometry complicates both steps. The norm used for
clipping generally varies from point to point, and adding a noise vector to a point
may not return a point of the manifold.} A natural way to preserve the required vector-space
structure is therefore to perturb gradients in the tangent space. \citet{han2024differentially}
formally introduce this tangent-space Gaussian mechanism and privatise Riemannian SGD
(DP-RSGD), while \citet{utpala2023improved} add a variance-reduced variant (DP-RSVRG),
later used by \citet{huang2026federated}. Existing Riemannian privacy guarantees
\citep{huang2026federated} rely on geometry-dependent sensitivity bounds and composition
of per-round guarantees. In contrast, \citet{zheng2021federated} bring the $f$-DP calculus
to the federated setting for record-level guarantees against a curious client or a
coalition of clients.

\subsection{Related work}
\correction{These challenges place our work at the intersection of federated learning, Riemannian geometry,
and differential privacy. We first review existing approaches to Riemannian federated
learning, then discuss projection-based methods on compact submanifolds, and finally
position our approach relative to the only prior work.}

On general manifolds, aggregation goes through the exponential and logarithmic maps and a
transport. \citet{li2022federated} gives the first Riemannian FL framework, RFedSVRG,
averaging by a Karcher-flow step with a variance-reduced local solver, later extended with second-order information~\citep{xiao2024riemannian,xiao2025riemannian}.
\citet{huang2024riemannian} instead aggregates the gradient streams accumulated along the
local trajectories (AGS), which needs only a retraction and a vector transport and admits
multiple local steps under partial participation, at the cost of an analysis specific to Riemannian SGD: aggregating gradients rather than iterates leaves no room for a different
local optimiser. For schemes that aggregate the clients' final iterates, no analysis
combines multiple local steps and partial participation, even without privacy. Neither addresses privacy, and both rely on maps that on Stiefel are
unavailable in closed form or recovered only by an iterative solve
(Table~\ref{tab:related}).

On compact submanifolds, the nearest-point projection offers an
alternative to general retractions, both for keeping local
iterates on the manifold and for aggregating them at the server.
\citet{zhang2024nonconvex} adopt it to support multiple local epochs, and
\citet{wang2025federated} carry the same construction to the zeroth-order setting,
substituting a gradient-free estimator for the Riemannian gradient. Neither, however,
addresses privacy, and both use a per-client control variate to mitigate client drift. This couples aggregation to the local
optimiser and confines their
analyses to full participation, since every variate must be refreshed each round. For SPDNet, \citet{pautrel2026fedspdnet} first proposed projection-based aggregation rules
for the Stiefel weights of its BiMap layers, without privacy, and \citet{pautrel2026personalized}
considered the same FedSPDNet setting in a personalised mode. We follow this projection-based line but add a
record-level privacy layer and drop the control variate, leaving aggregation a pure
function of the final iterates (Table~\ref{tab:related}).

Bringing privacy, federation, and manifold structure together,
PriRFed~\citep{huang2026federated} is, to our knowledge, the only prior method to combine
all three. It targets record-level privacy under an honest-but-curious server, but
on a general manifold: it aggregates through a retraction and its inverse, privatises with
the tangent-Gaussian mechanism of \citet{han2024differentially} in the underlying geometry, and analyses convergence
for the DP-RSGD and DP-RSVRG local solvers. The noise construction is rebuilt at every iterate, and the training run is accounted by
composing per-round $(\varepsilon,\delta)$ pairs. Convergence is established only when
multiple local steps and multiple clients are not combined. Aggregation, finally, needs a
retraction and its inverse in tandem, which on Stiefel is available in closed form in at
most one direction.
\begin{table*}[t]
\centering
\caption{Comparison of Riemannian federated methods. The convergence regime reports the
setting in which convergence is established, as a pair (local steps, participation), with (multi-step, partial) as the widest regime. Both PriRFed and our method are optimiser-agnostic by design. The Solver column
reports the optimiser under which each convergence result is proved.}
\label{tab:related}
\begin{threeparttable}
\footnotesize
\begin{tabular}{@{}lccp{4cm}cc@{}}
\toprule
Method & Manifold & Privacy & Convergence regime & Solver & Geom. ops. \\
\midrule
RFedSVRG~\citep{li2022federated}    & General & $\times$ & (multi-step, single-client) / (single-step, full)$^\dagger$ & RSVRG & Exp, Log, PT \\
AGS~\citep{huang2024riemannian}     & General & $\times$ & (multi-step, partial) & RSGD$^\star$ & R, VT \\
PriRFed~\citep{huang2026federated}  & General & $(\varepsilon,\delta)$-DP & (multi-step, single-client) / (single-step, partial)$^\dagger$ & RSGD, RSVRG$^\S$ & R, R$^{-1}$ \\
\citet{zhang2024nonconvex}          & Compact & $\times$ & (multi-step, full) & RSGD & $\Pi$ \\
\citet{wang2025federated}           & Compact & $\times$ & (multi-step, full) & ZO-RSGD & $\Pi$ \\
\textbf{Ours}                       & \textbf{Compact} & $\boldsymbol{f\text{-DP}}$ & \textbf{(multi-step, partial)} & RSGD$^\S$ & $\boldsymbol{\Pi}$ \\
\bottomrule
\end{tabular}
\begin{tablenotes}[flushleft]\footnotesize
\item[$\dagger$] The two regimes are proved separately and their combination is not covered. Multiple local steps are analysed only with a single participating client, and multi-client participation only with a single local step.
\item[$\star$] Structural: aggregating gradient streams precludes a general local optimiser.
\item[$\S$] Optimiser-agnostic framework outside the convergence analysis.
\item Exp/Log: exponential/logarithm; R/R$^{-1}$: retraction and inverse; PT/VT: parallel/vector transport; $\Pi$: nearest-point projection.
\end{tablenotes}
\end{threeparttable}
\end{table*}

\subsection{Contributions}
No existing method brings together projection-based aggregation, partial participation,
multiple local steps, and differential privacy. We close this gap with
\textsc{DP-RFedProj}, a record-level differentially private Riemannian federated learning
algorithm for compact submanifolds under the Euclidean-induced metric. Even without privacy, the method is the first iterate-aggregation scheme for Riemannian FL whose
convergence is established with multiple local steps and partial participation together.
Moreover, it is the second method after PriRFed to bring record-level differential privacy
to Riemannian FL, and the first with a guarantee that does not depend on the manifold.  Our contributions
are fourfold.
\begin{itemize}[leftmargin=*,itemsep=3pt]
\item \textbf{A projection-based, optimiser-agnostic private FL algorithm.} \correction{Closed-form
projections do the geometric work at both stages of a round: they keep each client's local
steps on the manifold, and they aggregate the returned iterates at the server. We use two
aggregation rules on that basis, both first proposed by~\citet{pautrel2026fedspdnet} for the
Stiefel weights of SPDNet, without privacy, and stated here on any compact submanifold: \textsc{ProjAvg}, which projects the ambient mean, and
\textsc{RLAvg}, which retracts the average of the tangent-lifted displacements at the
current iterate. Both fit a common template that also covers the prior Karcher-flow and retraction-based
rules, which differ from ours only in the submission and the server map
(Table~\ref{tab:agg-schemes}). Because private local training is decoupled from
aggregation, each client can pair any of them with any stochastic first-order optimiser
consuming minibatch gradients (SGD, momentum, Adam).}
\item \textbf{A geometry-free privacy guarantee in closed form.} We privatise each client's
    gradients through a fixed ambient geometry, drawing the noise ambiently and projecting
    it onto the tangent space at each iterate. The resulting guarantee is a closed-form
    expression in the $f$-DP/GDP calculus, with record-level subsampling amplification and
    no dependence on the number of participating clients. It holds across the template's
    aggregators and is independent of the manifold and of the local optimisers. Against the closest prior accounting, which composes per-round $(\varepsilon,\delta)$
pairs by advanced composition, it certifies the same budget with five to six times less noise
in the settings of our experiments
(Sections~\ref{sec:synth-exp} and~\ref{sec:num-exps}).
\item \textbf{Convergence under partial participation and multiple local steps.} Taking
    Riemannian SGD as the local solver, we give the first convergence analysis on compact
    submanifolds combining projection-based aggregation, partial participation, and
    multiple local steps under data heterogeneity, bypassing the logarithmic,
    exponential, and inverse-retraction maps entirely. The result is new even without
    privacy: setting $\sdp=0$ recovers the first such guarantee for an iterate-aggregation
        scheme. The analysis covers the clipped algorithm for any threshold, with the clipping
    bias as an explicit term. The leading privacy term enters the rate through the
    manifold's intrinsic dimension rather than the ambient one.

\item \textbf{Numerical experiments.} On federated PCA over the Grassmann manifold, where the analysis applies fully, the cost of
privacy scales with $\dM$ and $1/k$ as predicted, and our accounting reduces PriRFed's
optimality gap by a factor of $10$ to $27$. We then instantiate the framework on EEG motor-imagery
classification with SPDNet~\citep{huang2017riemannian}  on three MOABB datasets~\citep{chevallier2024largest}. We compare 
DP-RFedProj applied to SPDNet against the Euclidean network EEGNet~\citep{lawhern2018eegnet}, trained by standard private federated averaging. SPDNet outperforms EEGNet at every privacy budget and at full and half client participation. Among the
aggregation rules, our projection-based scheme matches the accuracy of the retraction-based one while aggregating an order of magnitude faster. Code reproducing all experiments is available at \codeurl.
\end{itemize}
Section~\ref{sec:prelim} fixes the background, Section~\ref{sec:setting}
presents the algorithm, Sections~\ref{sec:privacy} and~\ref{sec:convergence} give its privacy
guarantee and convergence analysis, Sections~\ref{sec:synth-exp} and ~\ref{sec:num-exps} report the synthetic and EEG experiments and Section~\ref{sec:conclusion} discusses limitations and perspectives. 

\section{Background}\label{sec:prelim}
\correction{Section~\ref{ssec:rie-geo} recalls the Riemannian geometry we use,
Section~\ref{ssec:fl-dp} the federated protocol, its threat model and differential
privacy. Sections~\ref{ssec:intrinsic} and \ref{ssec:limitations} then describe the existing PriRFed~\citep{huang2026federated} pipeline and its limitations that our
construction departs from. }
\subsection{Riemannian Geometry}\label{ssec:rie-geo}
This section recalls basic notions from Riemannian geometry and optimisation. We refer the reader to  \cite{absil2008optimization,boumal2023introduction} for further background.
\paragraph{Submanifolds and the induced metric.}
A smooth manifold $\M$ of dimension $\dM$, embedded in a Euclidean
space $(\mathcal{E},\langle\cdot,\cdot\rangle)$ of ambient dimension
$D$, is a subset locally diffeomorphic to
$\R^{\dM}$ via smoothly compatible charts. The tangent space
$T_x\M$ at $x\in\M$ is the $\dM$-dimensional subspace of $\mathcal{E}$
spanned by velocities of smooth curves through~$x$. Equipping each
tangent space with an inner product $g_x(\cdot,\cdot)$, varying smoothly
in $x$, makes $\M$ a Riemannian manifold. Among the possible
choices, the one induced by the ambient inner product,
\begin{equation}\label{eq:metric}
  g_x(U,V) \;:=\; \langle U,V\rangle,
  \qquad U,V\in T_x\M,\ x \in \M,
\end{equation}
makes $\M$ an embedded Riemannian submanifold. We write $\norm{\cdot}$ for the norm of
$(\mathcal{E},\langle\cdot,\cdot\rangle)$, $\norm{\cdot}_x=\sqrt{g_x(\cdot,\cdot)}$ for the
one induced on $T_x\M$, and $P_x:\mathcal{E}\to T_x\M$ for the orthogonal projector onto the
tangent space. We denote by $\conv(\M)\subseteq\mathcal{E}$ the convex hull of $\M$.
\paragraph{Euclidean and Riemannian gradients.}
For a smooth $f:\M\to\R$, the Riemannian gradient $\grad f(x)\in T_x\M$
is the unique tangent vector representing the differential through the chosen metric:
\[
  \mathrm{D}f(x)[\xi] = g_x(\grad f(x),\xi),
  \qquad \xi\in T_x\M.
\]
In the embedded setting, when $f$ extends to a smooth function on a neighbourhood
of $\M$ in $\mathcal{E}$ (also denoted $f$) with Euclidean gradient
$\nabla f(x)\in\mathcal{E}$, the Riemannian gradient is given by the orthogonal
projection of the ambient gradient onto $T_x\M$:
\(
  \grad f(x) = P_x\nabla f(x).
\)

\paragraph{Geodesics and exponential map.}
A geodesic is a curve that locally minimises arc length, i.e.\ the
Riemannian analogue of straight lines in curved spaces. The Riemannian
distance $d(x,y)$ is the length of a shortest geodesic connecting $x$
and $y$. The exponential map $\mathrm{Exp}_x:T_x\M\to\M$ sends
$\xi\in T_x\M$ to $\gamma(1)$, where $\gamma$ is the geodesic with
$\gamma(0)=x$ and $\dot\gamma(0)=\xi$. Its local inverse is the
logarithmic map $\mathrm{Log}_x: \M \to T_x\M$.

\paragraph{Retractions and projections.}
In Riemannian optimisation, retractions serve as computationally
efficient substitutes for the exponential
map. A smooth
mapping $\mathrm{R}_x:T_x\M\to\M$ is a retraction if
$\mathrm{R}_x(0_x)=x$ and $\mathrm{D}\mathrm{R}_x(0_x)[u]=u$ for all
$u\in T_x\M$, where $0_x$ denotes the zero vector of $T_x\M$. It thus
coincides with the exponential map to first order. On compact submanifolds, a natural choice is the
projection-based retraction~\citep{absil2012projection} defined
through a projection map $\Pi:\mathcal{E}\to\M$ via
$\mathrm{R}_x(u):=\Pi(x+u)$, $u\in T_x\M$. The canonical example of such
projection is the nearest-point map
\begin{equation}\label{eq:nearest-p-proj}
  \Pi(y) := \arg\min_{x\in\M}\tfrac12\|y-x\|^2.
\end{equation}
When the minimiser is not unique, $\Pi(y)$ denotes any of them.
\paragraph{Examples.}
\correction{The Stiefel manifold $\mathrm{St}(d,p)=\{X\in\R^{d\times p}:X^\top X=I_p\}$ and the Grassmann manifold of $p$-dimensional subspaces of $\R^d$ are compact submanifolds. The Stiefel manifold underlies the constrained layers of SPDNet~\citep{huang2017riemannian}, while its covariance inputs lie on the non-compact SPD manifold $\mathcal{S}^{++}_n$, typically equipped with the log-Euclidean~\citep{arsigny2007geometric} or affine-invariant~\citep{pennec2006riemannian} metric. See Appendix~\ref{app:manifolds} for more details.}

\subsection{Federated Learning and Differential Privacy}
\label{ssec:fl-dp}
Federated learning (FL) trains collaboratively a shared model across $n$ clients, coordinated by a central server~\citep{mcmahan2017communication}. Each client
$i\in[n]$ holds its own dataset
$\mathcal D_i:=\{z_{i,1},\dots,z_{i,m_i}\}\subset\mathcal Z$ of size $m_i$, drawn
from a client-specific distribution $\mu_i$ (possibly non-i.i.d.\ data).  In Riemannian FL, the goal is to minimise the global empirical risk
\begin{equation}\label{eq:fl-objective}
  \min_{x\in\M}\ \Bigl\{\,F(x):=\frac1n\sum_{i=1}^n f_i(x)\,\Bigr\},
  \qquad
  f_i(x):=\frac1{m_i}\sum_{\ell=1}^{m_i} f(x;z_{i,\ell}),
\end{equation}
where $\M$ is a Riemannian manifold  and $f(x;z)$ is a chosen loss function computed at model parameter $x\in\M$ on a data sample $z$. Here $f_i$ represents client $i$'s local
empirical risk associated with the loss $f$ on its dataset $\mathcal{D}_i$. For clarity we use uniform weights in~\eqref{eq:fl-objective}. A
sample-weighted objective (in which each $f_i$ is weighted by $m_i/\sum_jm_j$) can be  handled by the same arguments.

The key requirement is that raw data are never shared: to perform optimization \eqref{eq:fl-objective}, clients transmit only model updates computed locally, which the server
aggregates over $T$ communication rounds. Entering round $t$, the server holds the
global model parameter $x_t\in\M$ and samples a set $S_t\subset[n]$ of $k$ participating clients,
uniformly without replacement and independently of the data. The round has two stages:
\begin{enumerate}[label=(\roman*),nosep]
  \item \textbf{Local optimisation.} The server broadcasts the current global model $x_t$ to the $k$ selected clients.
Each selected client $i$ initialises $x_{t,0}^{(i)}=x_t$ and runs $\tau\ge1$
gradient-based steps: at step $j$, from the current local iterate $x_{t,j}^{(i)}$, it
draws a minibatch $\mathcal B_{t,j}^{(i)}\subset\mathcal D_i$ of size $B$ uniformly without replacement and forms
the corresponding stochastic Riemannian gradient of its local objective,
\begin{equation}\label{eq:sto-grad}
    g_{t,j}^{(i)}=\frac{1}{B}\sum_{\ell\in \mathcal{B}_{t,j}^{(i)}}\grad f(x_{t,j}^{(i)};z_{i,\ell}).
\end{equation}
The chosen gradient-based optimiser then maps this gradient to the next iterate
$x_{t,j+1}^{(i)}$, producing after $\tau$ steps the final local iterate
$x_{t,\tau}^{(i)}\in\M$. 
  \item \textbf{Aggregation.} \correction{An aggregation scheme specifies a client-side submission
$\Delta^{(i)}$ derived from $(x_t,x_{t,\tau}^{(i)})$ and a server-side rule $\textsc{Agg}$.
Each $i\in S_t$ sends $\Delta^{(i)}$, and the server sets
$x_{t+1}=\textsc{Agg}\bigl(x_t,\{\Delta^{(i)}\}_{i\in S_t}\bigr)$.
Both depend on the geometry; see Section~\ref{sssec:server-agg} for the schemes considered.}
\end{enumerate}
Although raw data never leave the client, the messages it sends do not keep them
safe: per-record gradients encode enough to reconstruct training
examples~\citep{zhu2019deep} or to infer whether a given record was
used~\citep{shokri2017membership}, and a submission accumulates those gradients
over the local steps. Minibatch sampling randomises the submission without privatising
it: conditionally on the records drawn, the update is still determined exactly.
\paragraph{Threat model.}  Every submission is thus leak-prone, and the server collects
them all, which makes it the natural adversary (see Figure~\ref{fig:threat-model}). Following the standard threat model for private federated
learning~\citep{kairouz2021distributed}, also used by PriRFed~\citep{huang2026federated}, we
assume the server to be an honest-but-curious (HBC) adversary: it follows the protocol
but seeks to infer clients' records from every message it receives, that is, from
the full transcript
\begin{equation}\label{eq:transcript}
  \mathcal V_T :=
  \bigl(\{(x_t,S_t,\{\Delta^{(i)}\}_{i\in S_t})\}_{t=0}^{T-1},\,x_T\bigr).
\end{equation}
We target record-level protection: no client's training example should be inferable from
$\mathcal V_T$.

\paragraph{Differential privacy.}
Meeting this goal requires randomisation, the mechanism at the heart of Differential
Privacy~\citep{dwork2014algorithmic}. Calibrated randomness makes each release nearly as
likely with or without any single record, so no given record can be identified from the
transcript. Calibrating it requires bounding how much one record can change what is
released. At the gradient level, a record contributes a single summand to the minibatch
average~\eqref{eq:sto-grad}, and capping that summand bounds its influence. We
therefore randomise each gradient before it is fed to an optimiser~\citep{abadi2016deep}.

Throughout, the local datasets are assumed pairwise disjoint, so that each record belongs
to exactly one client. This is the standing assumption under which record-level adjacency
is stated. 
\begin{definition}[Record-level adjacency]\label{def:adjacency}
  Two databases $D=(\mathcal D_1,\dots,\mathcal D_n)$ and
  $D'=(\mathcal D'_1,\dots,\mathcal D'_n)$ over the same cohort of clients are
  \emph{record-adjacent}, written $D\bowtie D'$, if $\mathcal D_i=\mathcal D'_i$
  for all but one client $i^\star$, whose dataset differs from
  $\mathcal D'_{i^\star}$ by a single substituted record (in particular
  $|\mathcal D_{i^\star}|=|\mathcal D'_{i^\star}|$).
\end{definition}
Substituting a record thus perturbs one local dataset and leaves the other $n-1$ untouched.

Differential privacy compares, across
such a substitution, the law of a \emph{randomised mechanism} $\mathcal A$: a map
sending a database $D$ to a random output in a space $\mathcal Y$. Here $\mathcal A$ is the whole training protocol and its output the
full transcript $\mathcal V_T$~\eqref{eq:transcript}.
\begin{definition}[$(\varepsilon,\delta)$-DP {\citep{dwork2014algorithmic}}]
\label{def:dp}
  A randomised mechanism $\mathcal A$ from databases to an output space
  $\mathcal Y$ is $(\varepsilon,\delta)$-differentially private if, for every
  adjacent pair $D\bowtie D'$ and every measurable $E\subseteq\mathcal Y$,
  \[
    \Pr\bigl(\mathcal A(D)\in E\bigr)
    \le e^{\varepsilon}\,\Pr\bigl(\mathcal A(D')\in E\bigr)+\delta ,
  \]
  the probability being over the randomness of $\mathcal A$. Here
  $\varepsilon\ge0$ bounds the worst-case privacy loss (smaller is stronger)
  and $\delta\ge 0$ allows for a residual failure probability.
\end{definition}
\noindent
Since each client releases a privatised gradient at every local step of every round,
certifying $\mathcal V_T$ means quantifying how the guarantee of a single such
release degrades once all of them are observed jointly. This is the accounting task of
Section~\ref{sec:privacy}.
\begin{figure}[t]
\centering
\begin{tikzpicture}[
    scale=0.7, transform shape, >=stealth, font=\scriptsize,
    sat/.style={draw, thick, circle, minimum size=1.2cm, align=center, fill=white},
    unsat/.style={draw, thick, dashed, circle, minimum size=1.2cm, align=center,
                  fill=white, text=black!55, draw=black!55},
    server/.style={draw, thick, rectangle, rounded corners=2pt, minimum width=3.2cm,
                   minimum height=1.25cm, fill=white, align=center},
    arrowS/.style={->, thick, black!55, shorten >=1pt, shorten <=1pt,
                   arrows={-Triangle[angle=45:2pt 3]}},
    arrowR/.style={->, thick, black, shorten >=1pt, shorten <=1pt,
                   arrows={-Triangle[angle=45:2pt 3]}},
    lbl/.style={inner sep=1.5pt, fill=white, fill opacity=0.85, text opacity=1}
  ]
  \node[server] (server) at (0,0)
    {\textbf{HBC Server}\\[2pt] $S_t$\\[1pt] $x_t \to x_{t+1}$};

  \node[unsat] (satj)  at (-6.0,-3.0) {\textbf{Client }$j$\\$\mathcal{D}_{j}$};
  \node[sat,fill=black!8]   (sat1)  at (-3.0,-3.0) {\textbf{Client }$i_1$\\$\mathcal{D}_{i_1}$};
  \node[sat,fill=black!8]   (sat2)  at ( 0.0,-3.0) {\textbf{Client }$i_2$\\$\mathcal{D}_{i_2}$};
  \node[sat,fill=black!8]   (sat3)  at ( 3.0,-3.0) {\textbf{Client }$i_k$\\$\mathcal{D}_{i_k}$};
  \node[unsat] (satj2) at ( 6.0,-3.0) {\textbf{Client }$j'$\\$\mathcal{D}_{j'}$};

  \draw[arrowS] ([yshift=2mm]server.west) to[out=195,in=125]
        node[lbl,pos=0.55,sloped,above] {$x_t$} (sat1.125);
  \draw[arrowR] (sat1.55) to[out=55,in=245]
        node[lbl,pos=0.45,sloped,above] {$\Delta^{(i_1)}$} ([xshift=-10mm]server.south);

  \draw[arrowS] ([xshift=-4mm]server.south) to[out=250,in=125]
        node[lbl,pos=0.5,left] {$x_t$} (sat2.125);
  \draw[arrowR] (sat2.55) to[out=55,in=290]
        node[lbl,pos=0.5,right] {$\Delta^{(i_2)}$} ([xshift=4mm]server.south);

  \draw[arrowS] ([yshift=2mm]server.east) to[out=-15,in=55]
        node[lbl,pos=0.55,sloped,above] {$x_t$} (sat3.55);
  \draw[arrowR] (sat3.125) to[out=125,in=295]
        node[lbl,pos=0.45,sloped,above] {$\Delta^{(i_k)}$} ([xshift=10mm]server.south);
\end{tikzpicture}
\caption{One round $t$ with Honest-But-Curious (HBC) server. The server draws a fixed-size active set
$S_t\subset[n]$, sends $x_t$ to the selected clients (shaded), collects their
submissions $\Delta^{(i)}$ from $\tau$ private local steps and aggregates them to form the next global iterate $x_{t+1}$. The HBC server sees 
$(x_t, S_t, \{\Delta^{(i)}\}_{i\in S_t})$.}
\label{fig:threat-model}
\end{figure}
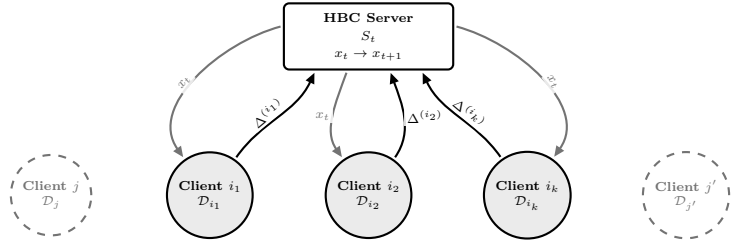
\subsection{Existing approach: PriRFed}\label{ssec:intrinsic}
The design below is the one instantiated by PriRFed~\citep{huang2026federated}: both the
privatisation and the aggregation are carried out in the manifold's own geometry. It is the
baseline our construction departs from.
\subsubsection{Local private optimisation}
\label{sssec:local-private-opt}
In the Euclidean setting~\citep{abadi2016deep}, gradient-level randomisation
combines two operations at each local step: (i) clipping each per-sample gradient
in $\ell_2$-norm, to bound any single record's influence,
then (ii) adding Gaussian noise to mask it. Neither transfers directly to a
Riemannian manifold, since the gradient at $x$ is not a free Euclidean vector but
an element of $T_x\M$. Both operations must therefore be anchored to the geometry
at the current iterate.

\paragraph{Clipping in the Riemannian norm.}
Given a norm $\norm{\cdot}_\bullet$ on a vector space and a threshold $C>0$, the clipping
operator rescales any vector $v$ to have norm at most $C$:
\begin{equation}\label{eq:clip-generic}
  \clip_C^{\norm{\cdot}_\bullet}(v):=v\cdot\min\!\bigl(1,\,C/\norm{v}_\bullet\bigr),
  \qquad \clip_C^{\norm{\cdot}_\bullet}(0):=0 .
\end{equation}
A tangent vector's magnitude is measured by the point-dependent $\norm{\cdot}_x$ rather than
by an ambient norm, so clipping is performed in $\norm{\cdot}_x$: the threshold $C$ then
bounds a record's influence as the geometry at $x$ measures it.

\paragraph{Tangent-space Gaussian noise.}
\correction{The noise is added to the clipped gradient, so it must lie in the same tangent space
$T_x\M$ and therefore depends on the current point $x$. PriRFed uses the
tangent-space Gaussian of \citet{han2024differentially}, denoted
$\N_x(0,\sdp^2)$, which is isotropic with respect to the Riemannian metric $g_x$.
Sampling it requires constructing a $g_x$-orthonormal basis of $T_x\M$ at every
iterate. \citet{utpala2023improved} avoid this cost by sampling the Gaussian at a
fixed reference point $\bar x$, where an orthonormal basis is available in closed
form, and transporting it to $T_x\M$ through a linear isometry
$\mathcal T_{\bar x\to x}$. Thus, per-step basis construction is replaced by a
per-step transport.}

\paragraph{Private gradient update.} Combining the two operations, at local step $j$ of round $t$ a selected client $i$
perturbs its clipped minibatch gradient with noise drawn afresh in the geometry of
the current iterate,
\begin{equation}\label{eq:han-grad}
  \tilde g_{t,j}^{(i)}
  := \tfrac1B\sum_{\ell\in\mathcal{B}_{t,j}^{(i)}}
     \clip_C^{\norm{\cdot}_{x_{t,j}^{(i)}}}\!\bigl(\grad f(x_{t,j}^{(i)};z_{i,\ell})\bigr)
   + \xi_{t,j}^{(i)}  \in T_{x_{t,j}^{(i)}}\M,
  \qquad
  \xi_{t,j}^{(i)}\sim\N_{x_{t,j}^{(i)}}\!\bigl(0,\sdp^2\bigr).
\end{equation}
This $\tilde g_{t,j}^{(i)}$ replaces the stochastic Riemannian gradient
$g_{t,j}^{(i)}$ of~\eqref{eq:sto-grad} as the optimiser's input, producing
$x_{t,j+1}^{(i)}$. The choice of optimiser is left open.

\subsubsection{Server aggregation on Riemannian manifolds}
\label{sssec:server-agg}
\correction{Privatisation being confined to the local step, aggregation carries no privacy role: the
server's task is purely geometric, to turn the submissions $\{\Delta^{(i)}\}_{i\in S_t}$
into a global iterate $x_{t+1}$ that is a meaningful average of the clients' final
iterates and lies on $\M$. In Euclidean FL~\citep{mcmahan2017communication} both
ingredients are trivial: the submission is the local iterate itself and \textsc{Agg} the
arithmetic mean, as in FedAvg. On a manifold neither survives, since the arithmetic mean
of points of $\M$ generally leaves $\M$ or carries no intrinsic meaning. Depending on the
scheme, the submission becomes a point of $\M$ or a tangent vector at $x_t$ }\footnote{This template requires $\Delta^{(i)}$ to depend only on
$(x_t,x_{t,\tau}^{(i)})$. It therefore excludes schemes folding a per-client drift
correction into the local update, such as \citet{zhang2024nonconvex}, where aggregation and
local optimiser are coupled.}.

On Riemannian manifolds, the natural substitute is the Fr\'echet mean (\textsc{FM}), which generalises the
arithmetic mean by minimising squared Riemannian rather than Euclidean
distance to the inputs. Each client submits its final iterate and the server returns
\[
  x_{t+1}^{\textsc{FM}}\;\gets\;\arg\min_{x\in\M}\frac1k\sum_{i\in S_t}
  d\bigl(x,\Delta^{(i)}\bigr)^2, \qquad \Delta^{(i)}:=x_{t,\tau}^{(i)}\in\M. 
\]
This minimiser generally has no closed form~\citep{pennec2006intrinsic} and is
computed iteratively by the Karcher flow~\citep{karcher2014riemannian}. Federated
aggregation keeps a single iteration: each client lifts its iterate to $T_{x_t}\M$ through the
logarithmic map, and the server averages these lifts and maps the result back to
the manifold via the exponential map~\citep{li2022federated},
\begin{equation}\label{eq:karcherflow}
  x_{t+1}^{\textsc{KFAvg}}
  \;=\; \mathrm{Exp}_{x_t}\!\left(\frac1k\sum_{i\in S_t}\Delta^{(i)}\right),
  \qquad \Delta^{(i)}=\mathrm{Log}_{x_t}\bigl(x_{t,\tau}^{(i)}\bigr)\in T_{x_t}\M.
\end{equation}
The logarithm has no closed form on many manifolds of practical interest, such as
Stiefel~\citep{zimmermann2022computing}. PriRFed~\citep{huang2026federated}
replaces $\mathrm{Exp}_{x_t}$ and $\mathrm{Log}_{x_t}$ by a retraction
$\mathrm{R}_{x_t}$ and its inverse $\mathrm{R}^{-1}_{x_t}$, so the submission
becomes the inverse-retraction lift in place of the logarithmic map:
\begin{equation}\label{eq:retavg}
  x_{t+1}^{\textsc{RetAvg}}
  \;=\; \mathrm{R}_{x_t}\!\left(\frac1k\sum_{i\in S_t}\Delta^{(i)}\right),
  \qquad
  \Delta^{(i)}=\mathrm{R}_{x_t}^{-1}\bigl(x_{t,\tau}^{(i)}\bigr)\in T_{x_t}\M.
\end{equation}
This substitution eases the previous difficulties but does not remove them entirely
(see Limitation~\ref{lim:geom-cost}).

\subsection{Limitations of the intrinsic pipeline}\label{ssec:limitations}
The design of Sections~\ref{sssec:local-private-opt}--\ref{sssec:server-agg},
instantiated by PriRFed~\citep{huang2026federated}, carries four limitations, which
the rest of the paper lifts.
\begin{enumerate}[label=(\arabic*),ref=\arabic*,leftmargin=*,itemsep=3pt]
\item\label{lim:moving-geometry}\textbf{A mechanism bound to a moving geometry.}
Both operations of~\eqref{eq:han-grad} are anchored at the current iterate, so the
noise construction must be redone at each of them, either via a fresh $g_x$-orthonormal basis,
or the transport of~\citet{utpala2023improved}, itself available only on manifolds
carrying a suitable reference point. Privatisation is thus inseparable from the
geometry.

\item\label{lim:accounting}\textbf{Loose accounting.} \correction{PriRFed builds its privacy guarantee in
three steps, with the first and third ones inflating the privacy budget:
\begin{enumerate}[label=(\alph*),leftmargin=*,itemsep=2pt,topsep=3pt]
  \item each client's local training is summarised by a single $(\varepsilon,\delta)$
    pair, which only partly expresses the protection brought by noise addition~\citep{gomez2026position};
  \item the $(\varepsilon,\delta)$ pairs of the clients active in a round are combined by parallel composition~\citep[Thms.~6 and~7]{huang2026federated};
  \item the $T$ per-round $(\varepsilon,\delta)$ pairs are combined by advanced composition~\citep[Thm.~7]{huang2026federated}, a step loose by construction which holds for the worst mechanism those pairs
    could describe.
\end{enumerate}
A separate limitation arises from client sampling: PriRFed benefits from privacy amplification
through client sampling~\citep[Def.~7, Lem.~1]{huang2026federated}, which requires the
active set to remain hidden from the adversary. Here, however, the server selects the
clients and observes the active set (Remark~\ref{rem:hidden-indices}). Section~\ref{sec:privacy}
therefore re-analyses the same sampling mechanism using $f$-DP, with amplification only at
the record level.
}

\item\label{lim:geom-cost}\textbf{Cost of the geometric primitives.} \textsc{RetAvg}
uses a retraction and its inverse in tandem, and on Stiefel no retraction supplies both
directions in closed form: the polar and QR retractions invert only by solving a matrix
equation, whereas the orthographic retraction inverts in closed form but has an
implicitly defined, Riccati-based forward
map~\citep{bouchard2025beyond,kaneko2012empirical}. An iterative solve is therefore
paid once per submission, for every participating client and every round.

\item\label{lim:agg-reach}\textbf{Reach of the convergence guarantees.} Existing
analyses of iterate aggregation cover multiple local steps only when one client
participates, and several clients only when each takes one
step~\citep{li2022federated,huang2026federated} but never both, the regime federated
training actually runs in. This holds with and without privacy. AGS~\citep{huang2024riemannian}
covers both, but aggregates gradient streams and is thus tied to Riemannian SGD. The lifting map is again
the obstruction: it is anchored at a base point that moves at every local step, so a
client's successive increments are never directly comparable, and reconciling them
drags the curvature of the manifold into the constants. The multi-step results moreover
assume the data identically distributed across clients, setting heterogeneity aside.
\end{enumerate}

\section{Projection-Based Private Federated Learning on Compact Submanifolds}\label{sec:setting}

\subsection{Structural properties} Lifting the limitations of
Section~\ref{ssec:limitations} costs generality: where PriRFed assumes a general
Riemannian manifold, we assume two structural properties.
\begin{enumerate}[label=(P\arabic*),ref=P\arabic*,leftmargin=*,itemsep=2pt]
  \item\label{sp:metric} \emph{Euclidean-induced metric}~\eqref{eq:metric}: the base-point
  norm $\norm{\cdot}_x$ is the restriction of the ambient norm $\norm{\cdot}$ to $T_x\M$,
  so gradient formation, clipping, and noise sampling reduce to applying the tangent
  projector $P_x$ to ambient quantities.
  \item\label{sp:compact} \emph{Compactness of $\M$}: it guarantees a manifold projection
    $\Pi:\mathcal{E}\to\M$, which returns updates to $\M$ through
    the retraction $\mathrm{R}_x(u)=\Pi(x+u)$.
\end{enumerate}
Each property allows us to lift a distinct limitation.
Property~\ref{sp:metric} detaches the mechanism from the geometry
(Limitation~\ref{lim:moving-geometry}): clipping and noise act on iterate-independent
ambient quantities, so no per-iterate construction is needed.
Property~\ref{sp:compact} supplies the projection the algorithm is built on. It replaces
the retraction and its inverse by the single map $\Pi$, applied at both the local and the
aggregation stage, and the nonlinear lift by the linear tangent projection
(Limitation~\ref{lim:geom-cost}). Specialising to the nearest-point
projection~\eqref{eq:nearest-p-proj} gives a map with controlled regularity in a tube
around the manifold. The convergence analysis uses proximal smoothness and tube stability
to control the projection error that $\Pi$ introduces at each aggregation. This allows us
to analyse multiple local steps and several clients together, under heterogeneity
(Limitation~\ref{lim:agg-reach}).
Limitation~\ref{lim:accounting} needs neither property: we lift it by re-accounting the
mechanism in the tight $f$-DP/GDP calculus. 

Modularity guides the design: server aggregation depends on each client's private local
training only through its final iterate, so the two FL stages are fully decoupled. Sections~\ref{ssec:local}
and~\ref{ssec:aggregation} develop them in turn, and
Section~\ref{ssec:algorithm} assembles them into Algorithm~\ref{alg:dp-rfedproj}.

\subsection{Privatised local optimisation}\label{ssec:local}
Private local training follows the pipeline of Section~\ref{sssec:local-private-opt}, with
each of its three intrinsic ingredients replaced by an ambient one: clipping moves to the
ambient norm, the tangent Gaussian is drawn ambiently and projected by $P_x$, and the
retraction becomes projection-based.
\paragraph{Extrinsic gradient privatisation.}
Under the induced metric~\eqref{eq:metric}, both operations of~\eqref{eq:han-grad} can be
carried out ambiently, in a single fixed geometry, on quantities that do not depend on the
iterate. Extrinsic and intrinsic privatisation thereby produce the same mechanism, which
Section~\ref{sec:privacy} re-accounts in the $f$-DP/GDP calculus.

The Riemannian gradient is the orthogonal projection of the
ambient Euclidean gradient, and by Property~\ref{sp:metric} the two norms $\norm{\cdot}$ and $\norm{\cdot}_x$ coincide
on $T_x\M$. We may therefore clip in the ambient norm, writing
$\clip_C:=\clip_C^{\norm{\cdot}}$ for the corresponding instance
of~\eqref{eq:clip-generic}. The threshold caps each per-record contribution by $C$
uniformly in $x$. For a minibatch $\mathcal{B}$ of size $B$, the clipped minibatch
Riemannian gradient of $f_i$ at $x$ is
\begin{equation}\label{eq:clipped-mb-grad}
\widehat{\grad}^{C} f_i(x;\mathcal{B})=\frac{1}{B}\sum_{\ell\in \mathcal{B}}\clip_C(P_x\nabla f(x;z_{i,\ell})) \in T_x\M.
\end{equation}

For the noise, we draw a single isotropic sample $\xi\sim\N(0,\sdp^2\Idp)$ in $\mathcal{E}$
and project it onto $T_x\M$. This is the tangent-space Gaussian, obtained without an
orthonormal basis.

\begin{lemma}[Projected ambient noise is tangent Gaussian]\label{lem:proj-gaussian}
  Let $x\in\M$ and $\xi\sim\N(0,\sdp^2\Idp)$ on $\mathcal{E}$. Then
  $P_x\xi\sim\N_x(0,\sdp^2)$.
\end{lemma}
\begin{proof}
    Let $(e_1,\dots,e_{\dM})$ be a $\norm{\cdot}_x$-orthonormal basis of $T_x\M$; by
  Property~\ref{sp:metric} it is orthonormal for the ambient inner product. Since $P_x$ is
  a self-adjoint projector fixing $T_x\M$, $\langle P_x\xi,e_a\rangle=\langle\xi,P_xe_a\rangle
  =\langle\xi,e_a\rangle$. These coordinates are jointly Gaussian, centred, with covariance
  $\sdp^2\langle e_a,e_b\rangle=\sdp^2\delta_{ab}$, hence i.i.d.\ $\N(0,\sdp^2)$,
  which is exactly the tangent Gaussian distribution of \citet{han2024differentially}.
\end{proof}

\noindent
Two consequences follow. First, noise sampling no longer needs a $g_x$-orthonormal basis rebuilt at each
iterate, nor the reference point and vector transport of~\citet{utpala2023improved}. Second,
the privatised gradient is distributionally identical to PriRFed's: under
Property~\ref{sp:metric} the two clipping operators coincide, and by
Lemma~\ref{lem:proj-gaussian} so do the two noise laws. 

Gathering both operations, at local step $j\in\{0,\dots,\tau-1\}$ of round $t$,
client $i\in S_t$ forms its privatised Riemannian gradient
\begin{equation}\label{eq:rec-noisy-grad}
  \tilde{g}_{t,j}^{(i)}
  \;:=\;
   \frac{1}{B}\sum_{\ell\in\mathcal{B}_{t,j}^{(i)}}
    \clip_C\!\bigl(P_{x_{t,j}^{(i)}}\nabla f(x_{t,j}^{(i)};z_{i,\ell})\bigr)
  \;+\; P_{x_{t,j}^{(i)}}\,\xi_{t,j}^{(i)},
  \qquad
  \xi_{t,j}^{(i)}\sim\N(0,\sdp^2\Idp),
\end{equation}
where the minibatch draws, the noise draws and client selection sets are mutually independent
and independent of the data. Each minibatch is thus drawn afresh at every local step,
independently of the previous ones, which is the sampling model assumed throughout
Sections~\ref{sec:privacy} and~\ref{sec:convergence}.
\paragraph{Optimiser-agnostic local updates.}
The privatised gradient~\eqref{eq:rec-noisy-grad} substitutes directly for the
stochastic Riemannian gradient~\eqref{eq:sto-grad}, which leaves the choice of local
optimiser free. The only requirement is that the optimiser accesses $\mathcal{D}_i$ through the
privatised minibatch gradients $\tilde g_{t,j}^{(i)}$ only. It may combine the current one with any quantity built from earlier privatised gradients, such as momentum buffer or adaptive moment estimates. Such computations add no privacy cost, by the post-processing property of differential privacy (Proposition~\ref{prop:calculus} (ii) in Section~\ref{sec:privacy}).

More precisely, client $i$
runs $\textsc{Opt}^{(i)}$ with hyperparameters $\theta^{(i)}$ and internal state
$s_{t,j}^{(i)}$. The state is initialised independently of the data and updated only from
the  privatised gradients, so it depends on $\mathcal{D}_i$ only through the client's earlier releases. At step $j$ the optimiser produces
\begin{equation}\label{eq:opt-step}
  \bigl(u_{t,j}^{(i)},\,s_{t,j+1}^{(i)}\bigr) \;=\;
  \textsc{Opt}^{(i)}\bigl(\tilde g_{t,j}^{(i)},\,s_{t,j}^{(i)};\theta^{(i)}\bigr),
\end{equation}
with $\textsc{Opt}^{(i)}$ and $\theta^{(i)}$ allowed to differ across clients.
This covers Riemannian stochastic gradient descent (RSGD) with stepsize $\eta$, which
takes $u_{t,j}^{(i)}=\eta\,\tilde g_{t,j}^{(i)}$ and needs no state. It also covers
heavy-ball momentum with parameter $\beta\in[0,1)$
\[
  s_{t,j+1}^{(i)}=\beta\,P_{x_{t,j}^{(i)}}s_{t,j}^{(i)}+\tilde g_{t,j}^{(i)},
  \qquad u_{t,j}^{(i)}=\eta\,s_{t,j+1}^{(i)}.
\]
Riemannian momentum and adaptive methods~\citep{becigneul2018riemannian,kasai2019riemannian} fit the same template. Optimisers that read $\mathcal{D}_i$ in any other way fall
outside the per-step accounting of Section~\ref{sec:privacy}. Riemannian stochastic
variance-reduced gradient (RSVRG)~\citep{zhang2016riemannian} is an example: its periodic
full gradient over $\mathcal{D}_i$ is a second data access. Privatising it would require
charging it separately, and it is computationally costly for large $m_i$.

\paragraph{Projection-based local steps.}
It remains to map the update direction back onto $\M$. Compactness supplies a
projection $\Pi:\mathcal{E}\to\M$, so the local update uses the
projection-based retraction $\mathrm{R}_x(u)=\Pi(x+u)$
of~\citet{absil2012projection} in place of a general one: starting from
$x_{t,0}^{(i)}:=x_t$, for $j=0,\dots,\tau-1$,
\begin{equation}\label{eq:local-step}
  x_{t,j+1}^{(i)} \;\gets\; \Pi\bigl(x_{t,j}^{(i)} - u_{t,j}^{(i)}\bigr),
\end{equation}
with $u_{t,j}^{(i)}$ given by~\eqref{eq:opt-step}. After $\tau$ such steps the client holds its final local iterate $x_{t,\tau}^{(i)}\in\M$,
from which it forms its submission $\Delta^{(i)}$ (Section~\ref{ssec:aggregation}).

\subsection{Server-side aggregation}\label{ssec:aggregation}
As in Section~\ref{sssec:server-agg}, aggregation operates on model iterates:
from the submissions $\{\Delta^{(i)}\}_{i\in S_t}$ formed by the clients from
their final local iterates, the server computes the next global iterate
$x_{t+1}\in\M$.

Building on the retracted-lifted (RL) barycenter analysis
of~\citet{bouchard2025beyond}, we use two schemes, \textsc{ProjAvg} and
\textsc{RLAvg}, involving only the tangent projector $P_x$ and the manifold projection $\Pi$ supplied by Property~\ref{sp:compact}. \citet{pautrel2026fedspdnet} introduced both for SPDNet's Stiefel weights. We state them on any compact submanifold and analyse them under privacy.
\paragraph{ProjAvg.} It projects the ambient mean of the iterates back onto $\M$:
\begin{equation}\label{eq:projavg}
  x_{t+1}^{\textsc{ProjAvg}}
  \;=\; \Pi\!\left(\frac{1}{k}\sum_{i\in S_t}\Delta^{(i)}\right),
  \qquad
  \Delta^{(i)} \;=\; x_{t,\tau}^{(i)}\in\M .
\end{equation}
This is a geometrically meaningful average: it is the closed-form RL-barycenter of
the iterates for the projection-based retraction and the tangent-projector
lifting map~\citep[Prop.~1]{bouchard2025beyond}. It also matches the project-the-mean update of~\citet{deng2025decentralized} in
the decentralised setting.

\paragraph{RLAvg.} It implements the federated analogue of the RL-barycenter in its
alternative fixed-point construction, anchored at the current iterate $x_t$:
each client submits the tangent projection at $x_t$ of its displacement, and
the server retracts their average at the same point,
\begin{equation}\label{eq:rlavg}
  x_{t+1}^{\textsc{RLAvg}}
  \;=\; \Pi\!\left(x_t + \frac{1}{k}\sum_{i\in S_t}\Delta^{(i)}\right),
  \qquad
  \Delta^{(i)} \;=\; P_{x_t}\bigl(x_{t,\tau}^{(i)} - x_t\bigr)\in T_{x_t}\M .
\end{equation}
This mirrors the lift--average--retract structure of \textsc{KFAvg} and
\textsc{RetAvg}, with the logarithm and inverse retraction replaced by the
linear projector $P_{x_t}$, retaining an anchored, tangent-space aggregation.

\subsection{Algorithm}\label{ssec:algorithm}
Algorithm~\ref{alg:dp-rfedproj} assembles the two stages: each client runs
$\tau$ private local steps~\eqref{eq:local-step} with its own optimiser and
sends the single submission $\Delta^{(i)}$, from which the server forms the
next global iterate via the chosen aggregation scheme.

\begin{algorithm}[h]
\caption{\textsc{DP-RFedProj}: record-level DP Riemannian federated learning}
\label{alg:dp-rfedproj}
\begin{algorithmic}[1]
\Require
  $n$ clients;\; $k$ clients/round;\; $\tau$ local steps;\;
  batch size $B$;\; $T$ rounds;\;
  client optimisers $\{\textsc{Opt}^{(i)}\}_{i\in[n]}$ with hyperparameters
  $\{\theta^{(i)}\}_{i\in[n]}$;\;
  aggregator $\textsc{Agg}\in\{\textsc{ProjAvg},\textsc{RLAvg}\}$
\Statex \hspace{-1.2em}\textbf{DP parameters:}\;
  clip $C$;\; per-step noise scale $\sdp$
\State Initialise $x_0\in\M$
\For{$t=0$ to $T-1$}\Comment{communication round}
  \State Draw $S_t\subset[n]$ uniformly without replacement, $|S_t|=k$
  \State Broadcast $x_t$ to every client $i\in S_t$
  \ForAll{$i\in S_t$ \textbf{in parallel}}
    \State $x_{t,0}^{(i)}\gets x_t$ and $s_{t,0}^{(i)}\gets$ initial optimiser state
    \For{$j=0$ to $\tau-1$}\Comment{privatised local step}
    \State Draw $\mathcal{B}_{t,j}^{(i)}\subset\mathcal{D}_i$ of size $B$, uniformly without replacement
      \State Draw $\xi_{t,j}^{(i)}\sim\N\!\bigl(0,\sdp^2\Idp\bigr)$
      \State $\displaystyle \tilde g_{t,j}^{(i)} \gets
        \frac{1}{B}\sum_{\ell\in\mathcal{B}_{t,j}^{(i)}}
        \clip_C\!\bigl(P_{x_{t,j}^{(i)}}\nabla f(x_{t,j}^{(i)};z_{i,\ell})\bigr)
        \;+\; P_{x_{t,j}^{(i)}}\,\xi_{t,j}^{(i)}$
        \Comment{Eq.~\eqref{eq:rec-noisy-grad}}
      \State $\bigl(u_{t,j}^{(i)},s_{t,j+1}^{(i)}\bigr)\gets\textsc{Opt}^{(i)}\bigl(\tilde g_{t,j}^{(i)},\,s_{t,j}^{(i)};\,\theta^{(i)}\bigr)$
      \State $x_{t,j+1}^{(i)}\gets\Pi\bigl(x_{t,j}^{(i)}-u_{t,j}^{(i)}\bigr)$
        \Comment{Eq.~\eqref{eq:local-step}}
    \EndFor
    \State $\Delta^{(i)}\gets \textsc{Agg.Sub}(x_t,x_{t,\tau}^{(i)})$ \Comment{Table~\ref{tab:agg-schemes}}
    \State Send submission $\Delta^{(i)}$ to the server
     \EndFor
     \State The server computes $\Delta_t\gets\frac{1}{k}\sum_{i\in S_t}\Delta^{(i)}$
  \State $x_{t+1}\gets \textsc{Agg.Out}(x_t,\Delta_t)$
\EndFor
\State \Return $x_T$
\end{algorithmic}
\end{algorithm}

\paragraph{Unified template.} Algorithm~\ref{alg:dp-rfedproj} is stated for our two schemes, but the aggregator enters it
at Lines~13 and 16 only. Writing those two lines as a client-side submission
\textsc{Agg.Sub} and a server-side output \textsc{Agg.Out} yields a unifying template:
substituting the first two rows of Table~\ref{tab:agg-schemes} recovers
\textsc{KFAvg}~\citep{li2022federated} and \textsc{RetAvg}~\citep{huang2026federated} aggregation schemes. The unification holds at the aggregation level only. At the local optimisation
level our approach privatises extrinsically~\eqref{eq:rec-noisy-grad}, whereas PriRFed privatises intrinsically, in the
iterate-dependent geometry~\eqref{eq:han-grad}.  Any privatised local optimiser of Section~\ref{ssec:local}
may be paired with any of the four schemes.

\begin{table}[h]
\centering
\caption{Aggregation primitives for the four schemes. \textsc{Agg.Sub}$(x_t,x_{t,\tau}^{(i)})$ is the client-side submission $\Delta^{(i)}$ and \textsc{Agg.Out}$(x_t,\Delta_t)$ is the server-side global update $x_{t+1}$, where $\Delta_t=\tfrac1k\sum_{i\in S_t}\Delta^{(i)}$.}
\label{tab:agg-schemes}
\begin{tabular}{@{}lll@{}}
\toprule
Scheme & $\textsc{Agg.Sub}(x_t,x_{t,\tau}^{(i)})$ & $\textsc{Agg.Out}(x_{t},\Delta_t)$  \\
\midrule
\textsc{KFAvg}  & $\mathrm{Log}_{x_t}\bigl(x_{t,\tau}^{(i)}\bigr)$        & $\mathrm{Exp}_{x_t}\bigl(\Delta_t\bigr)$  \\
\textsc{RetAvg} & $\mathrm{R}_{x_t}^{-1}\bigl(x_{t,\tau}^{(i)}\bigr)$     & $\mathrm{R}_{x_t}\bigl(\Delta_t\bigr)$    \\
\textsc{ProjAvg} (ours) & $x_{t,\tau}^{(i)}$                              & $\Pi\bigl(\Delta_t\bigr)$                 \\
\textsc{RLAvg} (ours) & $P_{x_t}\bigl(x_{t,\tau}^{(i)}-x_t\bigr)$         & $\Pi\bigl(x_t+\Delta_t\bigr)$            \\
\bottomrule
\end{tabular}
\end{table}
\subsection{Scope of the privacy and convergence guarantees}\label{ssec:scope}
The privacy guarantee (Section~\ref{sec:privacy}) and the convergence guarantee
(Section~\ref{sec:convergence}) are scoped along two independent axes: which aggregation
scheme is used (Table~\ref{tab:agg-schemes}), and which manifold the parameters live on
(Table~\ref{tab:scope}).

\paragraph{Choice of aggregator.} The convergence rate of
Section~\ref{sec:convergence} is established for \textsc{ProjAvg} and \textsc{RLAvg} only:
both their primitives are linear in the submissions, which is what lets the $\tau$ local
displacements telescope in the single ambient space $\mathcal{E}$ and the $k$ client
submissions be averaged rather
than composed (Lemmas~\ref{lem:linearisation} and~\ref{lem:server-decomp}). For
\textsc{KFAvg} and \textsc{RetAvg} the submission is a nonlinear lift and the output a
nonlinear map of the average, and neither step goes through.

The privacy guarantee of Section~\ref{sec:privacy}, by contrast, covers all four. The data
enter a client's computation only through the privatised gradient~\eqref{eq:rec-noisy-grad},
where Section~\ref{sec:privacy} settles the privacy cost in the ambient space, before any
geometry is applied. Everything downstream is computed from these privatised gradients
alone: the projections $P_x$ and $\Pi$, the local optimiser and the aggregation scheme.
By the post-processing property of differential privacy
(Proposition~\ref{prop:calculus}(ii)), applying to a private output any map that ignores
the data keeps the guarantee unchanged. The guarantee therefore carries through on any
embedded submanifold where the algorithm runs at all, and for any aggregation scheme of
Table~\ref{tab:agg-schemes} (Theorem~\ref{thm:gdp-subsampled}).
\paragraph{Choice of manifold.} With \textsc{ProjAvg} or \textsc{RLAvg}, the geometry enters
Algorithm~\ref{alg:dp-rfedproj} only through $P_x$ and $\Pi$, which fixes where it applies
(Table~\ref{tab:scope}). Properties~\ref{sp:metric} and \ref{sp:compact} play distinct roles here.
Property~\ref{sp:compact} is a sufficient condition which makes the algorithm runnable, by supplying $\Pi$: without
that map there is neither a local step~\eqref{eq:local-step} nor an aggregation.
Property~\ref{sp:metric} aligns the mechanism with the manifold's geometry. Under the induced
metric, the projected noise is isotropic for $g_x$ and the clip bounds each record's
contribution in $\norm{\cdot}_x$. Under another metric, the privacy guarantee still holds,
since it is settled in the ambient space, but noise and clipping follow the ambient geometry
instead of $g_x$. Both hold on the sphere and on the
Stiefel and Grassmann manifolds, where the two guarantees apply together.

Two flat cases fail compactness but retain the projection it was meant to supply. When
$\M=\mathcal{E}$ the two projection maps are the identity, the local
steps~\eqref{eq:rec-noisy-grad}--\eqref{eq:local-step} are DP-SGD~\citep{abadi2016deep} and
both aggregators reduce to the arithmetic mean. When $\M=\mathcal{S}^{++}_n$, under the
log-Euclidean metric~\citep{arsigny2007geometric}, the matrix logarithm is a global isometry
from $\mathcal{S}^{++}_n$ onto the flat space $\mathcal{S}_n$ of symmetric matrices, so
reparametrising by the log chart puts the parameters on
$\mathcal{S}_n\subset\R^{n\times n}$ with the induced metric, where $\Pi$ and $P_x$ both
reduce to symmetrisation~\citep{utpala2023improved}. In both, $\Pi$ is globally defined and
single-valued, so Algorithm~\ref{alg:dp-rfedproj} runs and the privacy guarantee holds. 
What does not carry over is the convergence analysis, which uses compactness for the uniform bounds it relies on.
The same $\mathcal{S}^{++}_n$ under the affine-invariant metric~\citep{pennec2006riemannian}
falls outside on both counts. There is no ambient projection back to the cone, so unlike the flat cases nothing
substitutes for $\Pi$ and Property~\ref{sp:compact} fails. Moreover, the metric is
iterate-dependent, so Property~\ref{sp:metric} fails. Algorithm~\ref{alg:dp-rfedproj}
cannot be run there at all and 
PriRFed~\citep{huang2026federated} remains the baseline. 

\begin{table}[h]
\centering
\caption{Where Algorithm~\ref{alg:dp-rfedproj} applies. Privacy
(Theorem~\ref{thm:gdp-subsampled}) holds for all four schemes of
Table~\ref{tab:agg-schemes}; the convergence rate (Corollary~\ref{cor:main}) for
\textsc{ProjAvg} and \textsc{RLAvg} with RSGD.}
\label{tab:scope}
\begin{tabular}{@{}llll@{}}
\toprule
Parameter space & Metric & Privacy & Conv. rate \\
\midrule
Sphere, Stiefel, Grassmann            & induced       & $\checkmark$ & $\checkmark$ \\
$\mathcal{E}$ (incl. $\log\bigl(\mathcal{S}^{++}_n\bigr)$)                                     & Euclidean         & $\checkmark$ & $\times$ \\
$\mathcal{S}^{++}_n$                               & affine-inv.    & \multicolumn{2}{c}{algorithm not runnable} \\
\bottomrule
\end{tabular}
\end{table}

\section{Privacy Analysis}\label{sec:privacy}
The honest-but-curious server observes the whole transcript
$\mathcal V_T$~\eqref{eq:transcript}, from which no client's record should be
inferable: we require $\mathcal V_T$ to be private under the record-level adjacency
of Definition~\ref{def:adjacency}. 
We adapt the standard DP-SGD privacy argument~\citep{abadi2016deep,dong2022gaussian} to projection-based federated learning, and obtain a record-level guarantee for the whole training run.

Section~\ref{ssec:background} recalls the tools of the $f$-DP calculus, Section~\ref{ssec:guarantee} states the guarantee, and Section~\ref{ssec:reading} derives, under additional asymptotic assumptions, what drives its privacy level.

\subsection{Privacy accounting in the \texorpdfstring{$f$}{f}-DP calculus}
\label{ssec:background}
We measure privacy with the trade-off-function calculus of
\citet{dong2022gaussian}, which quantifies by hypothesis testing how hard two
adjacent outputs are to distinguish, and composes tightly.

\begin{definition}[$f$-DP and GDP {\citep[Def.~2--4]{dong2022gaussian}}]\label{def:fdp}
  For distributions $P,Q$ on a common space, the trade-off function
  $\mathcal{T}(P,Q):\alpha\mapsto\inf\{\E_Q[1-\phi]:\E_P[\phi]\le\alpha\}$ is the
  smallest type-II error achievable at type-I level $\alpha$; it is convex,
  continuous and non-increasing, and larger means harder to distinguish. A
  mechanism $\mathcal{A}$ is \emph{$f$-DP} if
  $\mathcal{T}(\mathcal{A}(D),\mathcal{A}(D'))\ge f$ for every adjacent pair
  $D\bowtie D'$. Two trade-off functions matter below: $\mathrm{Id}(\alpha)=1-\alpha$,
  the largest one, attained when the two releases are indistinguishable; and
  $G_\mu(\alpha):=\Phi(\Phi^{-1}(1-\alpha)-\mu)$, for testing $\N(0,1)$ against
  $\N(\mu,1)$, with $\Phi$ the standard normal cumulative distribution function. A
  mechanism is \emph{$\mu$-GDP} if it is $G_\mu$-DP: larger $\mu$ is weaker, and
  $G_0=\mathrm{Id}$ is perfect privacy.
\end{definition}

An $f$-DP guarantee is equivalent to a whole family of $(\varepsilon,\delta)$-DP ones.

\begin{proposition}[Primal--dual correspondence
  {\citep[Prop.~3, 6 and Cor.~1]{dong2022gaussian}}]\label{prop:fdp-to-dp}
  A mechanism is $(\varepsilon,\delta)$-DP if and only if it is
  $f_{\varepsilon,\delta}$-DP, where
  \(
    f_{\varepsilon,\delta}(\alpha):=\max\{0,\,1-\delta-e^{\varepsilon}\alpha,\,
      e^{-\varepsilon}(1-\delta-\alpha)\}.
  \)
  Conversely, a mechanism is $f$-DP for a symmetric $f$ if and only if it is
  $(\varepsilon,\delta_f(\varepsilon))$-DP for every $\varepsilon\ge0$, with
  $\delta_f(\varepsilon)=1+f^{\ast}(-e^{\varepsilon})$ and
  $f^{\ast}(y):=\sup_{\alpha\in[0,1]}(y\alpha-f(\alpha))$. For $f=G_\mu$,
  \begin{equation}\label{eq:gdp-profile}
    \delta_{G_\mu}(\varepsilon)
    =\Phi\Bigl(-\tfrac{\varepsilon}{\mu}+\tfrac{\mu}{2}\Bigr)
     -e^{\varepsilon}\,\Phi\Bigl(-\tfrac{\varepsilon}{\mu}-\tfrac{\mu}{2}\Bigr).
  \end{equation}
\end{proposition}

The guarantee rests on four operations of this calculus, stated by \citet{dong2022gaussian}
for datasets differing in one record, that is, under
Definition~\ref{def:adjacency}.

\begin{proposition}[The $f$-DP calculus
  {\citep[Thm.~1, Prop.~4, Def.~5--6, Thm.~4 and~9, Cor.~2]{dong2022gaussian}}]
\label{prop:calculus}
  Under record-substitution adjacency (Definition~\ref{def:adjacency}):
  \begin{enumerate}[label=(\roman*), leftmargin=*, itemsep=2pt]
    \item \emph{Gaussian mechanism.} If $v$ maps databases to $\mathcal{E}$ with
      sensitivity $\sens(v):=\sup_{D\bowtie D'}\norm{v(D)-v(D')}$, then $v+\xi$ with
      $\xi\sim\N(0,\sigma^2\Idp)$ is $(\sens(v)/\sigma)$-GDP.
    \item \emph{Post-processing.} If $\mathcal{A}$ is $f$-DP and $h$ is a (possibly
      randomised) map independent of the database, then $h\circ\mathcal{A}$ is $f$-DP.
    \item \emph{Subsampling.} Write $f_q:=q f+(1-q)\mathrm{Id}$ and
      $C_q(f):=\min\{f_q,f_q^{-1}\}^{\ast\ast}$, with
      $f^{-1}(\alpha):=\inf\{t:f(t)\le\alpha\}$ and ${}^{\ast\ast}$ the greatest convex
      minorant. If $\mathcal{A}$ is $f$-DP on $B$-record inputs and $\mathtt{Sample}_B$
      draws a uniform $B$-subset without replacement from an $m$-record dataset, then
      $\mathcal{A}\circ\mathtt{Sample}_B$ is $C_q(f)$-DP with $q=B/m$. For symmetric $f$,
      $C_q(f)$ is symmetric and non-increasing in $q$, with $C_1(f)=f$ and
      $C_0(f)=\mathrm{Id}$.
    \item \emph{Adaptive composition.} For $f=\mathcal T(P,Q)$ and
      $f'=\mathcal T(P',Q')$, let $f\otimes f':=\mathcal T(P\times P',Q\times Q')$ be the
      optimal trade-off for the two releases seen jointly. It is commutative,
      associative, monotone for the pointwise order and symmetry-preserving, with
      $f\otimes\mathrm{Id}=f$. If $\mathcal A_1$ is $f$-DP and $\mathcal A_2(\cdot,y_1)$
      is $f'$-DP for every output $y_1$ of $\mathcal A_1$, their adaptive composition is
      $(f\otimes f')$-DP. Iterating, an $N$-fold adaptive composition of $f$-DP
      mechanisms is $f^{\otimes N}$-DP. In particular,
      $G_{\mu}^{\otimes N}=G_{\sqrt N\mu}$.
  \end{enumerate}
\end{proposition}

\begin{remark}[What is amplified]\label{rem:hidden-indices}
  Privacy amplification by subsampling (Proposition~\ref{prop:calculus}(iii)) requires
  that the subsample is not released, so that the adversary cannot tell whether the
  substituted record was drawn. It therefore applies to the minibatch, which never leaves
  the client, but not to the participation set $S_t$, which the server draws itself and
  keeps in $\mathcal V_T$: all amplification below is record-level, at rate $q_i=B/m_i$.
\end{remark}

\subsection{Record-level guarantee} \label{ssec:guarantee}
The transcript is assembled from privatised
gradients of the form~\eqref{eq:rec-noisy-grad}, so the analysis starts with a
single one. Lemma~\ref{lem:subsamp-step} is adapted from the per-step bound of DP-SGD, with the only adjustment being that the tangent projection of the noise is post-processing of the ambient release.

\begin{lemma}[One subsampled local step]\label{lem:subsamp-step}
  Fix $C,\sdp>0$, a client $i$, a point $x\in\M$ and a batch size $B\le m_i$. Let
  $\mathcal B$ be a uniform minibatch of size $B$ drawn from $\mathcal D_i$ without
  replacement, independently of $\xi\sim\N(0,\sdp^2\Idp)$, and set
  $\mu_0:=2C/(B\sdp)$. Then the mechanism
  \[
    \mathcal D_i\;\longmapsto\;
    \tilde g \;:=\;\hgradC f_i(x;\mathcal B)+P_x\xi
  \]
  is $C_{q_i}(G_{\mu_0})$-DP under substitution of one record of $\mathcal D_i$,
  at rate $q_i=B/m_i$.
\end{lemma}

\begin{proof}
  Fix the batch and consider $v_x(\mathcal B):=\hgradC f_i(x;\mathcal B)$ as a
  mechanism on $B$-record inputs (see~\eqref{eq:clipped-mb-grad}). Each clipped
  per-sample gradient has ambient norm at most $C$, which the induced
  metric~\eqref{eq:metric} makes its tangent norm too, so substituting one record of
  $\mathcal B$ moves the average by at most $\sens(v_x)=2C/B$. By
  Proposition~\ref{prop:calculus}(i) the ambient release $v_x(\mathcal B)+\xi$
  is $(\sens(v_x)/\sdp)$-GDP, that is $\mu_0$-GDP. The tangent projector $P_x$ is a
  fixed linear map that does not depend on the data, and $v_x$ is already
  $T_x\M$-valued, so
  \[
    P_x\bigl(v_x(\mathcal B)+\xi\bigr)
    \;=\;v_x(\mathcal B)+P_x\xi\;=\;\tilde g ,
  \]
  and $\tilde g$ inherits the budget $\mu_0$ by post-processing
  (Proposition~\ref{prop:calculus}(ii)), uniformly in $x$. Finally the batch is a
  uniform $B$-subset drawn without replacement and never released, so
  Proposition~\ref{prop:calculus}(iii) amplifies $G_{\mu_0}$ to
  $C_{q_i}(G_{\mu_0})$ at rate $q_i=B/m_i$.
\end{proof}

Two features of this proof carry the rest of the section. The budget is
settled on the ambient release, so the manifold enters only as post-processing, as will every
geometric map below. The guarantee is uniform in $x$, so the lemma still
applies when the iterate is produced by earlier releases.

It remains to lift Lemma~\ref{lem:subsamp-step} to the transcript $\mathcal V_T$. As for DP-SGD, this takes adaptive composition and post-processing. The only point specific to our setting is that the iterates, the projections and the aggregation must be data-independent functions of the earlier releases.
We split the protocol randomness into a part drawn independently of the data,
\[
  \mathcal R := \bigl(x_0,\ \{S_t\}_t,\ \{s^{(i)}_0\}_{i\in[n]},\
  \{\omega^{(i)}_{t,j}\}_{t,j,i}\bigr),
\]
consisting of model initialisation, client schedule, initial optimiser states and
per-step optimiser randomness, and a part that depends on it: the stream of
privatised gradients
\(
  \Gamma := \{\tilde g^{(i)}_{t,j}\}_{t,\,i\in S_t,\,j},
\)
each entry an instance of Lemma~\ref{lem:subsamp-step}. The minibatch draws stay
internal to the per-step mechanism, which is what keeps their amplification
available (Remark~\ref{rem:hidden-indices}).

\begin{proposition}[Geometry-free reduction]\label{prop:geom-free}
  There is a deterministic measurable map $\Psi$ with
  $\mathcal V_T=\Psi(\mathcal R,\Gamma)$, in which the manifold $\M$, the maps
  $(P_\cdot,\Pi)$, the aggregator $\textsc{Agg}$ and the optimisers
  $\{\textsc{Opt}^{(i)}\}$ act only as fixed, data-independent operations.
  Consequently, if $\Gamma$ is $f$-DP conditionally on $\mathcal R=r$, uniformly
  in $r$, then $\mathcal V_T$ is $f$-DP.
\end{proposition}

\begin{proof}
  Reconstruct the protocol from $(\mathcal R,\Gamma)$ by induction on the rounds:
  at round $t$ the active set $S_t$ is read off $\mathcal R$, each selected
  client's local iterates follow from $x_{t,0}^{(i)}=x_t$
  by~\eqref{eq:local-step}, its update direction reading the gradient from
  $\Gamma$ and the optimiser randomness from $\mathcal R$, and the submissions and
  next global iterate follow from Table~\ref{tab:agg-schemes}. No step consults
  the data; this defines $\Psi$. For the transfer, $\mathcal R$ is drawn
  independently of the data, so its release is $\mathrm{Id}$-DP and adaptive
  composition (Proposition~\ref{prop:calculus} (iv) ) makes $(\mathcal R,\Gamma)$
  $(\mathrm{Id}\otimes f)$-DP, that is $f$-DP. Finally $\mathcal V_T$ is a
  post-processing of $(\mathcal R,\Gamma)$, hence $f$-DP by
  Proposition~\ref{prop:calculus} (ii).
\end{proof}
In particular every iterate $x^{(i)}_{t,j}$, hence every projector
$P_{x^{(i)}_{t,j}}$, is determined by $\mathcal R$ and the earlier releases, so
Lemma~\ref{lem:subsamp-step} applies at each step with $x$ fixed and adaptive composition
chains them. The three pieces now assemble.

\begin{theorem}[Privacy of \textsc{DP-RFedProj} under HBC]
\label{thm:gdp-subsampled}
  Let $m_\star:=\min_{i\in[n]}m_i$, $q:=B/m_\star$ and $\mu_0:=2C/(B\sdp)$. For
  every $T\ge1$, the transcript $\mathcal V_T$ of
  Algorithm~\ref{alg:dp-rfedproj} is $C_q(G_{\mu_0})^{\otimes T\tau}$-DP under
  record-level adjacency. The budget depends only on $(C,B,m_\star,\sdp)$ and the
  product $T\tau$, not on the number $k$ of participating clients, not on
  $(\M,\Pi,\{\textsc{Opt}^{(i)}\})$, and not on the aggregator
  $\textsc{Agg}\in\{\textsc{ProjAvg},\textsc{RLAvg},\textsc{KFAvg},\textsc{RetAvg}\}$.
\end{theorem}

\begin{proof}
  Fix $D\bowtie D'$ differing at client $i^\star$ and condition on $\mathcal R$.
    The datasets being disjoint, every other client's release has, conditionally on the
  prefix of the transcript, the same law under $D$ and $D'$: their factors are
  $\mathrm{Id}$-DP and drop out of the tensor product, which is why $k$ does not appear.
  Each release of $i^\star$ is the mechanism of Lemma~\ref{lem:subsamp-step} at an
  iterate fixed by the prefix, hence $C_{q_{i^\star}}(G_{\mu_0})$-DP, and adaptive
  composition (Proposition~\ref{prop:calculus} (iv)) tensorises them. If $i^\star$ is
  not selected in every round it contributes fewer than $T\tau$ factors, and
  $f^{\otimes(N+1)}\le f^{\otimes N}$ since $f\le\mathrm{Id}$, so the $T\tau$-fold
  bound holds whatever the schedule. It cannot be reduced: the server draws $S_t$
  and records it in $\mathcal V_T$, so the number of selections is public and the
  guarantee must cover the schedule selecting $i^\star$ every round.
  This bound is uniform in $\mathcal R$, so Proposition~\ref{prop:geom-free}
  carries it to $\mathcal V_T$. Finally $q_{i^\star}\le q$ and $C_q$ is
  non-increasing in $q$ (Proposition~\ref{prop:calculus} (iii)), giving the cohort bound.
\end{proof}
The same argument covers clients that run different numbers $\tau_i \leq \tau$ of local steps, as in our experiments (Appendix~\ref{app:dp-settings}). 
Every scheme of Table~\ref{tab:agg-schemes} reads
the same privatised gradients, so the guarantee
also covers Algorithm~\ref{alg:dp-rfedproj} run with \textsc{KFAvg} or
\textsc{RetAvg}. It does not cover PriRFed~\citep{huang2026federated}: only the
aggregation rule carries over, not the way the gradients are privatised, which
there happens in the geometry of the current iterate~\eqref{eq:han-grad} rather
than in the fixed ambient one~\eqref{eq:rec-noisy-grad}.

\subsection{Interpreting the guarantee}\label{ssec:reading}
Proposition~\ref{prop:fdp-to-dp} turns Theorem~\ref{thm:gdp-subsampled} into a full
$(\varepsilon,\delta)$ curve, but a tensor power of subsampled trade-off functions is hard
to read. We therefore report a single summary number, obtained from the $f$-DP central
limit theorem. It is used for interpretation only, since the noise scales of
Sections~\ref{sec:synth-exp} and \ref{sec:num-exps} are calibrated by a numerical accountant that certifies an upper bound for
Theorem~\ref{thm:gdp-subsampled} itself (Appendix~\ref{app:dp-settings}).

The bound is a composition of subsampled Gaussian mechanisms, which is the setting the
$f$-DP central limit theorem covers: with $N:=T\tau$, $\mu_0$ fixed and
$q\sqrt N\to c>0$,
\begin{equation*}
  C_q(G_{\mu_0})^{\otimes N}\;\longrightarrow\;G_{\mu_\infty},
  \qquad \mu_\infty=c\,\sqrt{2\chi^2_+(G_{\mu_0})},
  \qquad \chi^{2}_+(G_{\mu_0})=e^{\mu_0^{2}}\Phi\bigl(\tfrac{3\mu_0}{2}\bigr)
    +3\Phi\bigl(-\tfrac{\mu_0}{2}\bigr)-2,
\end{equation*}
uniformly on $[0,1]$, where $\chi^2_+(f):=\int_0^1(|f'|-1)_+^2$
\citep[Thm.~11 and Lem.~3]{dong2022gaussian}. We therefore summarise a finite run by
\begin{equation}\label{eq:mu-tot}
  \mu_{\mathrm{tot}}:=q\sqrt{2N\,\chi^2_+(G_{\mu_0})},
\end{equation}
reporting the transcript as $\mu_{\mathrm{tot}}$-GDP and reading
$(\varepsilon,\delta)$ off~\eqref{eq:gdp-profile}. Expanding at small $\mu_0$ and
using $q\mu_0=2C/(m_\star\sdp)$ shows what drives it:
\begin{equation}\label{eq:mu-tot-expansion}
  \mu_{\mathrm{tot}}
  = \frac{2C\sqrt{T\tau}}{m_\star\,\sdp}
    \Bigl(1+\tfrac{\mu_0}{\sqrt{2\pi}}+O(\mu_0^{2})\Bigr).
\end{equation}
The budget grows as $\sqrt{T\tau}$ and is set by the clipping threshold relative to the
noise, damped by the smallest client's dataset size. At fixed $T\tau$, the batch size cancels at leading order.

\section{Convergence Analysis}\label{sec:convergence}
We now show that Algorithm~\ref{alg:dp-rfedproj} converges with several local
steps, partial participation and heterogeneous clients at once, which is precisely the regime
left open by Limitation~\ref{lim:agg-reach}. Three specialisations fix the
algorithm under analysis.
\begin{enumerate}[label=(\roman*),nosep]
 \item \emph{Nearest-point projection.} We set $\Pi$ to be the nearest-point
    map~\eqref{eq:nearest-p-proj}, which compactness makes well defined near $\M$
    (Assumption~\ref{ass:prox}) and which the linear tangent projector $P_x$
    approximates to first order at every base point~\eqref{eq:proj-quad-rem}.
    This linear model, available at each iterate, is what replaces the nonlinear
    inverse retraction of PriRFed~\citep{huang2026federated}.
  \item \emph{Local optimiser.} Each client runs plain Riemannian stochastic
    gradient descent (RSGD) with a rate $\eta>0$ held constant across steps,
    rounds and clients, so that the local update~\eqref{eq:local-step} reads
    \begin{equation}\label{eq:sgd}
      x_{t,j+1}^{(i)} = \Pi\bigl(x_{t,j}^{(i)} - \eta\,\tilde{g}_{t,j}^{(i)}\bigr),
    \end{equation}
  \item \emph{Server aggregation.} We restrict to \textsc{ProjAvg}~\eqref{eq:projavg}
    and \textsc{RLAvg}~\eqref{eq:rlavg}, whose primitives are the two projections.
    The nonlinear lift and retraction maps of \textsc{KFAvg} and \textsc{RetAvg} are
    precisely what the linearisation step (Lemma~\ref{lem:linearisation}) cannot handle.
\end{enumerate}
Throughout, $\alpha:=\eta\tau$ denotes the aggregate per-round stepsize, in which the rates
are stated. Prescribing $\alpha$ means running Algorithm~\ref{alg:dp-rfedproj} with
$\eta=\alpha/\tau$. 

\subsection{Assumptions}\label{ssec:conv-setup}
Both the local and the aggregation stages of Algorithm~\ref{alg:dp-rfedproj} evaluate $\Pi$ off the manifold. For the algorithm to be
well-defined and the iterates controllable, $\Pi$ must be single-valued throughout a
neighbourhood of $\M$, a property guaranteed by proximal
smoothness~\citep{clarke1995proximal,davis2025stochastic}: for $\gamma>0$, define
the open tube
\[
  U_\M(\gamma) := \{\,y\in\mathcal{E} : \mathrm{dist}(y,\M) < \gamma\,\},
  \qquad \mathrm{dist}(y,\M) := \inf_{x\in\M}\|y-x\|,
\]
and call $\M$ \emph{$\gamma$-proximally smooth} if $\Pi$ is single-valued on
$U_\M(\gamma)$.

A compact smooth embedded submanifold is proximally smooth. We fix the radius below.
\begin{assumption}[Proximal smoothness]\label{ass:prox}
  $\M$ is $2\gamma$-proximally smooth for some fixed $\gamma>0$.
\end{assumption}

Since $F$ is continuous on the compact $\M$, it attains finite extrema.
Writing $F_\star := \inf_{x\in\M} F(x)$, any two points of $\M$ differ in objective
value by at most
\begin{equation}\label{eq:Delta-F}
  \delta_F := \sup_{x\in\M} F(x) - F_\star .
\end{equation}
Each continuous ambient gradient $\nabla f(\cdot;z)$ is bounded on $\M$ at a given data point $z$. A uniform gradient bound over the training records only would make the stepsize condition, which involves
$G$, depend on the data, and the privacy guarantee would no longer apply. We therefore assume the following uniform gradient bound over all possible data. 
\begin{assumption}[Uniform gradient bound]\label{ass:gd-bd}
The per-sample gradients are bounded on $\M$ over the whole data space $\mathcal Z$. 
\[
  G := \sup_{z\in\mathcal Z}\;\sup_{x\in\M}\;\|\nabla f(x;z)\| \;<\;\infty.
\]
\end{assumption}

Clipping acts on each per-record gradient before averaging, so the analysis must control
every per-sample loss $f(\cdot\,;z)$, uniformly in $z$. This is stronger than the
smoothness of the averaged objective assumed by~\citet{deng2025decentralized}.
\begin{assumption}[Smoothness]\label{ass:smooth}
  Each per-sample loss $f(\cdot\,;z)$ has $\ell_f$-Lipschitz ambient gradient on
  $\mathrm{conv}(\M)$, with $\ell_f$ independent of $z$:
  \[
    \|\nabla f(x;z) - \nabla f(y;z)\|
    \le \ell_f\|x - y\|,
    \qquad x,y \in \mathrm{conv}(\M),\;z\in\mathcal{Z}.
  \]
\end{assumption}
Because $\ell_f$ is uniform in $z$, every per-sample loss obeys the same
quadratic bound, which averaging passes to $F$ as defined in \eqref{eq:fl-objective}. Under
Assumptions~\ref{ass:prox}--\ref{ass:smooth}, the following descent
inequality~\citep[Lem.~4.2]{deng2025decentralized} holds on $\M$, with $\LM := \ell_f + \frac{G}{2\gamma}$:
\begin{equation}\label{eq:descent-ineq}
  F(y) \le F(x) + \langle \grad F(x),\, y - x\rangle
  + \frac{\LM}{2}\|y - x\|^2.
\end{equation}

Local steps (Algorithm~\ref{alg:dp-rfedproj}, Line 10) use minibatch gradients, so the analysis must control their
dispersion.

\begin{assumption}[Stochastic gradients]\label{ass:sgd}~
    There exists $\sigma_{\mathrm{st}}^2 < \infty$ such that, for
      all $i \in [n]$ and $x \in \M$,
      \begin{equation}\label{eq:within-var}
        \frac{1}{m_i}\sum_{\ell=1}^{m_i}
        \|\grad f(x;z_{i,\ell}) - \grad f_i(x)\|^2
        \le \sigma_{\mathrm{st}}^2.
      \end{equation}
\end{assumption}
For a threshold $C>0$, a client $i$, and $x\in\M$, let us write the clipped per-client full
gradient as
\begin{equation}\label{eq:clip-grad-full}
  \gradC f_i(x)
    := \frac{1}{m_i}\sum_{\ell=1}^{m_i}\clip_C\!\bigl(\grad f(x;z_{i,\ell})\bigr), 
\end{equation}
whose minibatch estimator is \eqref{eq:clipped-mb-grad}. This estimator is unbiased, with variance at most $\rho_i\sigma_{\mathrm{st}}^2$, where $\rho_i:=(m_i-B)/(B(m_i-1))$ is the factor of sampling without replacement (Lemma~\ref{lem:clip-var} in Appendix~\ref{app:aux-lemmas}).

Our next assumption bounds the heterogeneity driving the drift of the client gradients from the  gradient of the objective $\grad F=\frac{1}{n}\sum_{i=1}^{n}\grad f_i$.
\begin{assumption}[Client-gradient heterogeneity]\label{ass:hetero}
  There exists $\sighet^2 \ge 0$ such that
  \[
    \sup_{x \in \M}\;\frac{1}{n}\sum_{i=1}^n
      \|\grad f_i(x) - \grad F(x)\|^2 \le \sighet^2 .
  \]
\end{assumption}
\noindent
Clipping perturbs each $\grad f_i(x)$ by a bias
$b_i^C(x):=\gradC f_i(x)-\grad f_i(x)$. These biases differ across clients, which
adds to the heterogeneity already present. We measure their spread by
\begin{equation}\label{eq:sigma-clip-def}
  \sigma_{\mathrm{clip}}^2 := \sup_{x\in\M}\frac1n\sum_{i=1}^n
    \Bigl\|b_i^C(x)-\frac1n\sum_{j=1}^n b_j^C(x)\Bigr\|^2 .
\end{equation}
This constant is explicit: $\|\clip_C(v)-v\|=(\|v\|-C)_+$ and
$\|\grad f(x;z)\|\le G$ give $\sigma_{\mathrm{clip}}\le(G-C)_+$, which vanishes
whenever $C\ge G$. The clipped client gradients then deviate from their mean by at most $(\sigma_{\mathrm{het}}+\sigma_{\mathrm{clip}})^2$ on average (Lemma~\ref{lem:clip-hetero} in Appendix~\ref{app:aux-lemmas}). 
The per-client clipping biases also have a nonzero mean, which shifts the aggregated signal
away from $\grad F$. It is quantified by the clipping-bias floor
\begin{equation}\label{eq:beta-C-def}
  \beta_C := \sup_{x\in\M}
    \Bigl\|\tfrac{1}{n}\sum_{i=1}^{n}b_i^{C}(x)\Bigr\|
  \;<\;\infty .
\end{equation}

\subsection{Per-round noise control}
The injected DP noise (Algorithm~\ref{alg:dp-rfedproj}, Lines 9--10) is
Gaussian and hence unbounded: with small but positive probability a single
draw ejects the pre-projection point from the tube $U_\M(2\gamma)$, where $\Pi$
ceases to be single-valued. We handle this by conditioning each round on a
high-probability event on which every Gaussian noise vector $\xi_{t,j}^{(i)}$ is
bounded, with the threshold calibrated so that the event fails only with small
probability. On the rare failure, the projection still returns a point of $\M$,
so the objective increases by at most $\delta_F$ over the
round.
\begin{definition}[Per-round noise control event]\label{def:good-event}
  Fix a noise-control margin $c_0 > 0$ and define the noise threshold
  \begin{equation}\label{eq:R-def}
    R := \sdp\bigl(\sqrt{D} + \sqrt{c_0}\bigr),
  \end{equation}
  together with the per-round noise control event
  \begin{equation}\label{eq:good-event}
    \mathcal{G}_t := \bigl\{\, \|\xi_{t,j}^{(i)}\| \le R
      \ \text{for all } (i,j)\in S_t\times\{0,\dots,\tau-1\} \,\bigr\}, \quad t \in \{0,\dots, T-1\}.
  \end{equation}
\end{definition}
\noindent
On $\mathcal{G}_t$, each privatised gradient~\eqref{eq:rec-noisy-grad} is bounded,
\begin{equation}\label{eq:grad-bound-geff}
  \|\tilde g_{t,j}^{(i)}\|\le C+\|\xi_{t,j}^{(i)}\|\le \Geff,
  \qquad \Geff:=C+R ,
\end{equation}
which is the bound the convergence analysis relies on. 
Whichever clients are selected, $\mathcal{G}_t$ constrains the same number
$k\tau$ of i.i.d.\ noise draws, so the probability of the complement
$\mathcal{G}_t^c$ does not depend on the round index $t$. Gaussian norm concentration at the
threshold~\eqref{eq:R-def}, followed by a union bound over the $k\tau$ draws,
gives the following, proved in Appendix~\ref{ssec:proof-pfail}.
\begin{lemma}[Per-round failure probability]\label{lem:pfail}
  The failure probability $\pi:=\Pr(\mathcal{G}_t^c)$ is the same for every round
  $t\in\{0,\dots,T-1\}$ and satisfies $\pi \le k\tau\,e^{-c_0/2}$.
\end{lemma}
The margin $c_0$ trades failure probability against the gradient bound: larger
$c_0$ shrinks $\Pr(\mathcal{G}_t^c)$ but inflates $R$, and with it the effective gradient
bound $\Geff=C+R$. It is a
free parameter, held fixed in the finite-$T$ bound (Theorem~\ref{thm:finite-T})
and tuned as $c_0\asymp\log T$ for the asymptotic rate (Corollary~\ref{cor:main}).

\subsection{Main results}

Our main guarantee is a finite-horizon stationarity bound for
Algorithm~\ref{alg:dp-rfedproj}. Because \textsc{ProjAvg} and \textsc{RLAvg} induce the same
server update to first order (see Lemma~\ref{lem:server-decomp}), a single analysis
covers both, and they attain identical rates. We state the result once for the two. Here and
below, $a\lesssim b$ means $a\le c\,b$ for a universal numerical constant $c$, made explicit
in Appendix~\ref{app:proof-main}.
\begin{theorem}
\label{thm:finite-T}
 Under Assumptions~\ref{ass:prox}--\ref{ass:hetero}, for every 
$T\ge 1$ and every constant stepsize $\alpha\le\min\{\gamma/(4\Geff),\,1/(4\LM)\}$, 
\begin{align}\label{eq:finite-T-rate}
  \frac{1}{T}\sum_{t=0}^{T-1}\E\bigl[\|\grad F(x_t)\|^2\bigr]
  &\;\lesssim\;
  \frac{F(x_0)-F_\star}{(1-\pi)\,\alpha\,T}
  + \frac{\LM\,\alpha}{k}\!\left[
    \frac{n-k}{n-1}(\sighet+\sigma_{\mathrm{clip}})^2
    + \frac{1}{\tau}\!\left(\bar\rho\,\sigma_{\mathrm{st}}^2 + \dM\sdp^2\right)
  \right] \notag\\
  &\quad
  + \beta_C^2
  + A\,\alpha^2
  + \frac{\pi\,\delta_F}{(1-\pi)\,\alpha}.
\end{align}
where $\bar\rho := \frac{1}{n}\sum_{i=1}^{n}\rho_i$ and the constant $A>0$ depends on the problem constants and on $c_0$.
\end{theorem}
The bound holds for any $\sdp$, in particular for the one calibrated to a target
$(\varepsilon,\delta)$ at horizon $T$.
\paragraph{Active clipping.} Theorem~\ref{thm:finite-T} holds for the clipped algorithm. Convergence analyses of private optimisation often omit clipping, or assume it never activates ($G\le C$)~\citep{huang2026federated}, and its bias is
studied separately~\citep{chen2020understanding,koloskova2023revisiting}.
Here clipping enters through two explicit constants, the dispersion
$\sigma_{\mathrm{clip}}$ of the per-client biases~\eqref{eq:sigma-clip-def} and
the bias floor $\beta_C$~\eqref{eq:beta-C-def}. Both vanish when $C\ge G$, which recovers
the unclipped bound.

\paragraph{Local steps and partial participation.} Unlike prior
analyses~\citep{li2022federated,huang2026federated}, the bound holds for any number $\tau$
of local steps and any number $k\le n$ of participating clients at once, the regime
Limitation~\ref{lim:agg-reach} left open. Both are reached together through a linearisation
of the local trajectory by the projection (Section~\ref{ssec:proof-outline}, Step~2), which
avoids the curvature-dependent constants of intrinsic aggregation.

\paragraph{Intrinsic and ambient dimensions.} Two dimensions enter~\eqref{eq:finite-T-rate}
at different orders: the intrinsic manifold dimension $\dM$ sits in the leading DP-noise
term, the ambient dimension only in the remainder $A\alpha^2$ and the stepsize ceiling,
through $\Geff$~\eqref{eq:R-def}. This is a direct benefit of the design. The noise is drawn in the ambient space but each tangent projection discards its normal component, so the leading privacy cost is that of a $\dM$-dimensional problem.

Theorem~\ref{thm:finite-T} holds for any fixed margin $c_0$ and letting $c_0$ grow with the
horizon yields the asymptotic rate of Corollary~\ref{cor:main}. 
\begin{corollary}
\label{cor:main}
Under Assumptions~\ref{ass:prox}--\ref{ass:hetero}, for $T$ large enough,
\begin{equation}\label{eq:main-rate}
\frac{1}{T}\sum_{t=0}^{T-1}\E\bigl[\|\grad F(x_t)\|^2\bigr]=\;O\left(
  \frac{1}{\sqrt{T}}\!\left[
    1 + \frac{1}{k}\!\left(
      \frac{n-k}{n-1}\,(\sighet+\sigma_{\mathrm{clip}})^2
      + \frac{\bar\rho\,\sigma_{\mathrm{st}}^2 + \dM\,\sdp^2}{\tau}
    \right)
  \right]
  + \beta_C^{\,2}\right).
\end{equation}
\end{corollary}
\begin{proof}
    Taking
$\alpha\asymp 1/\sqrt{T}$ collapses the optimisation and noise terms into 
$O(1/\sqrt{T})$. The choice
$c_0=4\log T+2\log(k\tau)$ then forces $\pi\le T^{-2}$ (Lemma \ref{lem:pfail}), so the bad-round term is
$O(T^{-3/2})$. This margin leaves the heterogeneity, stochastic-gradient,
DP-noise, and clipping-bias terms untouched, inflating only (through $\Geff$) the subdominant remainder $A\alpha^2$ to $o(1/\sqrt{T})$. Absorbing these into the $O(\cdot)$ gives the  rate, valid for $T$ large
enough for the stepsize constraint to be satisfied.
\end{proof}
Corollary~\ref{cor:main} keeps $\sdp$ fixed as $T$ grows. A fixed budget
$(\varepsilon,\delta)$ needs more noise for longer runs, with $\sdp\propto\sqrt{T\tau}$ in the
regime of~\eqref{eq:mu-tot-expansion}.
\paragraph{Interpretation.}
The bound splits into a vanishing part and a
$T$-independent floor $\beta_C^2$, so the iterates converge to a
$\beta_C$-neighbourhood of stationarity rather than to a stationary point unless
$C\ge G$. Term by term:
\begin{enumerate}[label=(\roman*),leftmargin=*,itemsep=3pt]
  \item \emph{Optimisation error}, $O(1/\sqrt T)$: the standard nonconvex stationarity
    rate, independent of $k$, $\tau$, and the privacy parameters.
  \item \emph{Client heterogeneity}, $(\sighet+\sigma_{\mathrm{clip}})^2$, scaled by
    $\tfrac1k\cdot\tfrac{n-k}{n-1}$: the price of partial participation. It
    combines the intrinsic gradient dissimilarity $\sighet$ (Assumption~\ref{ass:hetero})
    with the clipping-induced dispersion $\sigma_{\mathrm{clip}}$ of the per-client
    biases. The factor $\tfrac{n-k}{n-1}$ is the
    finite-population sampling correction: it vanishes at full participation
    ($k=n$) and is largest when a single client is sampled ($k=1$).
  \item \emph{Stochastic-gradient noise}, $\bar\rho\,\sigma_{\mathrm{st}}^2/(k\tau)$:
    the minibatch variance, averaged over the cohort ($1/k$) and the local
    trajectory ($1/\tau$). The factor
    $\bar\rho=\tfrac1n\sum_i\tfrac{m_i-B}{B(m_i-1)}$ is the without-replacement
    sampling factor (Lemma~\ref{lem:clip-var}); it decreases as $B$ grows and
    vanishes at full local batches ($B=m_i$), removing this term entirely.
  \item \emph{DP noise}, $\dM\,\sdp^2/(k\tau)$: the dominant privacy cost, damped by
    $1/(k\tau)$ at fixed $\sdp$. It shrinks only with $\sdp$, hence with weaker privacy. 
  \item \emph{Clipping-bias floor}, $\beta_C^2$: the systematic error from
    gradient clipping (Eq.~\eqref{eq:beta-C-def}), the only term that does not
    decay with $T$, $k$, or $\tau$. It is zero when the clip is inactive
    ($C\ge G$) and otherwise a fixed positive floor. Enlarging $C$ shrinks it but raises the noise scale needed for
    a given privacy level, exposing a clip--noise trade-off.
\end{enumerate}

\subsection{Outline of the proof of Theorem~\ref{thm:finite-T}}\label{ssec:proof-outline}
The proof has three steps, each stated as a lemma and
proved in Appendix~\ref{app:proof-main}. The new ingredient is the second step.

\paragraph{Step 1: tube stability.} On $\mathcal{G}_t$, each projection step displaces its
argument by at most $\eta\Geff$~\eqref{eq:grad-bound-geff}. Under
$\alpha\le\gamma/(4\Geff)$, every local and server pre-projection point thus stays in the
inner tube $U_\M(\gamma)$, where $\Pi$ is single-valued, and every local iterate stays
within $2\Geff\alpha$ of $x_t$ (Lemma~\ref{lem:tube-stability}).

\paragraph{Step 2: linearisation of the local trajectory.}
The exponential-map scheme of~\citet{li2022federated} and the retraction scheme
of~\citet{huang2026federated} represent each local step as a tangent vector at the
\emph{current} local iterate,
\[
  \mathrm{Log}_{x_{t,j}^{(i)}}\!\bigl(x_{t,j+1}^{(i)}\bigr)
  \;\Bigl(\text{resp. } \mathrm{R}^{-1}_{x_{t,j}^{(i)}}\!\bigl(x_{t,j+1}^{(i)}\bigr)\Bigr)
  = -\eta\,v_{t,j}^{(i)}\in T_{x_{t,j}^{(i)}}\M ,
\]
with $v_{t,j}^{(i)}$ the local update direction. Because these maps are nonlinear and
anchored at a base point that moves at every step, the $\tau$ increments live in $\tau$
different tangent spaces. Reconciling them requires transport, a curvature-dependent
argument~\citep{zhang2016riemannian} that prior analyses carry out only when a single client
participates ($k=1$).

The nearest-point projection removes this obstruction through one property: its first-order
behaviour at every base point is the tangent projector. There is a constant $\Cqr>0$,
depending only on $\gamma$ and the second-order geometry of $\M$,
such that~\citep[Lem.~4.3]{deng2025decentralized}
\begin{equation}\label{eq:proj-quad-rem}
  \Pi(x+u)=x+P_{x}\,u+w(x,u),
  \qquad \|w(x,u)\|\le\Cqr\|u\|^2,
  \qquad x\in\M,\ \|u\|\le\gamma .
\end{equation}
A local step~\eqref{eq:sgd} then moves the iterate by an ambient displacement
$x_{t,j+1}^{(i)}-x_{t,j}^{(i)}\in\mathcal{E}$, not by a tangent vector pinned to
$x_{t,j}^{(i)}$. All displacements lie in the single space $\mathcal{E}$, so the $\tau$
local steps telescope to $x_{t,\tau}^{(i)}-x_t$. Replacing each local projector
$P_{x_{t,j}^{(i)}}$ by the common $P_{x_t}$ costs $O(\eta\alpha)$ per step by Step~1, hence
$O(\alpha^2)$ over the round, with no $\tau$-dependent penalty. Each client's displacement
is thus $-\alpha\,\gradC f_i(x_t)$ plus stochastic-gradient and DP noise projected at $x_t$,
up to $O(\alpha^2)$ (Lemma~\ref{lem:linearisation}).

\paragraph{Step 3: server update and descent.}
Each submission is now an ambient displacement from the common point $x_t$, so aggregation
reduces to an arithmetic mean of displacements rather than a composition of nonlinear maps
over the sampled clients. Applying~\eqref{eq:proj-quad-rem} at the server, \textsc{ProjAvg}
and \textsc{RLAvg} give the same update to first order,
\begin{equation*}
  x_{t+1}-x_t
  = -\alpha\,\tfrac1k\textstyle\sum_{i\in S_t}\gradC f_i(x_t)
    + \zeta_t^{\mathrm{st}} + \zeta_t^{\mathrm{dp}} + O(\alpha^2),
\end{equation*}
with $\zeta_t^{\mathrm{st}}$ and $\zeta_t^{\mathrm{dp}}$ the stochastic-gradient and DP noise
averaged over the cohort and projected by $P_{x_t}$ (Lemma~\ref{lem:server-decomp}). The
rank-$\dM$ projector is what makes the DP variance scale with $\dM$ rather than
$D$. Plugging this decomposition into the ambient descent
inequality~\eqref{eq:descent-ineq}, which bypasses the exponential, logarithmic and
inverse-retraction maps, gives a one-round descent on $\mathcal{G}_t$
(Lemmas~\ref{lem:moments} and~\ref{lem:descent-good}). On the complement $\mathcal{G}_t^c$,
of probability $\pi$, the iterate stays in $\M$ and the objective rises by at most
$\delta_F$. Combining the two cases and telescoping over the $T$ rounds yields
Theorem~\ref{thm:finite-T}.

\section{Synthetic experiments}\label{sec:synth-exp}

We first illustrate our results on federated principal component analysis (PCA) on the
Grassmann manifold~\citep{nguyen2024federated}, which provides a controlled testbed with a
closed-form solution. The goal is
to find collaboratively the $p$-dimensional subspace that keeps the most variance of all
clients' records when the records are projected onto it. A subspace is a point $P$ of the
Grassmann manifold $\mathrm{Gr}(d,p)$, of dimension $\dM=p(d-p)$, represented by its
orthogonal projector in the ambient space $\mathcal E=\Sym(d)$, of dimension
$\dimE=d(d+1)/2$ (Appendix~\ref{app:grassmann}).

\subsection{Problem setting}

We consider $n$ clients, each holding $m$ records $z\in\R^d$. The records follow a spiked
covariance model~\citep{johnstone2001distribution}, in which client $i$ draws
\begin{equation}\label{eq:spiked}
  z\sim\N(0,\Sigma_i),\qquad \Sigma_i=U_i\,\mathrm{diag}(\lambda)\,U_i^\top+\frac{\nu^2}{d}I_d,
  \qquad U_i\in\mathrm{St}(d,p),
\end{equation}
then projects $z$ onto the unit ball, so that $\|z\|\le1$. The matrix $U_i$ spans the
signal subspace of client $i$, a perturbation of a $p$-dimensional subspace common to all
clients. The signal variances $\lambda\in\R^p$ and the noise energy $\nu^2$ are fixed
(Appendix~\ref{app:synth-data}).

The clients minimise
\[
  \min_{P\in\mathrm{Gr}(d,p)}\Bigl\{F(P):=\frac1n\sum_{i=1}^n\frac1m\sum_{\ell=1}^m
  f(P;z_{i,\ell})\Bigr\},\qquad f(P;z):=-\frac12z^\top Pz .
\]
Up to a constant, the loss $f(P;z)$ is the reconstruction error $\frac12\|z-Pz\|^2$. This
problem has a classical closed-form solution $F^\star=F(P^\star)$~\citep{fan1949theorem}, where $P^\star$ is the projector onto the top-$p$ eigenvectors of
$\bar\Sigma=\frac1{nm}\sum_{i,\ell}z_{i,\ell}z_{i,\ell}^\top$ (Appendix~\ref{app:synth-solution}). Comparing the output of the algorithms
with this optimum lets us check the properties of our analysis.

We run \textsc{DP-RFedProj} (Algorithm~\ref{alg:dp-rfedproj}) with Riemannian stochastic gradient descent
(RSGD), multiple local steps, partial participation and our two aggregation schemes
\textsc{ProjAvg} and \textsc{RLAvg}. The clip constant $C$ is the smallest value for which
clipping is inactive for every record in the unit ball (Appendix~\ref{app:synth-tuning}). The noise of \textsc{DP-RFedProj} is calibrated with
our accounting (Theorem~\ref{thm:gdp-subsampled}). As a baseline, we consider PriRFed~\citep{huang2026federated},
with the same RSGD local steps and \textsc{RetAvg} aggregation. PriRFed calibrates the
noise with its round composition~\citep[Thm.~7]{huang2026federated} fed with our per-round privacy profile. Each method uses its own constant local stepsize $\eta$,
tuned at each privacy budget (Appendix~\ref{app:synth-tuning}). Table~\ref{tab:synth-settings} summarises the
common settings.

\begin{table}[h]
\centering
\caption{Common settings. $k$ and $\varepsilon$ take these values when they are not varied.}
\label{tab:synth-settings}
\small
\begin{tabular}{llll}
\toprule
Clients & $n=50$ & Clients per round & $k=10$ \\
Records per client & $m=5000$ & Local steps, batch size & $\tau=10$, $B=500$ \\
Rounds & $T=200$ & Privacy & $\varepsilon=1$, $\delta=10^{-5}$ \\
\bottomrule
\end{tabular}
\end{table}

\subsection{Reported metrics}

For a privacy budget $\varepsilon$, we report the optimality gap $F(P_T)-F^\star$ of the last
iterate $P_T$ of Algorithm~\ref{alg:dp-rfedproj}, with noise calibrated to $\varepsilon$. Since
$F^\star$ is known in closed form, this gap is exact and easier to read than the squared
gradient norm, and near the optimum the two quantities behave alike. Let $P_T^0$ be the last
iterate of the same run without noise. The cost of privacy is the difference of the two gaps, $F(P_T)-F(P_T^0)$.

\subsection{Analysis}

Our convergence bound (Theorem~\ref{thm:finite-T}) charges privacy through a term proportional to
$\eta\sdp^2\dM/k$. We check this dependence empirically, by examining separately the role
of (i)~the manifold dimension against the ambient dimension, (ii)~the number $k$ of
clients per round, and (iii)~the accounting. For (iii), we also run \textsc{DP-RFedProj}
with \textsc{RetAvg}, PriRFed's aggregation, so that the two methods differ only in the
accounting and in their tuned stepsize.

\begin{figure}[!t]
\centering
\providecolor{C0}{HTML}{1B6CA8}
\providecolor{C1}{HTML}{E1701A}
\providecolor{C2}{HTML}{2C8B57}
\providecolor{C3}{HTML}{A02C2C}
\providecolor{C7}{HTML}{444444}
\begin{tikzpicture}
\begin{groupplot}[
  group style={group size=2 by 2, horizontal sep=1.6cm, vertical sep=1.9cm},
  width=0.47\linewidth, height=4.4cm,
  tick label style={font=\scriptsize}, label style={font=\small},
  title style={font=\small},
  grid=major, grid style={black!10},
  legend style={font=\small, draw=none, legend columns=-1, column sep=1.2ex},
  legend cell align=left,
]
\nextgroupplot[title={(a) $D=820$, $d_{\mathcal M}$ varies}, xlabel={manifold dimension $d_{\mathcal M}$}, ylabel={cost of privacy}, ymin=0.0001, ymax=0.01, xtick={39,76,144,256}, xticklabels={39,76,144,256}, minor xtick={}, xmode=log, ymode=log, legend to name=synthlegend]
\addplot+[C0, thick, mark=*, mark size=1.6pt, forget plot, error bars/.cd, y dir=both, y explicit] coordinates {
  (37.44,0.000124039) += (0,4.92653e-06) -= (0,1.53099e-05)
  (72.96,0.00024342) += (0,2.36486e-05) -= (0,1.309e-05)
  (138.24,0.000417959) += (0,3.83074e-05) -= (0,3.62941e-05)
  (245.76,0.000755688) += (0,4.77513e-05) -= (0,2.50718e-05)
};
\addplot+[C2, thick, mark=square, mark size=1.8pt, forget plot, error bars/.cd, y dir=both, y explicit] coordinates {
  (40.56,0.000123412) += (0,4.91248e-06) -= (0,1.52002e-05)
  (79.04,0.000242144) += (0,2.34674e-05) -= (0,1.31834e-05)
  (149.76,0.00041563) += (0,3.7797e-05) -= (0,3.61347e-05)
  (266.24,0.000751608) += (0,4.70867e-05) -= (0,2.51254e-05)
};
\addplot[C7, thick, dashed, forget plot] coordinates {
  (39,0.000457475)
  (76,0.00089149)
  (144,0.00168914)
  (256,0.00300291)
};
\addlegendimage{C0, thick, mark=*, mark size=1.6pt}
\addlegendentry{ProjAvg}
\addlegendimage{C2, thick, mark=square, mark size=1.8pt}
\addlegendentry{RLAvg}
\addlegendimage{C1, thick, mark=diamond*, mark size=2.2pt}
\addlegendentry{RetAvg}
\addlegendimage{C3, thick, mark=triangle*, mark size=2pt}
\addlegendentry{PriRFed}
\addlegendimage{C7, thick, dashed}
\addlegendentry{$\eta\,\sigma_{\mathrm{dp}}^2 d_{\mathcal M}/k$}
\nextgroupplot[title={(b) $d_{\mathcal M}=144$, $D$ varies}, xlabel={ambient dimension $D$}, ylabel={cost of privacy}, ymin=0.0001, ymax=0.01, xtick={351,820,2775,10585}, xticklabels={351,820,2775,10585}, minor xtick={}, xmode=log, ymode=log]
\addplot+[C0, thick, mark=*, mark size=1.6pt, forget plot, error bars/.cd, y dir=both, y explicit] coordinates {
  (336.96,0.000402334) += (0,5.35258e-05) -= (0,1.40025e-05)
  (787.2,0.000417959) += (0,3.83074e-05) -= (0,3.62941e-05)
  (2664,0.00041983) += (0,2.06084e-05) -= (0,2.21346e-05)
  (10161.6,0.000393981) += (0,4.12858e-05) -= (0,2.74711e-05)
};
\addplot+[C2, thick, mark=square, mark size=1.8pt, forget plot, error bars/.cd, y dir=both, y explicit] coordinates {
  (365.04,0.000400748) += (0,5.33876e-05) -= (0,1.38319e-05)
  (852.8,0.00041563) += (0,3.7797e-05) -= (0,3.61347e-05)
  (2886,0.000416128) += (0,1.98075e-05) -= (0,2.23756e-05)
  (11008.4,0.000386516) += (0,3.96991e-05) -= (0,2.77251e-05)
};
\addplot[C7, thick, dashed, forget plot] coordinates {
  (351,0.00168914)
  (820,0.00168914)
  (2775,0.00168914)
  (10585,0.00168914)
};
\nextgroupplot[title={(c) clients per round}, xlabel={$k$ (out of $n=50$)}, ylabel={cost of privacy}, ymin=1e-05, ymax=0.1, xtick={1,2,5,10,25,50}, xticklabels={1,2,5,10,25,50}, minor xtick={}, xmode=log, ymode=log]
\addplot+[C0, thick, mark=*, mark size=1.6pt, forget plot, error bars/.cd, y dir=both, y explicit] coordinates {
  (0.96,0.00453728) += (0,0.000344605) -= (0,0.0003548)
  (1.92,0.00211968) += (0,0.000289343) -= (0,0.00016692)
  (4.8,0.000839038) += (0,0.000120357) -= (0,3.70259e-05)
  (9.6,0.000417959) += (0,3.83074e-05) -= (0,3.62941e-05)
  (24,0.000170884) += (0,1.54873e-05) -= (0,2.11385e-05)
  (48,8.47599e-05) += (0,7.41346e-06) -= (0,4.03776e-06)
};
\addplot+[C2, thick, mark=square, mark size=1.8pt, forget plot, error bars/.cd, y dir=both, y explicit] coordinates {
  (1.04,0.00449061) += (0,0.000342667) -= (0,0.000354635)
  (2.08,0.00210322) += (0,0.00028613) -= (0,0.000166802)
  (5.2,0.00083384) += (0,0.000119762) -= (0,3.6751e-05)
  (10.4,0.00041563) += (0,3.7797e-05) -= (0,3.61347e-05)
  (26,0.000169938) += (0,1.54484e-05) -= (0,2.08716e-05)
  (52,8.43702e-05) += (0,7.34383e-06) -= (0,4.05537e-06)
};
\addplot[C7, thick, dashed, forget plot] coordinates {
  (1,0.0168914)
  (2,0.00844569)
  (5,0.00337828)
  (10,0.00168914)
  (25,0.000675655)
  (50,0.000337828)
};
\nextgroupplot[title={(d) DP-RFedProj vs PriRFed}, xlabel={$\varepsilon$}, ylabel={optimality gap}, ymin=1e-05, ymax=0.1, xtick={0.5,1,2,5,10}, xticklabels={0.5,1,2,5,10}, minor xtick={}, xmode=log, ymode=log]
\addplot+[C0, thick, mark=*, mark size=1.6pt, forget plot, error bars/.cd, y dir=both, y explicit] coordinates {
  (0.47,0.00150521) += (0,7.60481e-05) -= (0,0.00011772)
  (0.94,0.000442605) += (0,4.49597e-05) -= (0,3.88946e-05)
  (1.88,0.000146742) += (0,2.84337e-05) -= (0,1.20247e-05)
  (4.7,4.87641e-05) += (0,1.81683e-05) -= (0,2.89947e-06)
  (9.4,2.9942e-05) += (0,8.25788e-06) -= (0,1.69174e-06)
};
\addplot+[C2, thick, mark=square, mark size=1.8pt, forget plot, error bars/.cd, y dir=both, y explicit] coordinates {
  (0.5,0.00147775) += (0,7.39112e-05) -= (0,0.000113871)
  (1,0.000440274) += (0,4.4441e-05) -= (0,3.87311e-05)
  (2,0.000146697) += (0,2.8112e-05) -= (0,1.22102e-05)
  (5,4.87354e-05) += (0,1.8166e-05) -= (0,2.8957e-06)
  (10,2.99309e-05) += (0,8.35827e-06) -= (0,1.68964e-06)
};
\addplot+[C1, thick, mark=diamond*, mark size=2.2pt, forget plot, error bars/.cd, y dir=both, y explicit] coordinates {
  (0.53,0.00153313) += (0,7.54996e-05) -= (0,0.000124573)
  (1.06,0.000444862) += (0,4.52218e-05) -= (0,3.95667e-05)
  (2.12,0.000146808) += (0,2.84578e-05) -= (0,1.18836e-05)
  (5.3,4.87871e-05) += (0,1.81753e-05) -= (0,2.89805e-06)
  (10.6,2.99526e-05) += (0,8.21985e-06) -= (0,1.69281e-06)
};
\addplot+[C3, thick, mark=triangle*, mark size=2pt, forget plot, error bars/.cd, y dir=both, y explicit] coordinates {
  (0.5,0.0313886) += (0,0.00123298) -= (0,0.000968755)
  (1,0.0076654) += (0,0.00315445) -= (0,0.000826118)
  (2,0.00392106) += (0,3.95631e-05) -= (0,0.000364305)
  (5,0.000804543) += (0,6.07689e-05) -= (0,6.5501e-05)
  (10,0.000289711) += (0,3.74729e-05) -= (0,2.59744e-05)
};
\end{groupplot}
\path (group c1r2.outer south) -- (group c2r2.outer south)
  node[midway, anchor=north, yshift=-3pt] {\pgfplotslegendfromname{synthlegend}};
\end{tikzpicture}
\caption{Effects of manifold dimension against ambient dimension, client participation and accounting. Median and interquartile range over 10 seeds. Markers are shifted slightly where values coincide.}
\label{fig:synth}
\end{figure}
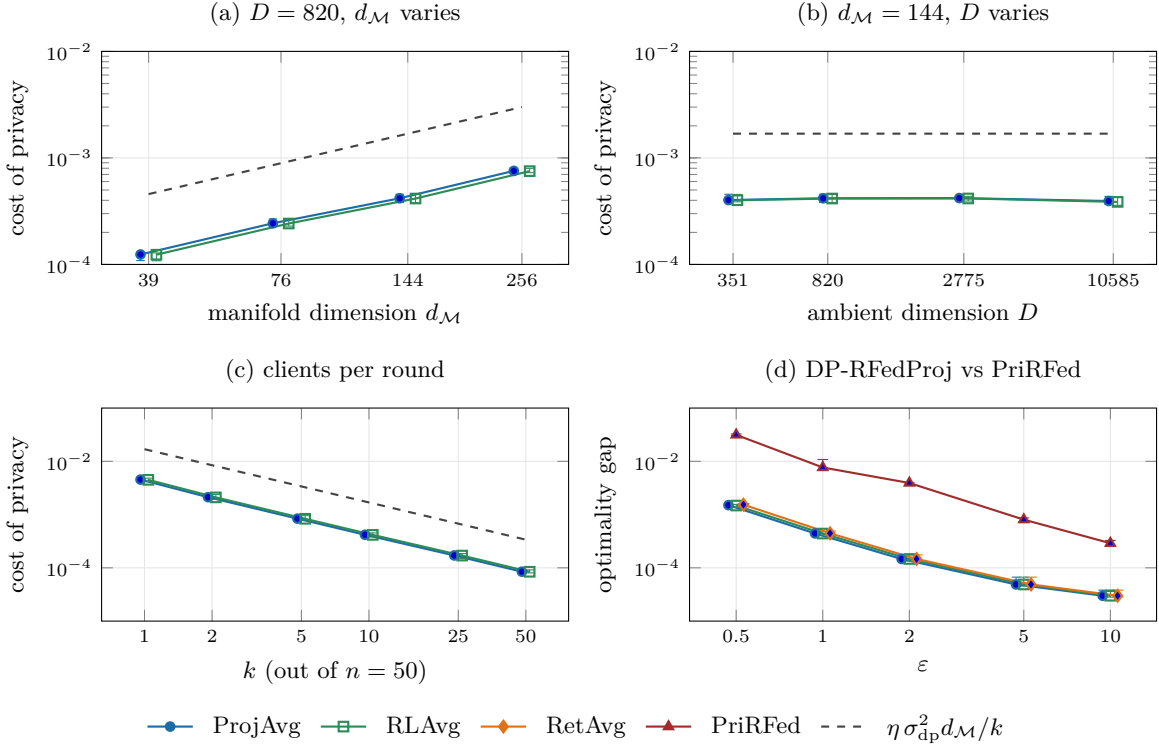

\paragraph{Manifold dimension.} In Figure~\ref{fig:synth}(a), the ambient dimension is fixed and the
cost of privacy grows with $\dM$, parallel to $\eta\sdp^2\dM/k$. In
Figure~\ref{fig:synth}(b), the manifold dimension is fixed and the cost stays flat while $\dimE$ grows. In this experiment, the cost of privacy depends on the manifold dimension, and the
ambient dimension has no visible effect. \textsc{ProjAvg} and \textsc{RLAvg} give the same
cost, within $2\,\%$.

\paragraph{Clients per round.} Figure~\ref{fig:synth}(c) shows that the cost of privacy
follows the $1/k$ dependence of the bound. It falls by a factor of about 53 from $k=1$ to full
participation, $k=50$, for both aggregation schemes.

\paragraph{Accounting.} At equal $(\varepsilon,\delta)$, PriRFed's accounting needs more
than 5 times as much noise as ours. At $\varepsilon=1$, it sets $\sdp=0.131$, against
$0.024$ for ours. Under the same aggregation, \textsc{RetAvg}, PriRFed's optimality gap is
about 10 to 27 times larger than that of \textsc{DP-RFedProj} (Figure~\ref{fig:synth}(d)). With
our accounting, the three aggregation schemes reach the same optimality gap, within
$4\,\%$. In this experiment, the tightness of the accounting drives the accuracy, and the
aggregation scheme plays a minor role. This favours the computationally cheapest scheme, \textsc{ProjAvg}.

\section{EEG experiments}\label{sec:num-exps}

We evaluate the framework on real EEG data, through motor-imagery classification, a setting in which data
confidentiality motivates private federated learning and in which geometry is central:
spatial covariance matrices are highly effective EEG
descriptors~\citep{barachant2012multiclass}, motivating methods that respect their SPD
structure. We compare SPDNet~\citep{huang2017riemannian}, a lightweight Riemannian architecture operating on covariance matrices, with EEGNet~\citep{lawhern2018eegnet}, a Euclidean convolutional network operating on raw signals. EEGNet's parameters are Euclidean, so it is trained by Algorithm~\ref{alg:dp-rfedproj} with $\M=\mathcal{E}$, which reduces to standard record-level private federated averaging in Euclidean space (Section~\ref{ssec:scope}).

Each EEG record is a labelled trial $z = (X, y)$ with $y \in \{1, \dots, n_{\mathrm{cl}}\}$. Client $i$ holds the $m_i$ trials of a fixed set of subjects and never shares them. For each model, the goal is to solve collaboratively
\begin{equation}\label{eq:fl-objective-exp}
  \min_{x \in \M} \Bigl\{\, F(x) := \frac{1}{n}\sum_{i=1}^{n} f_i(x) \,\Bigr\},
  \qquad
  f_i(x) := \frac{1}{m_i}\sum_{\ell=1}^{m_i} f(x; z_{i,\ell}),
  \qquad
  f(x; z) = -\log\bigl[h_x(z)\bigr]_y ,
\end{equation}
where $h_x(z)$ is the vector of class probabilities the model with parameters $x$ assigns
to $z$, making $f$ the cross-entropy loss. The two model architectures differ only in the space of learnable parameters $\M$. For SPDNet, $\M$ is the product of a Stiefel manifold and a Euclidean space, on which
$\Pi$ acts factorwise.

Sections~\ref{ssec:centralized} and~\ref{ssec:fl} describe the centralized baselines and
the federated protocol; Sections~\ref{ssec:dp} to~\ref{ssec:prirfed} report the
privacy--utility trade-off, the aggregation rules and the comparison with PriRFed.

\subsection{Datasets and preprocessing}\label{ssec:datasets}

We use three datasets from the MOABB benchmark~\citep{chevallier2024largest}, summarized in
Table~\ref{tab:datasets}. In order to span different regimes, we consider low to medium electrode count, subject size and classes.

\begin{table}[h]
\centering
\caption{The three EEG motor-imagery datasets, with the BiMap dimension $d$ and the ReEig
floor $\varepsilon_{\mathrm{R}}$ of SPDNet (Section~\ref{ssec:architectures}).}
\label{tab:datasets}
\small
\begin{tabular}{lcccccccc}
\toprule
& & & & & & & \multicolumn{2}{c}{SPDNet} \\
\cmidrule(lr){8-9}
Dataset & $n_{\text{subj}}$ & $n_{\text{chan}}$ & $n_{\text{cl}}$
        & Native S.R.\ (Hz) & $\mathcal{T}$ & Trials/class & $d$ & $\varepsilon_{\mathrm{R}}$ \\
\midrule
BNCI2014\_001 & 9  & 22 & 4 & 250 & 513 & 144 & 16 & $8\cdot 10^{-5}$ \\
Weibo2014     & 10 & 60 & 7 & 200 & 513 & 80  & 20 & $2\cdot 10^{-5}$ \\
Cho2017       & 52 & 64 & 2 & 512 & 384 & 100 & 16 & $6\cdot 10^{-5}$ \\
\bottomrule
\end{tabular}
\end{table}

All signals are band-pass filtered to $[8, 32]$\,Hz and resampled from their native rate to
$f_s = 128$\,Hz, leaving $\mathcal{T}$ time samples per trial. Trial counts are per subject
and within a dataset subjects contribute comparable numbers. In particular, clients hold
comparable amounts of data. Each trial is then normalized on its own, either with per-channel
standardization for EEGNet or unit-trace covariance for SPDNet, which removes the recording
gain and keeps every trial independent of the rest.
\subsection{Model architectures}\label{ssec:architectures}

\paragraph{EEGNet.} EEGNet~\citep{lawhern2018eegnet} is a compact convolutional network
operating directly on raw trials (see Fig.~\ref{fig:eegnet} in Appendix~\ref{app:experiments}). Its first block applies a
temporal convolution, then a depthwise spatial convolution that learns electrode
combinations for each temporal filter; its second block refines the result with a
separable convolution. Each block ends with a batch normalization, an ELU nonlinearity,
average pooling and dropout, and a linear layer maps the flattened features to class
scores. Batch normalization in training mode is incompatible with record-level privacy
(Appendix~\ref{app:bn-pb}), so we remove it in every run. EEGNet below denotes this variant.

\paragraph{SPDNet.} SPDNet~\citep{huang2017riemannian} operates on the per-trial sample
covariance $\Sigma_j = (\mathcal{T}-1)^{-1} \bar{X}_j \bar{X}_j^{\top} \in
\mathcal{S}^{++}_{n_{\mathrm{chan}}}$, with $\bar{X}_j$ the centered trial (see Fig.~\ref{fig:spdnet} in Appendix~\ref{app:experiments}). A BiMap layer $\Sigma \mapsto W^{\top} \Sigma W$, with
$W \in \mathrm{St}(n_{\mathrm{chan}}, d)$, reduces dimension while preserving SPD
structure; a ReEig layer rectifies eigenvalues below a floor $\varepsilon_{\mathrm{R}}$; a
LogEig layer $U \Lambda U^{\top} \mapsto U \log(\Lambda) U^{\top}$ maps to the tangent
space at the identity; and a softmax classifier acts on the $d(d+1)/2$ half-vectorized
features. 

Two hyperparameters are specific to SPDNet: the BiMap dimension $d$ and the ReEig floor
$\varepsilon_{\mathrm{R}}$. The dimension $d$ is chosen so that SPDNet has fewer parameters
than EEGNet on every dataset (Table~\ref{tab:fed-setting}). This rules out capacity as the source of any advantage of
SPDNet below. The floor $\varepsilon_{\mathrm{R}}$ is the $15$th percentile of the
trace-normalized covariance spectra of the first three subjects of each dataset.  Both values are listed in
Table~\ref{tab:datasets}.

\paragraph{Training and evaluation protocol.}
The trials of each dataset are split once $75/10/15$ into training, validation and test
sets, stratified by class. Each client keeps the trials of its own subjects, so the
centralized and federated runs of a seed use the same split. 
Both models are trained with Adam, with a cosine schedule from $\eta_0$ to $\eta_0/100$
over $T$ epochs (centralized) or rounds (federated). For SPDNet, each Adam step is followed
by the polar projection $\Pi$ of the BiMap weight. The hyperparameters $(\eta_0, B, T)$ are selected on validation accuracy, with
$T$ fixed to its non-private value under privacy (Appendix~\ref{app:grids}). We report the
final model on the test set, pooling the test trials of all subjects, as a mean over $10$ seeds that each redraw the split and the
initialization.

\subsection{Centralized non-private baselines}\label{ssec:centralized}
We first establish a centralized baseline without privacy, representing unrestricted data
sharing: all subjects' recordings are pooled and trained on under the settings of
Table~\ref{tab:config-central} (Appendix~\ref{app:grids}). These runs upper-bound what the
federated and private runs can reach. SPDNet leads on two of the
three datasets, and EEGNet leads on Cho2017 (Table~\ref{tab:centralized-results}).

\begin{table}[t]
\centering
\caption{Centralized non-private test accuracy (\%), mean $\pm$ standard deviation over 10 seeds, last epoch. Best per dataset in bold.}
\label{tab:centralized-results}
\small
\begin{tabular}{lcrr}
\toprule
Dataset & Chance & SPDNet & EEGNet \\
\midrule
BNCI2014\_001 & $25.0$ & $\mathbf{56.5 \pm 1.8}$ & $53.7 \pm 1.6$ \\
Cho2017 & $50.0$ & $63.9 \pm 1.0$ & $\mathbf{67.4 \pm 1.2}$ \\
Weibo2014 & $14.3$ & $\mathbf{49.7 \pm 1.0}$ & $44.2 \pm 1.9$ \\
\bottomrule
\end{tabular}
\end{table}

\subsection{Federated protocol}\label{ssec:fl}

\paragraph{Setup.}
Each client holds the recordings of a fixed group of subjects: three subjects per client on
BNCI2014\_001, hence $n=3$ clients, and two per client on Weibo2014 and Cho2017, hence
$n=5$ and $n=26$ clients (Table~\ref{tab:fed-setting} in Appendix~\ref{app:experiments}).
Both models keep their centralized architecture (Table~\ref{tab:datasets}). Each client runs
one local epoch per round. The privacy accounting charges every client
$\tau=\lceil m_{\max}/B\rceil$ local steps, those of the largest client, with $q=B/m_\star$ from
the smallest one (Appendix~\ref{app:dp-settings}). The clipping threshold is fixed a priori
at $C=1$~\citep{de2022unlocking}.

\paragraph{Aggregation.} EEGNet's parameters are all Euclidean and averaged arithmetically
(FedAvg). For SPDNet, the Stiefel weight of the BiMap layer is
aggregated by one of \textsc{ProjAvg}, \textsc{RLAvg} or \textsc{RetAvg}, its Euclidean classifier weights
arithmetically in every case. \textsc{RLAvg} uses the settings selected for
\textsc{ProjAvg}, and \textsc{RetAvg} the same batch size with its own learning rate.

\paragraph{Participation.} We study full participation and half participation, in which
$\lceil n/2\rceil$ clients are sampled per round (2 of 3, 3 of 5 and 13 of 26). At
full participation, we run EEGNet and SPDNet with each of the three Riemannian
aggregation rules. At half participation, we run EEGNet and SPDNet with
\textsc{ProjAvg} alone, with the settings of full participation, including $T$.

\paragraph{Privacy budgets.} We report a non-private run and three budgets,
$\varepsilon \in \{20, 10, 5\}$ at $\delta = 10^{-5}$, for both architectures.
These budgets are large because each client holds few records (Table~\ref{tab:fed-setting}). They serve
as a proof of concept to compare architectures and aggregation rules, not as tight guarantees. For each budget,
$\sdp$ is the smallest noise scale that meets the target $\varepsilon$ under a numerical privacy-loss-distribution accountant that returns a certified upper bound for Theorem~\ref{thm:gdp-subsampled}.
The reported $\varepsilon$ does not include the selection of
the hyperparameters on the data. This cost is often ignored in
practice, although non-private tuning can leak information~\citep{papernot2022hyperparameter}.
\subsection{Privacy--utility trade-off}\label{ssec:dp}
Table~\ref{tab:final}  reports the results and Figure~\ref{fig:curves} the learning curves. Without privacy, SPDNet leads on BNCI2014\_001 and Weibo2014 and EEGNet on Cho2017 as in the centralized runs. Under privacy, SPDNet leads in
every cell. With \textsc{ProjAvg}, it keeps $10$ to $61\%$ of its non-private accuracy above chance, against at
most $18\%$ for EEGNet.
\begin{table}[t]
\centering
\caption{Federated test accuracy (\%) at full client participation, mean $\pm$ standard deviation over 10 seeds, last round, $\delta = 10^{-5}$. Best per column within a dataset in bold.}
\label{tab:final}
\small
\begin{tabular}{llrrrr}
\toprule
& & & \multicolumn{3}{c}{Private} \\
\cmidrule(lr){4-6}
Dataset & Arm & Non-private & $\varepsilon = 20$ & $\varepsilon = 10$ & $\varepsilon = 5$ \\
\midrule
\multirow{4}{*}{\shortstack[l]{BNCI2014\_001\\\footnotesize chance $25.0$}}
  & SPDNet \textsc{ProjAvg} & $55.1 \pm 1.8$ & $43.3 \pm 1.9$ & $39.0 \pm 2.2$ & $35.7 \pm 2.0$ \\
  & SPDNet \textsc{RLAvg} & $\mathbf{55.2 \pm 1.6}$ & $\mathbf{43.9 \pm 1.6}$ & $39.4 \pm 1.0$ & $\mathbf{35.9 \pm 1.4}$ \\
  & SPDNet \textsc{RetAvg} & $54.7 \pm 2.3$ & $43.3 \pm 0.8$ & $\mathbf{41.3 \pm 2.2}$ & $34.9 \pm 0.9$ \\
  & EEGNet & $53.2 \pm 1.9$ & $30.0 \pm 2.2$ & $28.6 \pm 1.2$ & $27.4 \pm 2.1$ \\
\midrule
\multirow{4}{*}{\shortstack[l]{Cho2017\\\footnotesize chance $50.0$}}
  & SPDNet \textsc{ProjAvg} & $62.9 \pm 0.8$ & $55.6 \pm 1.5$ & $\mathbf{53.9 \pm 1.3}$ & $52.0 \pm 1.2$ \\
  & SPDNet \textsc{RLAvg} & $62.9 \pm 0.8$ & $\mathbf{55.7 \pm 0.9}$ & $53.4 \pm 1.7$ & $\mathbf{52.1 \pm 1.1}$ \\
  & SPDNet \textsc{RetAvg} & $62.9 \pm 0.7$ & $54.1 \pm 1.9$ & $52.8 \pm 1.6$ & $51.8 \pm 1.0$ \\
  & EEGNet & $\mathbf{66.6 \pm 0.7}$ & $50.0 \pm 0.0$ & $50.1 \pm 0.7$ & $50.1 \pm 0.9$ \\
\midrule
\multirow{4}{*}{\shortstack[l]{Weibo2014\\\footnotesize chance $14.3$}}
  & SPDNet \textsc{ProjAvg} & $46.9 \pm 1.9$ & $23.6 \pm 2.0$ & $20.8 \pm 1.7$ & $\mathbf{17.5 \pm 2.0}$ \\
  & SPDNet \textsc{RLAvg} & $46.9 \pm 1.9$ & $23.8 \pm 2.0$ & $21.2 \pm 2.2$ & $17.4 \pm 1.5$ \\
  & SPDNet \textsc{RetAvg} & $\mathbf{47.0 \pm 1.9}$ & $\mathbf{24.6 \pm 1.8}$ & $\mathbf{21.5 \pm 1.9}$ & $17.1 \pm 1.0$ \\
  & EEGNet & $38.4 \pm 1.4$ & $15.6 \pm 1.6$ & $14.8 \pm 0.9$ & $14.7 \pm 1.1$ \\
\bottomrule
\end{tabular}
\end{table}

SPDNet keeps this advantage under partial participation (Table~\ref{tab:final-half}). With
$\lceil n/2\rceil$ clients sampled per round, it still leads EEGNet on every dataset and at
every budget. Partial participation costs SPDNet at most $2.4$ points under privacy, since
each round averages the noise of fewer clients, and at most $0.5$ point without privacy.

\begin{table}[h]
\centering
\caption{Federated test accuracy (\%) at half client participation, mean $\pm$ standard deviation over 10 seeds, last round, $\delta = 10^{-5}$. Learning rate, batch size and $T$ are those of full participation. Best per column within a dataset in bold.}
\label{tab:final-half}
\small
\begin{tabular}{llrrrr}
\toprule
& & & \multicolumn{3}{c}{Private} \\
\cmidrule(lr){4-6}
Dataset & Arm & Non-private & $\varepsilon = 20$ & $\varepsilon = 10$ & $\varepsilon = 5$ \\
\midrule
\multirow{2}{*}{\shortstack[l]{BNCI2014\_001\\\footnotesize chance $25.0$}}
  & SPDNet \textsc{ProjAvg} & $\mathbf{54.7 \pm 1.6}$ & $\mathbf{41.5 \pm 2.0}$ & $\mathbf{37.8 \pm 2.1}$ & $\mathbf{35.0 \pm 1.3}$ \\
  & EEGNet  & $52.8 \pm 0.9$ & $28.5 \pm 2.1$ & $27.9 \pm 2.0$ & $25.3 \pm 1.5$ \\
\midrule
\multirow{2}{*}{\shortstack[l]{Cho2017\\\footnotesize chance $50.0$}}
  & SPDNet \textsc{ProjAvg} & $62.7 \pm 0.9$ & $\mathbf{54.6 \pm 1.7}$ & $\mathbf{52.4 \pm 1.9}$ & $\mathbf{51.8 \pm 1.6}$ \\
  & EEGNet  & $\mathbf{67.1 \pm 0.7}$ & $49.9 \pm 0.3$ & $50.4 \pm 0.5$ & $50.3 \pm 0.6$ \\
\midrule
\multirow{2}{*}{\shortstack[l]{Weibo2014\\\footnotesize chance $14.3$}}
  & SPDNet \textsc{ProjAvg} & $\mathbf{47.1 \pm 2.5}$ & $\mathbf{21.2 \pm 1.6}$ & $\mathbf{19.3 \pm 1.8}$ & $\mathbf{17.7 \pm 1.3}$ \\
  & EEGNet  & $38.1 \pm 2.0$ & $15.7 \pm 1.1$ & $14.8 \pm 1.3$ & $14.9 \pm 1.0$ \\
\bottomrule
\end{tabular}
\end{table}

\providecolor{C2}{HTML}{2E8B57}
\providecolor{C4}{HTML}{6A4C93}
\providecommand{\curveheight}{3.3cm}
\providecommand{\curverow}[2]{\pgfplotsset{every axis/.append style={ymin=#1, ymax=#2}}}
\providecommand{\curvenoylabel}{\pgfplotsset{every axis post/.append style={ylabel={}}}}
\providecommand{\curverowlabel}[1]{%
  \begin{minipage}[c]{0.035\linewidth}\centering\rotatebox{90}{\small #1}\end{minipage}}
\begin{figure}[h]
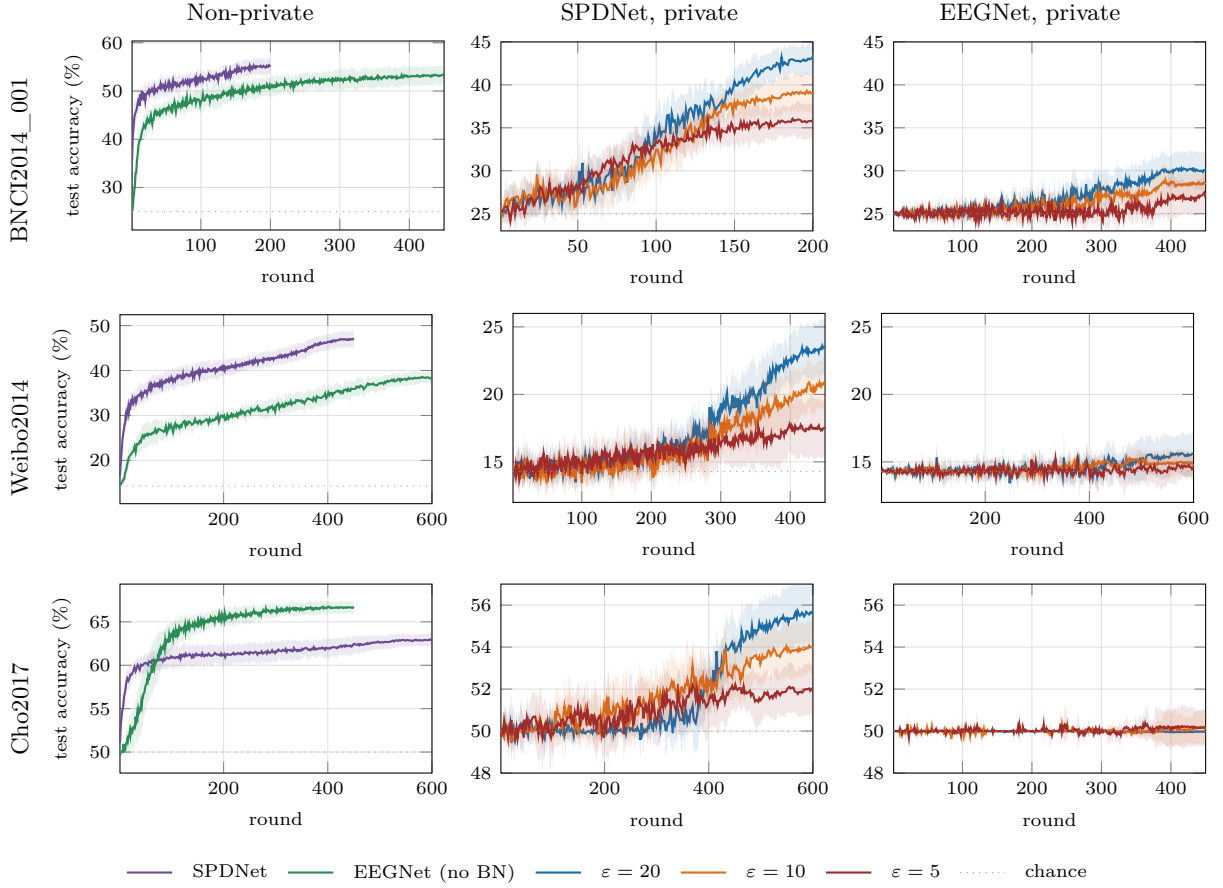

\centering
\edef\curveaxiswidth{\the\dimexpr0.25\textwidth\relax}
\providecommand{\curveaxisheight}{2.5cm}
\pgfplotsset{every axis post/.append style={scale only axis, width=\curveaxiswidth,
  height=\curveaxisheight}}%
\begin{minipage}{0.035\linewidth}\hfill\end{minipage}%
\begin{minipage}{0.335\linewidth}\centering\small Non-private\end{minipage}%
\begin{minipage}{0.305\linewidth}\centering\small SPDNet, private\end{minipage}%
\begin{minipage}{0.305\linewidth}\centering\small EEGNet, private\end{minipage}

\curverowlabel{BNCI2014\_001}%
\begin{minipage}[c]{0.335\linewidth}\centering
  \input{curves/curve_BNCI2014_001_nonprivate}\end{minipage}%
{\curverow{23}{45}\curvenoylabel%
\begin{minipage}[c]{0.305\linewidth}\centering
  \input{curves/curve_BNCI2014_001_spdnet}\end{minipage}%
\begin{minipage}[c]{0.305\linewidth}\centering
  \input{curves/curve_BNCI2014_001_eegnet_nobn}\end{minipage}}

\medskip
\curverowlabel{Weibo2014}%
\begin{minipage}[c]{0.335\linewidth}\centering
  \input{curves/curve_Weibo2014_nonprivate}\end{minipage}%
{\curverow{12}{26}\curvenoylabel%
\begin{minipage}[c]{0.305\linewidth}\centering
  \input{curves/curve_Weibo2014_spdnet}\end{minipage}%
\begin{minipage}[c]{0.305\linewidth}\centering
  \input{curves/curve_Weibo2014_eegnet_nobn}\end{minipage}}

\medskip
\curverowlabel{Cho2017}%
\begin{minipage}[c]{0.335\linewidth}\centering
  \input{curves/curve_Cho2017_nonprivate}\end{minipage}%
{\curverow{48}{57}\curvenoylabel%
\begin{minipage}[c]{0.305\linewidth}\centering
  \input{curves/curve_Cho2017_spdnet}\end{minipage}%
\begin{minipage}[c]{0.305\linewidth}\centering
  \input{curves/curve_Cho2017_eegnet_nobn}\end{minipage}}

\medskip
\begin{tikzpicture}
  \begin{axis}[hide axis, width=2cm, height=2cm, xmin=0, xmax=1, ymin=0, ymax=1,
    legend columns=6, legend cell align=left,
    legend style={font=\scriptsize, draw=none, fill=none, column sep=1.2ex}]
    \addlegendimage{C4, thick}\addlegendentry{SPDNet}
    \addlegendimage{C2, thick}\addlegendentry{EEGNet (no BN)}
    \addlegendimage{C0, thick}\addlegendentry{$\varepsilon = 20$}
    \addlegendimage{C1, thick}\addlegendentry{$\varepsilon = 10$}
    \addlegendimage{C3, thick}\addlegendentry{$\varepsilon = 5$}
    \addlegendimage{black!40, dotted}\addlegendentry{chance}
  \end{axis}
\end{tikzpicture}
\caption{Test accuracy against communication round at full client participation, mean over
10 seeds with a band spanning $\pm 1$ standard deviation. One row per dataset. Left:
non-private runs of both models. Middle and right: private runs of SPDNet
(\textsc{ProjAvg}) and EEGNet, with different $y$-axis graduations between non-private and private curves.}
\label{fig:curves}
\end{figure}

At a fixed budget, a longer run needs more noise, since $\sdp$ grows roughly as $\sqrt{T\tau}$ (Section~\ref{ssec:reading}). On BNCI2014\_001 and Weibo2014, halving or doubling $T$ with $\sdp$ recalibrated changes the accuracy of SPDNet by at most 2 points, within one standard deviation, and that of EEGNet by less than one point (Appendix~\ref{app:T-ablation}).
 
\subsection{Riemannian aggregation rules}\label{ssec:agg-results}
The three Riemannian rules reach the same accuracy but differ in cost. They land within
$2.3$ points of one another in every cell of Table~\ref{tab:final}, and the two
projection-based rules within $0.6$ point, below the seed-to-seed standard deviation.
Computationally, Table~~\ref{tab:aggtime} in Appendix~\ref{app:aggtime} shows that \textsc{RetAvg} costs $9$ to $18$ times more per round than \textsc{ProjAvg}. The projection-based rules thus give the same
accuracy at a fraction of the cost.

\subsection{Comparison with PriRFed}\label{ssec:prirfed}
On a compact submanifold with the induced metric, PriRFed and \textsc{DP-RFedProj} privatise
each gradient with the same mechanism. With the same rule \textsc{RetAvg} and same optimiser, they differ only in how they account for the whole
run. We fix $(C,B,q,\tau,T)$ and calibrate $\sdp$ under each accountant. PriRFed gets its
best case: its composition receives our per-round privacy profile instead of its published
closed form (Appendix~\ref{app:dp-settings}).

PriRFed needs five to six times as much noise (Figure~\ref{fig:prirfed-ratio}), and this noise costs accuracy. With the same rule, \textsc{RetAvg}, and the same
hyperparameters, SPDNet trained at PriRFed's noise loses $12.7$, $3.7$ and $8.2$ points on
BNCI2014\_001, Cho2017 and Weibo2014, and ends within $6$ points of chance.
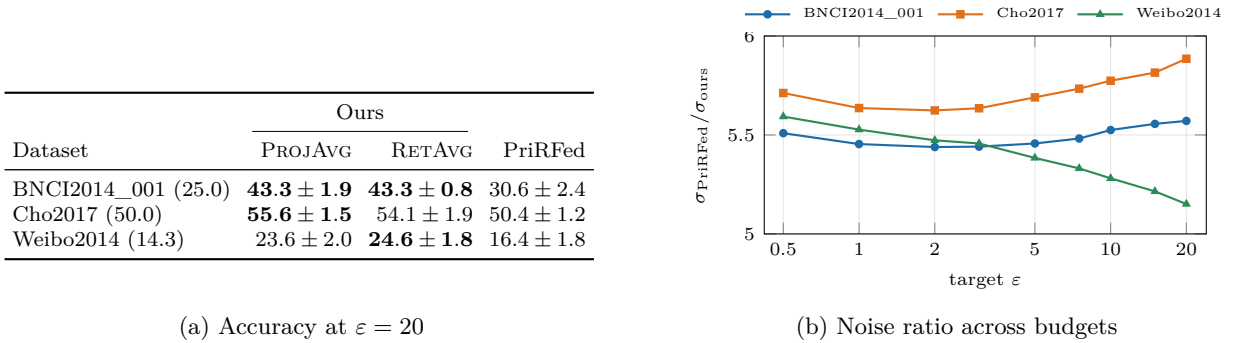
\begin{figure}[t]
\centering
\begin{subfigure}[b]{0.5\linewidth}
\centering
\footnotesize
\setlength{\tabcolsep}{3pt}
\begin{tabular}{lrrr}
\toprule
& \multicolumn{2}{c}{Ours} & \\
\cmidrule(lr){2-3}
Dataset & \textsc{ProjAvg} & \textsc{RetAvg} & PriRFed \\
\midrule
BNCI2014\_001 ($25.0$) & $\mathbf{43.3 \pm 1.9}$ & $\mathbf{43.3 \pm 0.8}$ & $30.6 \pm 2.4$ \\
Cho2017 ($50.0$)       & $\mathbf{55.6 \pm 1.5}$ & $54.1 \pm 1.9$ & $50.4 \pm 1.2$ \\
Weibo2014 ($14.3$)     & $23.6 \pm 2.0$ & $\mathbf{24.6 \pm 1.8}$ & $16.4 \pm 1.8$ \\
\bottomrule
\end{tabular}
\vspace{14pt}
\caption{Accuracy at $\varepsilon=20$}
\label{fig:prirfed-acc}
\end{subfigure}\hfill
\begin{subfigure}[b]{0.45\linewidth}
\centering
\begin{tikzpicture}
\begin{axis}[width=\linewidth, height=4.2cm, xmode=log, xmin=0.42, xmax=24,
  xtick={0.5,1,2,5,10,20}, xticklabels={0.5,1,2,5,10,20},
  ymin=5, ymax=6, ytick={5,5.5,6},
  xlabel={target $\varepsilon$},
  ylabel={$\sigma_{\mathrm{PriRFed}}/\sigma_{\mathrm{ours}}$},
  tick label style={font=\scriptsize}, label style={font=\scriptsize},
  grid=major, grid style={black!10},
  legend style={font=\tiny, at={(0.5,1.02)}, anchor=south, legend columns=3,
    draw=none, column sep=0.6ex},
  legend cell align=left]
\addplot[C0, thick, mark=*, mark size=1.2pt] coordinates {(0.5,5.509) (1,5.454) (2,5.439) (3,5.441) (5,5.457) (7.5,5.482) (10,5.525) (15,5.556) (20,5.571)};
\addlegendentry{BNCI2014\_001}
\addplot[C1, thick, mark=square*, mark size=1.2pt] coordinates {(0.5,5.712) (1,5.636) (2,5.624) (3,5.635) (5,5.690) (7.5,5.734) (10,5.774) (15,5.815) (20,5.885)};
\addlegendentry{Cho2017}
\addplot[C2, thick, mark=triangle*, mark size=1.2pt] coordinates {(0.5,5.593) (1,5.527) (2,5.473) (3,5.457) (5,5.384) (7.5,5.331) (10,5.281) (15,5.215) (20,5.151)};
\addlegendentry{Weibo2014}
\end{axis}
\end{tikzpicture}
\caption{Noise ratio across budgets}
\label{fig:prirfed-ratio}
\end{subfigure}
\caption{Comparison with PriRFed on SPDNet, $\delta=10^{-5}$. (a) Test accuracy (\%) at
$\varepsilon=20$ over 10 seeds, chance in parentheses. PriRFed is \textsc{RetAvg} at the noise
of its own accounting. (b) Ratio
$\sigma_{\mathrm{PriRFed}}/\sigma_{\mathrm{ours}}$ of the noise levels against the target
$\varepsilon$, for the mechanism of (a).}
\label{fig:prirfed}
\end{figure}

\section{Conclusion}\label{sec:conclusion}
We introduced \textsc{DP-RFedProj}, a record-level differentially private federated learning
algorithm for compact submanifolds. Its privacy guarantee comes from privatising gradients in
a fixed ambient geometry, so it holds for any manifold, aggregator and local optimiser. Its
convergence analysis comes from aggregating through closed-form projections, and covers
multiple local steps, partial participation and heterogeneous clients together. On federated
PCA, the experiments match the theory: the cost of privacy grows with the intrinsic dimension
and decreases with the number of clients per round. On EEG motor-imagery classification,
SPDNet outperforms EEGNet at every privacy budget, even on the dataset where EEGNet leads
without privacy. Our accounting needs five to six times less noise than the closest prior
method, and projection-based aggregation reaches the accuracy of retraction-based aggregation
at a fraction of the cost.

Three limitations bound these results. The framework requires a compact submanifold under the
Euclidean-induced metric. This excludes iterate-dependent metrics without an ambient
projection, such as the affine-invariant metric on SPD matrices. The convergence rate is proved
for Riemannian SGD at a constant stepsize, while the EEG experiments use Adam with a schedule.
The accounting charges the full stream of privatised gradients rather than the submissions the
server observes. This keeps it optimiser-agnostic, but leaves room for tighter,
optimiser-specific analyses.

Several directions extend this work. Subject-level privacy is a natural next step for EEG: a
subject records strongly correlated trials, so protecting one trial bounds what the server
learns about that trial, not about the subject. Personalisation, where each client keeps part
of the model local~\citep{pautrel2026personalized}, raises the question of how the shared and
local parts split the privacy budget. Finally, batch normalisation, Euclidean or
Riemannian~\citep{brooks2019riemannian}, breaks the per-record sensitivity bound. Privatising
its statistics, or using per-sample normalisations compatible with differential
privacy~\citep{ponomareva2023dpfy}, could recover accuracy without weakening the guarantee.

\section*{Broader Impact Statement}
This work aims to make collaborative training on sensitive records, such as EEG recordings,
safer. Its guarantee holds only under the stated threat model: an honest-but-curious server
and record-level protection. Privacy also costs accuracy, so the budget should be chosen
for the protection it affords, not for the accuracy it permits.

\bibliographystyle{tmlr}
\bibliography{references}
\clearpage
\appendix
\section{The Stiefel, Grassmann and SPD manifolds}\label{app:manifolds}

\subsection{The Stiefel manifold}\label{app:stiefel}

The Stiefel manifold $\mathrm{St}(d,p) = \{X \in \R^{d\times p} : X^{\top} X = I_p\}$,
$p \le d$, is the set of matrices with orthonormal columns, a compact embedded submanifold
of $\R^{d \times p}$ of dimension $dp - \tfrac{p(p+1)}{2}$. It reduces to the orthogonal
group when $p = d$ and to the unit sphere when $p = 1$. Compactness follows from
$\|X\|_F^2 = p$.

The tangent space at $X$ is
\begin{equation}\label{eq:stiefel-tangent}
  T_X\mathrm{St}(d,p) = \{\xi \in \R^{d\times p} : X^{\top}\xi + \xi^{\top}X = 0\},
\end{equation}
that is, the matrices whose $X$-component is skew-symmetric. Equipping
$\mathrm{St}(d,p)$ with the metric induced by the ambient Frobenius inner product
$\langle \xi, \eta \rangle = \operatorname{tr}(\xi^{\top}\eta)$, the orthogonal projector
onto \eqref{eq:stiefel-tangent} is
\begin{equation}\label{eq:stiefel-orth-proj}
  P_X(\xi) = \xi - X\,\mathrm{sym}(X^{\top}\xi),
  \qquad \mathrm{sym}(A) := \tfrac{1}{2}(A + A^{\top}),
\end{equation}
and the Riemannian gradient of a smooth $f$ is $\mathrm{grad} f(X) = P_X(\nabla f(X))$
with $\nabla f$ the ambient gradient.

For $A \in \R^{d\times p}$ of full column rank, two retractions are standard. The polar
retraction uses the orthogonal polar factor
\begin{equation}\label{eq:stiefel-proj}
  \Pi(A) = A(A^{\top}A)^{-1/2} = UV^{\top},
  \qquad A = USV^{\top} \ \text{(thin SVD)},
\end{equation}
which is also the nearest-point projection onto $\mathrm{St}(d,p)$, unique and smooth on
that set~\citep{absil2012projection,bouchard2025beyond}. The QR retraction instead takes the
$Q$ factor of a thin QR decomposition, which has the same complexity with a smaller constant
and is numerically more stable, but is not the nearest-point map. Both are idempotent on $\mathrm{St}(d,p)$ and define
projection-based retractions $\mathrm{R}_X(\xi) = \Pi(X + \xi)$. We use the polar factor
throughout, since it is the projection our aggregation rules require.

The Riemannian logarithm on $\mathrm{St}(d,p)$ requires an
iterative algorithm~\citep{zimmermann2022computing}, and so does the inverse polar
retraction, which solves the Sylvester equation $MS + SM^{\top} = 2I_p$ with
$M = X^{\top}Y$ at $O(dp^2 + p^3)$ per iteration. That solution exists only when
$M + M^{\top}$ is positive definite, so the inverse retraction is moreover defined only in
a neighbourhood of $X$.

\subsection{The Grassmann manifold}\label{app:grassmann}

We represent the Grassmann manifold $\mathrm{Gr}(d,p)$ of $p$-dimensional subspaces of
$\R^d$ by the orthogonal projectors onto them
\begin{equation}\label{eq:grassmann-projector-model}
  \mathrm{Gr}(d,p) \;\cong\; \{P \in \mathcal{S}_d : P^2 = P,\ \operatorname{tr} P = p\}.
\end{equation}
It is a compact embedded submanifold of the symmetric matrices $\mathcal{S}_d$ of dimension
$p(d-p)$.

Under the metric induced by the ambient Frobenius inner product, the tangent space at $P$ is
\begin{equation}\label{eq:grassmann-tangent}
  T_P\mathrm{Gr}(d,p) = \{\xi \in \mathcal{S}_d : P\xi P = 0
    \ \text{and}\ (I-P)\xi(I-P) = 0\},
\end{equation}
that is, the symmetric matrices supported off the two diagonal blocks, and the orthogonal
projector onto \eqref{eq:grassmann-tangent} keeps exactly those blocks,
\begin{equation}\label{eq:grassmann-orth-proj}
  \mathrm{Proj}_P(\xi) = P\xi(I-P) + (I-P)\xi P .
\end{equation}
For $A \in \mathcal{S}_d$ with eigenvalues $\lambda_1 \ge \dots \ge \lambda_d$ and
associated orthonormal eigenvectors $u_1,\dots,u_d$, the nearest-point projection is the
spectral projector onto a top-$p$ eigenspace,
\begin{equation}\label{eq:grassmann-proj}
  \Pi(A) = \sum_{i=1}^{p} u_i u_i^{\top},
\end{equation}
unique when $\lambda_p > \lambda_{p+1}$, at the cost of one symmetric eigendecomposition.
\subsection{The manifold of SPD matrices}\label{app:spd}

Let $\mathcal{S}_n$ denote the space of symmetric $n \times n$ matrices. The set
\begin{equation}
  \mathcal{S}^{++}_n = \{\Sigma \in \mathcal{S}_n : x^{\top}\Sigma x > 0
  \ \text{for all } x \in \R^n \setminus \{0\}\}
\end{equation}
of symmetric positive definite matrices is an open convex cone in $\mathcal{S}_n$, hence a
smooth manifold of dimension $\tfrac{n(n+1)}{2}$ whose tangent space at every $\Sigma$ is
$\mathcal{S}_n$ itself. It is not compact. Two metrics are standard.

\paragraph{Affine-invariant.} With
$\langle \xi, \eta \rangle_\Sigma = \operatorname{tr}(\Sigma^{-1}\xi\,\Sigma^{-1}\eta)$ the
manifold is complete with non-positive curvature, and
\begin{equation}
  \mathrm{Exp}_\Sigma(\xi) =
  \Sigma^{1/2}\exp\bigl(\Sigma^{-1/2}\xi\Sigma^{-1/2}\bigr)\Sigma^{1/2}.
\end{equation}
The metric is invariant under $\Sigma \mapsto G\Sigma G^{\top}$ for any invertible $G$,
which for covariance descriptors means invariance to linear mixing of the sensors. This property makes it the standard choice in EEG classification.

\paragraph{Log-Euclidean.} Pulling the Euclidean metric of $\mathcal{S}_n$ back through the
matrix logarithm gives
\begin{equation}
  \langle \xi, \eta \rangle_\Sigma
  = \bigl\langle D_\Sigma\!\log(\xi),\, D_\Sigma\!\log(\eta) \bigr\rangle_F ,
\end{equation}
with $D_\Sigma\!\log$ the differential of the matrix logarithm at $\Sigma$. Under this
metric $\log$ is an isometry onto $\mathcal{S}_n$, so the manifold is flat and all
operations reduce to Euclidean ones in the logarithmic domain. It is cheaper than the
affine-invariant metric and invariant to orthogonal transformations and scaling, but not to
general linear mixing.

\section{Proofs for the convergence analysis}\label{app:conv}

\subsection{Geometric preliminaries}\label{app:geom-prelims}
Let $\M$ be a compact smooth embedded $\dM$-dimensional submanifold of $\mathcal{E}$, with
$\dim(\mathcal{E})=D$, and let it be $2\gamma$-proximally smooth
(Assumption~\ref{ass:prox}). In finite dimension, this notion matches the proximal
smoothness of~\citet{clarke1995proximal}. We use two regularity
properties~\citep{clarke1995proximal,deng2025decentralized}.

\paragraph{Step-to-projection bound.} $\Pi$ is $2$-Lipschitz on
$\overline{U_\M(\gamma)}$~\citep[Thm.~4.8]{clarke1995proximal}. Since $\Pi(x)=x$ for
$x\in\M$, for every $x\in\M$ and $u\in\mathcal{E}$ with $\|u\|\le\gamma$,
\begin{equation}\label{eq:proj-step}
  \|\Pi(x+u)-x\| \le 2\|u\|.
\end{equation}

\paragraph{Lipschitz tangent projectors.} The map $x\mapsto P_x$ is smooth on the compact
$\M$, hence Lipschitz~\citep{deng2025decentralized}: there is $\Ldp>0$, depending only on
$\M$, such that
\begin{equation}\label{eq:Px-lip}
  \opnorm{P_x - P_y} \le \Ldp\,\|x-y\|,
  \qquad x,y\in\M.
\end{equation}

\paragraph{Clipping.} The clipping operator~\eqref{eq:clip-generic} is the metric projection
onto the ball $\{v\in\mathcal{E}:\|v\|\le C\}$. Hence, for all $u,v\in\mathcal{E}$,
\begin{equation}\label{eq:clip-props}
  \|\clip_C(v)\| \le C,
  \qquad
  \|\clip_C(u) - \clip_C(v)\| \le \|u - v\| .
\end{equation}
Since $\clip_C$ only rescales its argument, $\gradC f_i(x)\in T_x\M$ and
$\hgradC f_i(x;\mathcal{B})\in T_x\M$, and both have norm at most $C$.

\subsection{Auxiliary lemmas}\label{app:aux-lemmas}

\begin{lemma}[Clipped stochastic gradient]\label{lem:clip-var}
  Under Assumption~\ref{ass:sgd}, for every $i\in[n]$ with $m_i\ge2$ and every $x\in\M$, the minibatch
  estimator~\eqref{eq:clipped-mb-grad} with $\mathcal{B}$ drawn uniformly without
  replacement satisfies $\E_{\mathcal B}[\hgradC f_i(x;\mathcal B)]=\gradC f_i(x)$ and
  \begin{equation}\label{eq:clip-var-bound}
    \E_{\mathcal B}\bigl\|\hgradC f_i(x;\mathcal B) - \gradC f_i(x)\bigr\|^2
    \le \rho_i\,\sigma_{\mathrm{st}}^2,
    \qquad \rho_i := \frac{m_i - B}{B(m_i - 1)} .
  \end{equation}
\end{lemma}
\begin{proof}
Write $c_\ell:=\clip_C(\grad f(x;z_{i,\ell}))$, so that $\gradC f_i(x)$ is the mean of the
$m_i$ vectors $c_\ell$ and $\hgradC f_i(x;\mathcal B)$ their mean over $\mathcal{B}$. Each
index is drawn with probability $B/m_i$, which gives unbiasedness. The variance of a sample
mean drawn without replacement~\citep[Thm.~2.2]{cochran1977sampling} is
$\rho_i V_i$, with $V_i:=\frac1{m_i}\sum_\ell\|c_\ell-\gradC f_i(x)\|^2$. The mean minimises
the mean squared deviation, so evaluating at $\clip_C(\grad f_i(x))$ instead and
using~\eqref{eq:clip-props} and Assumption~\ref{ass:sgd},
\[
  V_i\le\frac1{m_i}\sum_{\ell=1}^{m_i}\bigl\|c_\ell-\clip_C(\grad f_i(x))\bigr\|^2
  \le\frac1{m_i}\sum_{\ell=1}^{m_i}\bigl\|\grad f(x;z_{i,\ell})-\grad f_i(x)\bigr\|^2
  \le\sigma_{\mathrm{st}}^2 .\qedhere
\]
\end{proof}

Write $\bar b^C:=\frac1n\sum_{i}b_i^C$ for the mean clipping bias and
$\overline{\gradC f}:=\frac1n\sum_{i}\gradC f_i=\grad F+\bar b^C$ for the mean clipped
gradient, so that $\beta_C=\sup_{x\in\M}\|\bar b^C(x)\|$.

\begin{lemma}[Clipped-gradient heterogeneity]\label{lem:clip-hetero}
  Under Assumption~\ref{ass:hetero}, for every $x\in\M$,
  \begin{equation}\label{eq:clip-hetero-bound}
    \frac{1}{n}\sum_{i=1}^{n}\bigl\|\gradC f_i(x) - \overline{\gradC f}(x)\bigr\|^2
    \le (\sighet + \sigma_{\mathrm{clip}})^2.
  \end{equation}
\end{lemma}
\begin{proof}
At fixed $x$,
$\gradC f_i-\overline{\gradC f}=(\grad f_i-\grad F)+(b_i^C-\bar b^C)$. Young's inequality
$\|a+b\|^2\le(1+\lambda)\|a\|^2+(1+\lambda^{-1})\|b\|^2$, averaged over $i$, gives the bound
$(1+\lambda)\sighet^2+(1+\lambda^{-1})\sigma_{\mathrm{clip}}^2$ by
Assumption~\ref{ass:hetero} and~\eqref{eq:sigma-clip-def}. Its infimum over $\lambda>0$ is
$(\sighet+\sigma_{\mathrm{clip}})^2$.
\end{proof}

\begin{lemma}[Projected Gaussian under norm truncation]\label{lem:projected-noise}
  Let $\xi \sim \N(0, s^2 \Idp)$ on $\mathcal{E}$ and let $P$ be an orthogonal projector of
  rank $\dM\le D$. Then $\E\|P\xi\|^2 = \dM s^2$ and, for every $R>0$,
  $\E\bigl[\|P\xi\|^2 \bigm| \|\xi\|\le R\bigr] \le \dM s^2$.
\end{lemma}
\begin{proof}
Write $\xi=\nu U$ with $\nu:=\|\xi\|$. By rotational invariance, $U$ is uniform on the unit
sphere and independent of $\nu$, and $\E\|PU\|^2=\dM/D$. Hence
$\E[\|P\xi\|^2\mid\nu\le R]=\frac{\dM}{D}\E[\nu^2\mid\nu\le R]\le\frac{\dM}{D}\E[\nu^2]
=\dM s^2$. The inequality holds because $\nu^2$ is nondecreasing and
$\mathbf{1}\{\nu\le R\}$ nonincreasing in $\nu$, so they are negatively correlated. The case
$R=\infty$ gives the first claim.
\end{proof}

\subsection{Probabilistic setting}\label{app:notation}
By definition of $\Pi$~\eqref{eq:nearest-p-proj}, every iterate lies in $\M$, even when a
pre-projection point leaves the tube. We regard the minibatches and the noise of round $t$
as drawn for every client $i\in[n]$, and used only if $i\in S_t$. Let $\mathcal{F}_t$ be the
$\sigma$-algebra generated by all the randomness of rounds $0,\dots,t-1$. It fixes $x_t$,
and leaves $S_t$, the minibatches and the noise of round $t$ random. We write
$\E_t[\cdot]:=\E[\cdot\mid\mathcal{F}_t]$ and
$\EtG[\cdot]:=\E[\cdot\mid\mathcal{F}_t,\mathcal{G}_t]$, with $\mathcal{G}_t$ the event of
Definition~\ref{def:good-event}.

\subsection{Proof of Lemma~\ref{lem:pfail}}\label{ssec:proof-pfail}
Fix $t$. Given $\mathcal{F}_t$ and $S_t$, the $k\tau$ noise vectors
$\xi_{t,j}^{(i)}$, $(i,j)\in S_t\times\{0,\dots,\tau-1\}$, are i.i.d.\ $\N(0,\sdp^2\Idp)$, so
$\Pr(\mathcal{G}_t\mid\mathcal{F}_t,S_t)=\Pr(\|\xi\|\le R)^{k\tau}$ with
$\xi\sim\N(0,\sdp^2\Idp)$. This constant depends neither on $t$ nor on $S_t$. For one draw,
$\|\xi\|/\sdp$ has the law of $\|W\|$ with $W\sim\N(0,\Idp)$. The norm is $1$-Lipschitz and
$\E\|W\|\le\sqrt D$, so Gaussian concentration~\citep[Thm.~5.6]{boucheron2013concentration}
gives $\Pr(\|W\|\ge\sqrt D+\sqrt{c_0})\le e^{-c_0/2}$, that is
$\Pr(\|\xi\|>R)\le e^{-c_0/2}$ by~\eqref{eq:R-def}. A union bound over the $k\tau$ draws gives
$\pi\le k\tau e^{-c_0/2}$. In particular, $\Pr(\mathcal{G}_t^c\mid\mathcal{F}_t)=\pi$ almost
surely, and $S_t$ is independent of $\mathcal{G}_t$ given $\mathcal{F}_t$.

\subsection{Proof of Theorem~\ref{thm:finite-T}}\label{app:proof-main}

For a selected client $i\in S_t$ and a local step $j$, the
minibatch error and its anchored version are
\begin{equation}\label{eq:delta-zeta}
  \delta_{t,j}^{(i)}:=\hgradC f_i\bigl(x_{t,j}^{(i)};\mathcal{B}_{t,j}^{(i)}\bigr)
    -\gradC f_i\bigl(x_{t,j}^{(i)}\bigr)\in T_{x_{t,j}^{(i)}}\M,
  \qquad
  \zeta_{t,j}^{(i)}:=P_{x_t}\delta_{t,j}^{(i)} .
\end{equation}
At the server, we write
\begin{equation}\label{eq:server-notation}
  h_t:=\frac1k\sum_{i\in S_t}\gradC f_i(x_t),\qquad
  \zeta_t^{\mathrm{st}}:=-\frac{\eta}{k}\sum_{i\in S_t}\sum_{j=0}^{\tau-1}\zeta_{t,j}^{(i)},\qquad
  \zeta_t^{\mathrm{dp}}:=-\frac{\eta}{k}\sum_{i\in S_t}\sum_{j=0}^{\tau-1}P_{x_t}\xi_{t,j}^{(i)},
\end{equation}
and $y_t:=\frac1k\sum_{i\in S_t}x_{t,\tau}^{(i)}$ for the mean of the returned iterates. By
\eqref{eq:projavg} and~\eqref{eq:rlavg}, the server sets $x_{t+1}=\Pi(y_t)$ under
\textsc{ProjAvg} and $x_{t+1}=\Pi(x_t+P_{x_t}(y_t-x_t))$ under \textsc{RLAvg}.

The proof follows the three steps of Section~\ref{ssec:proof-outline}. Steps~1 to~3 hold on
$\mathcal{G}_t$. The last part handles the complement and sums over rounds.

\subsubsection*{Step 1: tube stability}
\begin{lemma}[Tube stability]\label{lem:tube-stability}
  Let $\alpha\le\gamma/(4\Geff)$. On $\mathcal{G}_t$, for every $i\in S_t$:
  \begin{enumerate}[label=(\alph*),leftmargin=*]
    \item $\|x_{t,j}^{(i)} - x_t\| \le 2\Geff\eta j\le\gamma/2$ for $j=0,\dots,\tau$;
    \item every local pre-projection point $x_{t,j}^{(i)}-\eta\tilde g_{t,j}^{(i)}$ lies in
      $U_\M(\gamma)$;
    \item $\|P_{x_t}(y_t-x_t)\|\le\|y_t-x_t\|\le 2\Geff\alpha\le\gamma/2$, so both server
      pre-projection points lie in $U_\M(\gamma)$.
  \end{enumerate}
  In particular every projection in round $t$ is single-valued.
\end{lemma}
\begin{proof}
On $\mathcal{G}_t$, $\|\eta\tilde g_{t,j}^{(i)}\|\le\eta\Geff$
by~\eqref{eq:grad-bound-geff}, and $\eta\Geff=\alpha\Geff/\tau\le\gamma/4$.
(a)~By~\eqref{eq:proj-step},
$\|x_{t,j+1}^{(i)}-x_{t,j}^{(i)}\|\le2\eta\Geff$. Summing over the first $j$ steps gives the
bound, and $2\Geff\eta\tau=2\Geff\alpha\le\gamma/2$.
(b)~The point lies within $\eta\Geff\le\gamma/4$ of $x_{t,j}^{(i)}\in\M$.
(c)~$y_t-x_t$ is the mean of the $x_{t,\tau}^{(i)}-x_t$, each of norm at most $2\Geff\alpha$
by~(a), and $P_{x_t}$ is an orthogonal projector. Both pre-projection points lie within
$\gamma/2$ of $x_t\in\M$.
\end{proof}

\subsubsection*{Step 2: linearisation of the local trajectory}
\begin{lemma}[Linearisation]\label{lem:linearisation}
  Let $\alpha\le\gamma/(4\Geff)$. On $\mathcal{G}_t$, every $i\in S_t$ satisfies
  \begin{equation}\label{eq:linearisation}
    x_{t,\tau}^{(i)} - x_t
    = -\alpha\,\gradC f_i(x_t)
      - \eta\sum_{j=0}^{\tau-1}\zeta_{t,j}^{(i)}
      - \eta\sum_{j=0}^{\tau-1} P_{x_t}\xi_{t,j}^{(i)}
      + r_t^{(i)},
    \qquad \|r_t^{(i)}\|\le A_1\alpha^2,
  \end{equation}
  with
  \begin{equation}\label{eq:A1}
    A_1 := (\ell_f + \Ldp G)\Geff + 2\Ldp C\Geff + \Ldp\Geff R + \Cqr\Geff^2 .
  \end{equation}
\end{lemma}
\begin{proof}
Write $P_{t,j}:=P_{x_{t,j}^{(i)}}$ and work on $\mathcal{G}_t$.

\emph{Per-step expansion.} The step $-\eta\tilde g_{t,j}^{(i)}$ is tangent at
$x_{t,j}^{(i)}$, of norm at most $\eta\Geff\le\gamma$. The expansion~\eqref{eq:proj-quad-rem}
gives
\begin{equation}\label{eq:perstep}
  x_{t,j+1}^{(i)} - x_{t,j}^{(i)} = -\eta\,\tilde g_{t,j}^{(i)} + w_{t,j}^{(i)},
  \qquad \|w_{t,j}^{(i)}\| \le \Cqr\eta^2\Geff^2 .
\end{equation}
By~\eqref{eq:rec-noisy-grad},
$\tilde g_{t,j}^{(i)}=\gradC f_i(x_{t,j}^{(i)})+\delta_{t,j}^{(i)}+P_{t,j}\xi_{t,j}^{(i)}$.
Summing~\eqref{eq:perstep} over $j$, we replace each of the three terms by its value
anchored at $x_t$ and bound the error. We use $\|x_{t,j}^{(i)}-x_t\|\le2\Geff\eta j$
(Lemma~\ref{lem:tube-stability}(a)) and $\eta^2\sum_{j<\tau}j\le\alpha^2/2$.

\emph{Signal drift.} With $\grad f(x;z)=P_x\nabla f(x;z)$, \eqref{eq:clip-props},
Assumption~\ref{ass:smooth}, \eqref{eq:Px-lip} and $\|\nabla f\|\le G$
(Assumption~\ref{ass:gd-bd}),
\[
  \bigl\|\gradC f_i(x_{t,j}^{(i)})-\gradC f_i(x_t)\bigr\|
  \le\frac1{m_i}\sum_{\ell=1}^{m_i}\bigl\|P_{t,j}\nabla f(x_{t,j}^{(i)};z_{i,\ell})
    -P_{x_t}\nabla f(x_t;z_{i,\ell})\bigr\|
  \le(\ell_f+\Ldp G)\,\|x_{t,j}^{(i)}-x_t\|,
\]
so $\eta\sum_j\|\gradC f_i(x_{t,j}^{(i)})-\gradC f_i(x_t)\|\le(\ell_f+\Ldp G)\Geff\alpha^2$.

\emph{Minibatch error.} $\delta_{t,j}^{(i)}$ is tangent at $x_{t,j}^{(i)}$, so
$\delta_{t,j}^{(i)}-\zeta_{t,j}^{(i)}=(P_{t,j}-P_{x_t})\delta_{t,j}^{(i)}$. With
$\|\delta_{t,j}^{(i)}\|\le2C$ and~\eqref{eq:Px-lip}, its norm is at most
$4\Ldp C\Geff\eta j$, and $\eta$ times its sum over $j$ is at most $2\Ldp C\Geff\alpha^2$.

\emph{DP noise.} On $\mathcal{G}_t$, $\|\xi_{t,j}^{(i)}\|\le R$, so
$\|(P_{t,j}-P_{x_t})\xi_{t,j}^{(i)}\|\le2\Ldp\Geff R\eta j$, and $\eta$ times its sum over
$j$ is at most $\Ldp\Geff R\alpha^2$.

\emph{Projection remainder.} $\|\sum_jw_{t,j}^{(i)}\|\le\tau\Cqr\eta^2\Geff^2\le\Cqr\Geff^2\alpha^2$.

Since $\eta\tau=\alpha$, the four bounds give~\eqref{eq:linearisation}.
\end{proof}

\subsubsection*{Step 3: server update and one-round descent}
\begin{lemma}[Server update]\label{lem:server-decomp}
  Let $\alpha\le\gamma/(4\Geff)$. On $\mathcal{G}_t$, both \textsc{ProjAvg} and
  \textsc{RLAvg} satisfy
  \begin{equation}\label{eq:server-update}
    x_{t+1} - x_t
    = -\alpha\,h_t + \zeta_t^{\mathrm{st}} + \zeta_t^{\mathrm{dp}} + r_t,
    \qquad \|r_t\|\le A_2\alpha^2,\quad A_2:=A_1+4\Cqr\Geff^2 .
  \end{equation}
\end{lemma}
\begin{proof}
Averaging~\eqref{eq:linearisation} over $i\in S_t$ gives
$y_t-x_t=-\alpha h_t+\zeta_t^{\mathrm{st}}+\zeta_t^{\mathrm{dp}}+\bar r_t$, with
$\|\bar r_t\|\le A_1\alpha^2$. The vectors $h_t$, $\zeta_t^{\mathrm{st}}$ and
$\zeta_t^{\mathrm{dp}}$ lie in $T_{x_t}\M$, so
$P_{x_t}(y_t-x_t)=-\alpha h_t+\zeta_t^{\mathrm{st}}+\zeta_t^{\mathrm{dp}}+P_{x_t}\bar r_t$.
Applying~\eqref{eq:proj-quad-rem} at $x_t$, with $\|y_t-x_t\|\le2\Geff\alpha$
(Lemma~\ref{lem:tube-stability}(c)),
\[
  \textsc{ProjAvg}:\ x_{t+1}-x_t=\Pi(y_t)-x_t=P_{x_t}(y_t-x_t)+w_t^{\mathrm{srv}},\qquad
  \textsc{RLAvg}:\ x_{t+1}-x_t=P_{x_t}(y_t-x_t)+w_t^{\mathrm{srv}},
\]
with $\|w_t^{\mathrm{srv}}\|\le4\Cqr\Geff^2\alpha^2$ in both cases. Set $r_t:=P_{x_t}\bar r_t+w_t^{\mathrm{srv}}$.
\end{proof}

\begin{lemma}[Moments]\label{lem:moments}
  Let $\alpha\le\gamma/(4\Geff)$. Then
  \begin{enumerate}[label=(\alph*),leftmargin=*]
    \item $\EtG[h_t]=\grad F(x_t)+\bar b^C(x_t)$ and
      \begin{equation}\label{eq:ht-sq}
        \EtG\|h_t\|^2\le2\|\grad F(x_t)\|^2+2\beta_C^2
        +\frac{n-k}{k(n-1)}(\sighet+\sigma_{\mathrm{clip}})^2;
      \end{equation}
    \item $\EtG[\zeta_t^{\mathrm{st}}]=0$ and
      $\EtG\|\zeta_t^{\mathrm{st}}\|^2\le\frac{\alpha^2}{k\tau}\bar\rho\,\sigma_{\mathrm{st}}^2$;
    \item $\EtG[\zeta_t^{\mathrm{dp}}]=0$ and
      $\EtG\|\zeta_t^{\mathrm{dp}}\|^2\le\frac{\alpha^2}{k\tau}\dM\sdp^2$;
    \item $\EtG\langle h_t,\zeta_t^{\mathrm{st}}\rangle
      =\EtG\langle h_t,\zeta_t^{\mathrm{dp}}\rangle
      =\EtG\langle\zeta_t^{\mathrm{st}},\zeta_t^{\mathrm{dp}}\rangle=0$.
  \end{enumerate}
\end{lemma}
\begin{proof}
By Section~\ref{ssec:proof-pfail}, $S_t$ is independent of $\mathcal{G}_t$ given
$\mathcal{F}_t$, so it keeps its uniform law under $\EtG$.

(a)~$h_t$ depends on $(x_t,S_t)$ only. Each client is selected with probability $k/n$, so
$\EtG[h_t]=\overline{\gradC f}(x_t)=\grad F(x_t)+\bar b^C(x_t)$. The variance of a mean of
$k$ values drawn without replacement among $n$ is
$\frac{n-k}{k(n-1)}\cdot\frac1n\sum_i\|\gradC f_i(x_t)-\overline{\gradC f}(x_t)\|^2$, which
Lemma~\ref{lem:clip-hetero} bounds. Adding
$\|\overline{\gradC f}(x_t)\|^2\le2\|\grad F(x_t)\|^2+2\beta_C^2$ gives~\eqref{eq:ht-sq}.

(b)~Fix $(i,j)$ and let $\mathcal{A}_{t,j}^{(i)}$ be the $\sigma$-algebra generated by
$\mathcal{F}_t$, $S_t$, all round-$t$ noise draws and all round-$t$ minibatches except
$\mathcal{B}_{t,j}^{(i)}$. Then $\mathcal{G}_t$ and $x_{t,j}^{(i)}$ are
$\mathcal{A}_{t,j}^{(i)}$-measurable, and $\mathcal{B}_{t,j}^{(i)}$ is uniform and
independent of $\mathcal{A}_{t,j}^{(i)}$. Lemma~\ref{lem:clip-var} and
$\opnorm{P_{x_t}}\le1$ give
$\E[\zeta_{t,j}^{(i)}\mid\mathcal{A}_{t,j}^{(i)}]=0$ and
$\E[\|\zeta_{t,j}^{(i)}\|^2\mid\mathcal{A}_{t,j}^{(i)}]\le\rho_i\sigma_{\mathrm{st}}^2$. For
$(i',j')\neq(i,j)$, with $j'<j$ when $i'=i$, $\zeta_{t,j'}^{(i')}$ is
$\mathcal{A}_{t,j}^{(i)}$-measurable, so every cross term has zero conditional mean. Since
$\sigma(\mathcal{F}_t,S_t)\subseteq\mathcal{A}_{t,j}^{(i)}$ and
$\mathcal{G}_t\in\mathcal{A}_{t,j}^{(i)}$, the tower property carries these identities to
conditioning on $(\mathcal{F}_t,\mathcal{G}_t,S_t)$. Hence
\[
  \E\bigl[\|\zeta_t^{\mathrm{st}}\|^2\bigm|\mathcal{F}_t,\mathcal{G}_t,S_t\bigr]
  \le\frac{\eta^2\tau\sigma_{\mathrm{st}}^2}{k^2}\sum_{i\in S_t}\rho_i ,
\]
and averaging over $S_t$, with $\EtG\sum_{i\in S_t}\rho_i=k\bar\rho$ and
$\eta^2\tau/k=\alpha^2/(k\tau)$, gives the bound.

(c)~Given $(\mathcal{F}_t,S_t)$, the draws are i.i.d.\ and $\mathcal{G}_t$ is the
intersection of the events $\{\|\xi_{t,j}^{(i)}\|\le R\}$. Under $\EtG$ and given $S_t$, the
draws are therefore independent, each with the law of $\N(0,\sdp^2\Idp)$ truncated to the
ball of radius $R$. This law is symmetric, so each $P_{x_t}\xi_{t,j}^{(i)}$ has zero mean, and
Lemma~\ref{lem:projected-noise} bounds its second moment by $\dM\sdp^2$. Summing the $k\tau$
independent terms gives
$\EtG\|\zeta_t^{\mathrm{dp}}\|^2\le\frac{\eta^2}{k^2}k\tau\dM\sdp^2=\frac{\alpha^2}{k\tau}\dM\sdp^2$.

(d)~$h_t$ is $(\mathcal{F}_t,S_t)$-measurable and, by (b) and (c), $\zeta_t^{\mathrm{st}}$
and $\zeta_t^{\mathrm{dp}}$ have zero mean given $(\mathcal{F}_t,\mathcal{G}_t,S_t)$. For the
last term, each $P_{x_t}\xi_{t,j'}^{(i')}$ is $\mathcal{A}_{t,j}^{(i)}$-measurable, so
$\E[\langle\zeta_{t,j}^{(i)},P_{x_t}\xi_{t,j'}^{(i')}\rangle\mid\mathcal{A}_{t,j}^{(i)}]=0$.
\end{proof}

\begin{lemma}[One-round descent]\label{lem:descent-good}
  Let $\alpha\le\min\{\gamma/(4\Geff),1/(4\LM)\}$. Then
  \begin{align}\label{eq:descent-good-precise}
    \EtG[F(x_{t+1})]
    &\le F(x_t) - \frac{9\alpha}{16}\|\grad F(x_t)\|^2
       + (2\alpha + \LM\alpha^2)\,\beta_C^2 \notag\\
    &\quad + \frac{\LM\alpha^2}{2}\cdot\frac{n-k}{k(n-1)}(\sighet + \sigma_{\mathrm{clip}})^2
       + \frac{\LM\alpha^2}{2k\tau}\bigl(\bar\rho\,\sigma_{\mathrm{st}}^2 + \dM\sdp^2\bigr)
       + A_3\,\alpha^3,
  \end{align}
  with $A_3 := 5 A_2^2 + 3\LM A_2 \Geff$.
\end{lemma}
\begin{proof}
All quantities are evaluated at $x_t$. Apply the descent inequality~\eqref{eq:descent-ineq}
to $x_t,x_{t+1}\in\M$, take $\EtG$ and insert~\eqref{eq:server-update}.

\emph{Linear term.} By Lemma~\ref{lem:moments}(a)--(c),
$\EtG\langle\grad F,x_{t+1}-x_t\rangle=-\alpha\|\grad F\|^2-\alpha\langle\grad F,\bar
b^C\rangle+\langle\grad F,\EtG r_t\rangle$. Young's inequality and
$\|\bar b^C\|\le\beta_C$, $\|r_t\|\le A_2\alpha^2$ give
$\alpha|\langle\grad F,\bar b^C\rangle|\le\frac\alpha8\|\grad F\|^2+2\alpha\beta_C^2$ and
$\|\grad F\|A_2\alpha^2\le\frac{\alpha}{16}\|\grad F\|^2+4A_2^2\alpha^3$. Hence
\begin{equation}\label{eq:linear-total}
  \EtG\langle\grad F,x_{t+1}-x_t\rangle
  \le-\frac{13\alpha}{16}\|\grad F\|^2+2\alpha\beta_C^2+4A_2^2\alpha^3 .
\end{equation}

\emph{Quadratic term.} By Lemma~\ref{lem:moments}(d),
\[
  \EtG\|x_{t+1}-x_t\|^2=\alpha^2\EtG\|h_t\|^2+\EtG\|\zeta_t^{\mathrm{st}}\|^2
  +\EtG\|\zeta_t^{\mathrm{dp}}\|^2+Q_t,
\]
where $Q_t:=\EtG\bigl[2\langle-\alpha h_t+\zeta_t^{\mathrm{st}}+\zeta_t^{\mathrm{dp}},r_t\rangle
+\|r_t\|^2\bigr]$ collects the terms in $r_t$. On $\mathcal{G}_t$, $\|h_t\|\le C$,
$\|\zeta_t^{\mathrm{st}}\|\le2C\alpha$ since $\|\delta_{t,j}^{(i)}\|\le2C$, and
$\|\zeta_t^{\mathrm{dp}}\|\le R\alpha$, so
$\|\zeta_t^{\mathrm{st}}+\zeta_t^{\mathrm{dp}}\|\le2\Geff\alpha$ and
$|Q_t|\le2A_2(C+2\Geff)\alpha^3+A_2^2\alpha^4$. With Lemma~\ref{lem:moments}(a)--(c) and
$\LM\alpha\le1/4$,
\begin{equation}\label{eq:quad-total}
\begin{split}
  \frac{\LM}{2}\EtG\|x_{t+1}-x_t\|^2
  &\le\frac\alpha4\|\grad F\|^2+\LM\alpha^2\beta_C^2
  +\frac{\LM\alpha^2}{2}\cdot\frac{n-k}{k(n-1)}(\sighet+\sigma_{\mathrm{clip}})^2\\
  &\quad+\frac{\LM\alpha^2}{2k\tau}\bigl(\bar\rho\sigma_{\mathrm{st}}^2+\dM\sdp^2\bigr)
  +\LM A_2(C+2\Geff)\alpha^3+\tfrac18A_2^2\alpha^3 .
\end{split}
\end{equation}

\emph{Sum.} The coefficient of $\|\grad F\|^2$ is $-13/16+1/4=-9/16$. Using $C\le\Geff$, the
$\alpha^3$ terms add up to at most $(5A_2^2+3\LM A_2\Geff)\alpha^3$.
\end{proof}

\subsubsection*{Bad rounds and summation}
Every iterate lies in $\M$ (Section~\ref{app:notation}), so on $\mathcal{G}_t^c$,
$F(x_{t+1})-F(x_t)\le\delta_F$ by~\eqref{eq:Delta-F}. Since
$\Pr(\mathcal{G}_t^c\mid\mathcal{F}_t)=\pi$ (Section~\ref{ssec:proof-pfail}),
Lemma~\ref{lem:descent-good} gives
\begin{equation}\label{eq:combined}
  \E_t[F(x_{t+1})]
  =(1-\pi)\EtG[F(x_{t+1})]+\pi\,\E[F(x_{t+1})\mid\mathcal{F}_t,\mathcal{G}_t^c]
  \le F(x_t)-(1-\pi)\frac{9\alpha}{16}\|\grad F(x_t)\|^2+(1-\pi)\Lambda+\pi\delta_F,
\end{equation}
where $\Lambda$ collects the last four terms of~\eqref{eq:descent-good-precise}. Taking full
expectations, summing over $t=0,\dots,T-1$ and using $F(x_T)\ge F_\star$,
\[
  \frac1T\sum_{t=0}^{T-1}\E\|\grad F(x_t)\|^2
  \le\frac{16(F(x_0)-F_\star)}{9(1-\pi)\alpha T}+\frac{16\Lambda}{9\alpha}
  +\frac{16\pi\delta_F}{9(1-\pi)\alpha}.
\]
Since $\LM\alpha\le1/4$, $\frac{16}{9}(2+\LM\alpha)\le4$, and
\[
  \frac{16\Lambda}{9\alpha}\le4\beta_C^2
  +\frac{8\LM\alpha}{9k}\Bigl(\frac{n-k}{n-1}(\sighet+\sigma_{\mathrm{clip}})^2
  +\frac1\tau\bigl(\bar\rho\sigma_{\mathrm{st}}^2+\dM\sdp^2\bigr)\Bigr)
  +\frac{16A_3}{9}\alpha^2 .
\]
This is~\eqref{eq:finite-T-rate}, with numerical constant $4$ behind $\lesssim$ and
$A:=16A_3/9$. \qed

\subsection{Proof of Corollary~\ref{cor:main}}\label{app:clean-rate}
Take $c_0=4\log T+2\log(k\tau)$ and $\alpha=c_\alpha/\sqrt T$ in~\eqref{eq:finite-T-rate}, with
$T\ge2$.
\begin{itemize}[leftmargin=*]
\item \emph{Bad rounds.} Lemma~\ref{lem:pfail} gives $\pi\le T^{-2}$, so
  $\pi\delta_F/((1-\pi)\alpha)\le\delta_F T^{-3/2}/(c_\alpha(1-T^{-2}))=o(T^{-1/2})$.
\item \emph{Remainder.} $\Geff=C+\sdp(\sqrt D+\sqrt{c_0})=O(\sqrt{\log T})$. Since $A_1$
  and $A_2$ are $O(\Geff^2)$, $A_3=O((\log T)^2)$ and
  $A\alpha^2=O((\log T)^2/T)=o(T^{-1/2})$.
\item \emph{Stepsize condition.} $\alpha=c_\alpha/\sqrt T$ satisfies
  $\alpha\le\min\{\gamma/(4\Geff),1/(4\LM)\}$ once
  $T\ge\max\{16c_\alpha^2\Geff^2/\gamma^2,16c_\alpha^2\LM^2\}$. As $\Geff^2=O(\log T)$, this holds for all
  $T$ beyond some $T_0$.
\end{itemize}
The remaining terms of~\eqref{eq:finite-T-rate} give~\eqref{eq:main-rate}. The constants
hidden in the $O(\cdot)$ depend on $F(x_0)-F_\star$, $\delta_F$, $\ell_f$, $\LM$, $\Ldp$,
$\Cqr$, $\gamma$, $C$, $G$, $\sdp$, $D$, $k\tau$ and $c_\alpha$, but not on $T$. The ambient
dimension $D$ enters only through $\Geff$, in the stepsize condition and the subdominant
remainder.

\section{Details of the synthetic experiments}\label{app:synthetic}

Code: \codeurl.

\paragraph{Data.}\label{app:synth-data} Client $i$ draws its records from~\eqref{eq:spiked}, then projects them
onto the unit ball, $z\leftarrow z/\max(1,\|z\|)$. All clients share a basis $U$, the $Q$
factor of the QR decomposition of a $d\times p$ matrix with independent standard Gaussian
entries. Client $i$ perturbs it: $U_i$ is the $Q$ factor of $U+0.3\,G_i/\sqrt{d-p}$, with
$G_i$ a $d\times p$ matrix with independent standard Gaussian entries, drawn independently
across clients. The signal variances are $\lambda_j=0.06\,w_j/\bar w$ for
$j=1,\dots,p$, with $w_j=1+(p-j)/(2p)$ and $\bar w$ their mean, so they average $0.06$
whatever $p$. The noise energy is $\nu^2=0.3$ whatever $d$.

\paragraph{Closed-form solution.}\label{app:synth-solution} Let $\bar\lambda_1\ge\dots\ge\bar\lambda_d$ be the
eigenvalues of $\bar\Sigma$. Then $F(P)=-\frac12\trace(P\bar\Sigma)$. By Ky Fan's
maximum principle~\citep{fan1949theorem}, $\trace(P\bar\Sigma)\le\sum_{j\le p}\bar\lambda_j$
for every rank-$p$ orthogonal projector $P$, with equality for the projector onto the
top-$p$ eigenvectors of $\bar\Sigma$. This projector is the minimiser $P^\star$, and
$F^\star=-\frac12\sum_{j\le p}\bar\lambda_j$. 

\paragraph{Clip constant and stepsize.}\label{app:synth-tuning} The Riemannian gradient of $f(\cdot\,;z)$ at $P$ is
$-\frac12\bigl(Pz(z-Pz)^\top+(z-Pz)(Pz)^\top\bigr)$. Its norm is at most $1/(2\sqrt2)$ for
$\|z\|\le1$. We therefore set $C=1/(2\sqrt2)$, the smallest clip constant for which
clipping is inactive. The stepsize of each method is selected at each $\varepsilon$ over
$\eta\in\{2^{-3},\dots,2^3\}/\tau$, on the mean optimality gap of 3 runs disjoint from the
10 reported ones. Tuning is not charged to the privacy budget. \textsc{DP-RFedProj}
selects $\eta=0.2$ at every $\varepsilon$ with each aggregation scheme, and
Figure~\ref{fig:synth}(a--c) uses this value. PriRFed selects $\eta=0.05$ at
$\varepsilon=0.5$, $\eta=0.1$ at $\varepsilon=1$, and $\eta=0.2$ above. \textsc{RetAvg} uses the projection-based retraction $\mathrm{R}_P(X)=\Pi(P+X)$, whose inverse is a linear solve of size $p(d-p)$.

\paragraph{Values.}\label{app:synth-values} Tables~\ref{tab:synth-cost} and~\ref{tab:synth-gaps} give the values of
Figure~\ref{fig:synth}. In all configurations, the median cost of privacy is between
$0.23$ and $0.27$ times $\eta\sdp^2\dM/k$, for both aggregation schemes.

\begin{table}[h]
\centering
\caption{Median cost of privacy at $\varepsilon=1$ over 10 seeds, in units of $10^{-4}$.}
\label{tab:synth-cost}
\small
\begin{tabular}{c rr rrr c cc}
\toprule
Panel & $d$ & $p$ & $\dM$ & $\dimE$ & $k$ & $\eta\sdp^2\dM/k$ & \textsc{ProjAvg} & \textsc{RLAvg} \\
\midrule
(a) & 40 & 1 & 39 & 820 & 10 & 4.57 & 1.24 & 1.23 \\
    & 40 & 2 & 76 & 820 & 10 & 8.91 & 2.43 & 2.42 \\
    & 40 & 4 & 144 & 820 & 10 & 16.9 & 4.18 & 4.16 \\
    & 40 & 8 & 256 & 820 & 10 & 30.0 & 7.56 & 7.52 \\
\midrule
(b) & 26 & 8 & 144 & 351 & 10 & 16.9 & 4.02 & 4.01 \\
    & 40 & 4 & 144 & 820 & 10 & 16.9 & 4.18 & 4.16 \\
    & 74 & 2 & 144 & 2775 & 10 & 16.9 & 4.20 & 4.16 \\
    & 145 & 1 & 144 & 10585 & 10 & 16.9 & 3.94 & 3.87 \\
\midrule
(c) & 40 & 4 & 144 & 820 & 1 & 169 & 45.4 & 44.9 \\
    & 40 & 4 & 144 & 820 & 2 & 84.5 & 21.2 & 21.0 \\
    & 40 & 4 & 144 & 820 & 5 & 33.8 & 8.39 & 8.34 \\
    & 40 & 4 & 144 & 820 & 10 & 16.9 & 4.18 & 4.16 \\
    & 40 & 4 & 144 & 820 & 25 & 6.76 & 1.71 & 1.70 \\
    & 40 & 4 & 144 & 820 & 50 & 3.38 & 0.85 & 0.84 \\
\bottomrule
\end{tabular}
\end{table}

\begin{table}[h]
\centering
\caption{Noise level and median optimality gap over 10 seeds of Figure~\ref{fig:synth}(d).
\textsc{ProjAvg}, \textsc{RLAvg} and \textsc{RetAvg} use our accounting and $\eta=0.2$.
PriRFed uses \textsc{RetAvg}, its round-level accounting and its tuned $\eta$.}
\label{tab:synth-gaps}
\small
\setlength{\tabcolsep}{4pt}
\begin{tabular}{r cc ccc cc}
\toprule
& \multicolumn{2}{c}{$\sdp$} & \multicolumn{3}{c}{our accounting} & \multicolumn{2}{c}{PriRFed} \\
\cmidrule(lr){2-3}\cmidrule(lr){4-6}\cmidrule(lr){7-8}
$\varepsilon$ & ours & PriRFed & \textsc{ProjAvg} & \textsc{RLAvg} & \textsc{RetAvg} & $\eta$ & gap \\
\midrule
0.5 & 0.0452 & 0.247 & $1.5\cdot10^{-3}$ & $1.5\cdot10^{-3}$ & $1.5\cdot10^{-3}$ & 0.05 & $3.1\cdot10^{-2}$ \\
1   & 0.0242 & 0.131 & $4.4\cdot10^{-4}$ & $4.4\cdot10^{-4}$ & $4.4\cdot10^{-4}$ & 0.1 & $7.7\cdot10^{-3}$ \\
2   & 0.0132 & 0.0710 & $1.5\cdot10^{-4}$ & $1.5\cdot10^{-4}$ & $1.5\cdot10^{-4}$ & 0.2 & $3.9\cdot10^{-3}$ \\
5   & 0.00617 & 0.0328 & $4.9\cdot10^{-5}$ & $4.9\cdot10^{-5}$ & $4.9\cdot10^{-5}$ & 0.2 & $8.0\cdot10^{-4}$ \\
10  & 0.00368 & 0.0193 & $3.0\cdot10^{-5}$ & $3.0\cdot10^{-5}$ & $3.0\cdot10^{-5}$ & 0.2 & $2.9\cdot10^{-4}$ \\
\bottomrule
\end{tabular}
\end{table}

\section{Details of the EEG experiments}\label{app:experiments}

Code, with the command behind each table and figure: \codeurl.
\subsection{Architectures}
    \EEGNetFigureWithNotes   
    \SPDNetFigureWithNotes 
\subsection{The batch-normalization problem}\label{app:bn-pb}
Lemma~\ref{lem:subsamp-step} bounds the sensitivity by $2C/B$ because each per-sample
gradient depends on its own record only, so substituting one record changes one term of
the minibatch average. In training mode, batch normalization uses the minibatch mean and
variance, so every per-sample gradient depends on all $B$ records. Substituting one record
then changes all $B$ terms, and the sensitivity grows to
$2C$~\citep[Sec.~5.5.1]{ponomareva2023dpfy}. We therefore use EEGNet
without batch normalization in every run, so that the guarantee covers all
training records. This choice also puts the two architectures on an
equal footing, since SPDNet uses no batch normalization either. 

\subsection{Hyperparameter selection}\label{app:grids}

Hyperparameters are selected on the validation accuracy of the last iterate, averaged over
three tuning seeds disjoint from the ten reported ones. For SPDNet with \textsc{ProjAvg} and
for EEGNet, we select $(\eta_0, B)$ on logarithmic grids at fixed $T$, then the horizon
$T$ as the smallest one within one standard error of the best. We do this centrally, then in
the federated setting without privacy, then again for $(\eta_0, B)$ at each privacy budget,
with $T$ kept at its non-private value. A grid is extended when its selected value lies on
its edge with a clear margin. \textsc{RLAvg} inherits the settings of \textsc{ProjAvg}. We
screen \textsc{RetAvg} separately: it keeps the batch size of \textsc{ProjAvg} and gets its
own learning rate, the best among those for which every tuning seed completes. Half
participation and the horizon ablation reuse the settings of full participation.
Tables~\ref{tab:config-central} and~\ref{tab:config-fed} list the selected configurations.

\begin{table}[h]
\centering
\caption{Selected centralized configurations: learning rate $\eta_0$, batch size $B$ and horizon $T$ in epochs.}
\label{tab:config-central}
\small
\begin{tabular}{llrrr}
\toprule
Dataset & Model & $\eta_0$ & $B$ & $T$ \\
\midrule
\multirow{3}{*}{BNCI2014\_001} & SPDNet & $10^{-2}$ & 64 & 200 \\
 & EEGNet  & $10^{-3}$ & 16 & 300 \\
\midrule
\multirow{3}{*}{Cho2017} & SPDNet & $3\cdot 10^{-2}$ & 16 & 300 \\
 & EEGNet  & $10^{-3}$ & 64 & 300 \\
\midrule
\multirow{3}{*}{Weibo2014} & SPDNet & $10^{-3}$ & 32 & 300 \\
 & EEGNet & $3\cdot 10^{-3}$ & 128 & 300 \\
\bottomrule
\end{tabular}
\end{table}

\begin{table}[h]
\centering
\caption{Selected federated configurations, as ($\eta_0$, $B$), for each privacy level. $T$ is the number of rounds, shared by all levels. \textsc{RetAvg} uses the batch size of \textsc{ProjAvg} and its own learning rate.}
\label{tab:config-fed}
\small
\begin{adjustbox}{max width=\linewidth}
\begin{tabular}{lllcccc}
\toprule
Dataset & Arm & $T$ & Non-private & $\varepsilon = 20$ & $\varepsilon = 10$ & $\varepsilon = 5$ \\
\midrule
\multirow{3}{*}{BNCI2014\_001} & SPDNet \textsc{ProjAvg}, \textsc{RLAvg} & 200 & $10^{-2}$, 16 & $10^{-1}$, 256 & $3\cdot 10^{-2}$, 32 & $3\cdot 10^{-2}$, 256 \\
 & SPDNet \textsc{RetAvg} & 200 & $6\cdot 10^{-3}$, 16 & $6\cdot 10^{-2}$, 256 & $2\cdot 10^{-2}$, 32 & $2\cdot 10^{-2}$, 256 \\
 & EEGNet & 450 & $3\cdot 10^{-3}$, 32 & $3\cdot 10^{-3}$, 16 & $3\cdot 10^{-3}$, 16 & $10^{-2}$, 32 \\
\midrule
\multirow{3}{*}{Cho2017} & SPDNet \textsc{ProjAvg}, \textsc{RLAvg} & 600 & $3\cdot 10^{-3}$, 16 & $2\cdot 10^{-1}$, 64 & $2\cdot 10^{-1}$, 256 & $10^{-1}$, 128 \\
 & SPDNet \textsc{RetAvg} & 600 & $3\cdot 10^{-3}$, 16 & $6\cdot 10^{-2}$, 64 & $2\cdot 10^{-1}$, 256 & $10^{-1}$, 128 \\
 & EEGNet & 450 & $10^{-2}$, 64 & $10^{-3}$, 128 & $3\cdot 10^{-2}$, 128 & $3\cdot 10^{-2}$, 128 \\
\midrule
\multirow{3}{*}{Weibo2014} & SPDNet \textsc{ProjAvg}, \textsc{RLAvg} & 450 & $3\cdot 10^{-3}$, 32 & $3\cdot 10^{-2}$, 32 & $3\cdot 10^{-2}$, 32 & $3\cdot 10^{-2}$, 256 \\
 & SPDNet \textsc{RetAvg} & 450 & $3\cdot 10^{-3}$, 32 & $1.5\cdot 10^{-2}$, 32 & $10^{-2}$, 32 & $2\cdot 10^{-2}$, 256 \\
 & EEGNet & 600 & $10^{-2}$, 64 & $10^{-2}$, 64 & $3\cdot 10^{-3}$, 64 & $10^{-2}$, 64 \\
\bottomrule
\end{tabular}
\end{adjustbox}
\end{table}

\begin{table}[h]
\centering
\caption{Number of clients, smallest and largest client training sizes over the 10 seeds,
and number of parameters of each model.}
\label{tab:fed-setting}
\small
\begin{tabular}{lrrrrr}
\toprule
& & & & \multicolumn{2}{c}{Parameters} \\
\cmidrule(lr){5-6}
Dataset & $n$ & $m_\star$ & $m_{\max}$ & EEGNet & SPDNet \\
\midrule
BNCI2014\_001 & 3  & $1{,}257$ & $1{,}326$ & $2{,}404$ & $900$ \\
Cho2017       & 26 & $277$     & $358$     & $2{,}434$ & $1{,}298$ \\
Weibo2014     & 5  & $762$     & $873$     & $3{,}783$ & $2{,}677$ \\
\bottomrule
\end{tabular}
\end{table}

\subsection{Learning curves and horizon}\label{app:T-ablation}
Figure~\ref{fig:curves} shows the learning curves of every budget at the selected horizon.
For the horizon ablation, we halve and double $T$ at each budget (rounded to a multiple of $10$ rounds) on BNCI2014\_001 and Weibo2014, at the same $(\eta_0,B)$ and with $\sdp$ recalibrated (Table~\ref{tab:T-ablation} and Figure \ref{fig:T-ablation-eps5}). Doubling $T$ changes the accuracy of SPDNet by $-0.9$ to $+1.4$ points, and halving it by $-2.0$ to $+0.8$.

\begin{table}[h]
\centering
\caption{Horizon under privacy, $\delta = 10^{-5}$, test accuracy (\%) over 10 seeds. SPDNet uses \textsc{ProjAvg}. The middle column is the horizon $T_{\mathrm{fed}}$ calibrated without noise and used in every private run. The other two halve and double it at the learning rate and batch size tuned for that budget, with the noise $\sigma$ recalibrated. $\sigma$ is the noise on the averaged minibatch gradient, so it also depends on the batch size, which differs across budgets.}
\label{tab:T-ablation}
\small
\begin{adjustbox}{max width=\linewidth}
\begin{tabular}{lllllll}
\toprule
Dataset & Model & $\varepsilon$ & $T$ & $T_{\mathrm{fed}}/2$ & $T_{\mathrm{fed}}$ & $2\,T_{\mathrm{fed}}$ \\
\midrule
\multirow{6}{*}{\shortstack[l]{BNCI2014\_001\\\footnotesize chance $25.0$}} & SPDNet & $20$ & 100/200/400 & $41.3 \pm 2.1$ \scriptsize($\sigma=0.013$) & $43.3 \pm 1.9$ \scriptsize($\sigma=0.018$) & $\mathbf{43.6 \pm 2.4}$ \scriptsize($\sigma=0.025$) \\
 &  & $10$ & 100/200/400 & $\mathbf{39.7 \pm 1.8}$ \scriptsize($\sigma=0.075$) & $39.0 \pm 2.2$ \scriptsize($\sigma=0.097$) & $38.5 \pm 2.1$ \scriptsize($\sigma=0.128$) \\
 &  & $5$ & 100/200/400 & $34.6 \pm 2.4$ \scriptsize($\sigma=0.037$) & $35.7 \pm 2.0$ \scriptsize($\sigma=0.051$) & $\mathbf{37.1 \pm 2.7}$ \scriptsize($\sigma=0.071$) \\
\cmidrule(lr){2-7}
 & EEGNet  & $20$ & 220/450/900 & $\mathbf{30.3 \pm 3.0}$ \scriptsize($\sigma=0.111$) & $30.0 \pm 2.2$ \scriptsize($\sigma=0.139$) & $29.9 \pm 1.5$ \scriptsize($\sigma=0.178$) \\
 &  & $10$ & 220/450/900 & $27.8 \pm 2.1$ \scriptsize($\sigma=0.158$) & $28.6 \pm 1.2$ \scriptsize($\sigma=0.210$) & $\mathbf{29.0 \pm 1.8}$ \scriptsize($\sigma=0.295$) \\
 &  & $5$ & 220/450/900 & $27.0 \pm 1.9$ \scriptsize($\sigma=0.164$) & $\mathbf{27.4 \pm 2.1}$ \scriptsize($\sigma=0.226$) & $27.3 \pm 1.4$ \scriptsize($\sigma=0.310$) \\
\midrule
\multirow{6}{*}{\shortstack[l]{Weibo2014\\\footnotesize chance $14.3$}} & SPDNet & $20$ & 220/450/900 & $23.0 \pm 1.2$ \scriptsize($\sigma=0.081$) & $\mathbf{23.6 \pm 2.0}$ \scriptsize($\sigma=0.106$) & $22.7 \pm 1.7$ \scriptsize($\sigma=0.140$) \\
 &  & $10$ & 220/450/900 & $20.4 \pm 0.8$ \scriptsize($\sigma=0.123$) & $\mathbf{20.8 \pm 1.7}$ \scriptsize($\sigma=0.167$) & $20.0 \pm 1.5$ \scriptsize($\sigma=0.226$) \\
 &  & $5$ & 220/450/900 & $\mathbf{18.3 \pm 1.5}$ \scriptsize($\sigma=0.069$) & $17.5 \pm 2.0$ \scriptsize($\sigma=0.098$) & $17.5 \pm 2.2$ \scriptsize($\sigma=0.138$) \\
\cmidrule(lr){2-7}
 & EEGNet  & $20$ & 300/600/1200 & $\mathbf{16.3 \pm 1.6}$ \scriptsize($\sigma=0.059$) & $15.6 \pm 1.6$ \scriptsize($\sigma=0.079$) & $15.3 \pm 1.6$ \scriptsize($\sigma=0.107$) \\
 &  & $10$ & 300/600/1200 & $14.1 \pm 0.6$ \scriptsize($\sigma=0.094$) & $14.8 \pm 0.9$ \scriptsize($\sigma=0.128$) & $\mathbf{15.7 \pm 1.5}$ \scriptsize($\sigma=0.177$) \\
 &  & $5$ & 300/600/1200 & $\mathbf{14.9 \pm 1.2}$ \scriptsize($\sigma=0.159$) & $14.7 \pm 1.1$ \scriptsize($\sigma=0.220$) & $14.6 \pm 1.6$ \scriptsize($\sigma=0.307$) \\
\bottomrule
\end{tabular}
\end{adjustbox}
\end{table}

\begin{figure}[h]
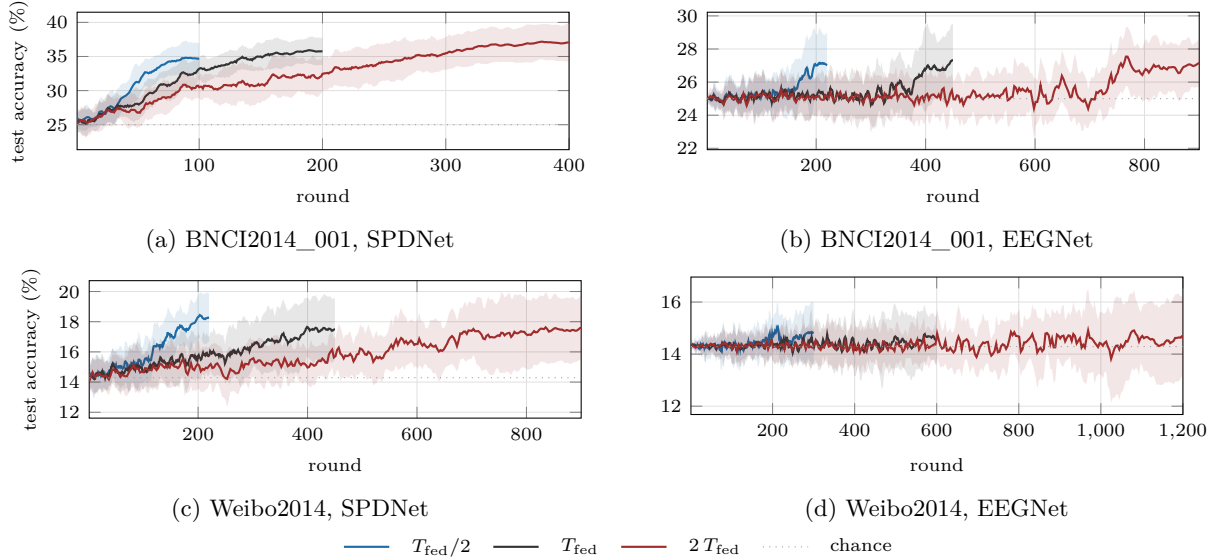

\centering
{\pgfplotsset{every axis post/.append style={height=3.4cm, xlabel={round},
  tick label style={font=\scriptsize}, label style={font=\scriptsize}}}%
\begin{subfigure}[t]{0.49\linewidth}\centering
  \pgfplotsset{every axis post/.append style={legend to name=ablleg-a}}%
  \input{figs_ablation/ablation_BNCI2014_001_spdnet_eps5}
  \caption{BNCI2014\_001, SPDNet}\end{subfigure}\hfill
\begin{subfigure}[t]{0.49\linewidth}\centering
  \pgfplotsset{every axis post/.append style={ylabel={}, legend to name=ablleg-b}}%
  \input{figs_ablation/ablation_BNCI2014_001_eegnet_nobn_eps5}
  \caption{BNCI2014\_001, EEGNet }\end{subfigure}

\begin{subfigure}[t]{0.49\linewidth}\centering
  \pgfplotsset{every axis post/.append style={legend to name=ablleg-c}}%
  \input{figs_ablation/ablation_Weibo2014_spdnet_eps5}
  \caption{Weibo2014, SPDNet}\end{subfigure}\hfill
\begin{subfigure}[t]{0.49\linewidth}\centering
  \pgfplotsset{every axis post/.append style={ylabel={}, legend to name=ablleg-d}}%
  \input{figs_ablation/ablation_Weibo2014_eegnet_nobn_eps5}
  \caption{Weibo2014, EEGNet }\end{subfigure}}

\begin{tikzpicture}
  \begin{axis}[hide axis, width=2cm, height=2cm, xmin=0, xmax=1, ymin=0, ymax=1,
    legend columns=4, legend cell align=left,
    legend style={font=\scriptsize, draw=none, fill=none, column sep=1.4ex}]
    \addlegendimage{C0, thick}\addlegendentry{$T_{\mathrm{fed}}/2$}
    \addlegendimage{black!80, thick}\addlegendentry{$T_{\mathrm{fed}}$}
    \addlegendimage{C3, thick}\addlegendentry{$2\,T_{\mathrm{fed}}$}
    \addlegendimage{black!40, dotted}\addlegendentry{chance}
  \end{axis}
\end{tikzpicture}
\caption{Test accuracy against communication round at $\varepsilon = 5$, at the horizon
$T_{\mathrm{fed}}$ used in every private run and at half and twice that horizon. Mean over
10 seeds, $\pm 1$ standard deviation, smoothed over 5 rounds.}
\label{fig:T-ablation-eps5}
\end{figure}

\subsection{Aggregation cost}\label{app:aggtime}
Table~\ref{tab:aggtime} gives the aggregation times,
measured on the nodes of a shared cluster, with AMD EPYC (Zen~2 to Zen~4) and Intel Xeon
(Skylake-SP to Emerald Rapids) processors and $4$ cores per job. The runs of every rule
were spread over the same pool, so the ratios between rules are comparable, while absolute
times depend on the hardware.
\begin{table}[h]
\centering
\caption{Aggregation time per round (ms) of the Riemannian rules: median and interquartile range over all rounds of the 10 non-private runs, with the number of clients in parentheses. The last column is the ratio of the medians of \textsc{RetAvg} and \textsc{ProjAvg}.}
\label{tab:aggtime}
\footnotesize
\setlength{\tabcolsep}{4pt}
\begin{tabular}{lrrrr}
\toprule
Dataset & \textsc{ProjAvg} & \textsc{RLAvg} & \textsc{RetAvg} & \textsc{RetAvg}/\textsc{ProjAvg} \\
\midrule
BNCI2014\_001 (3) & $0.50$ \scriptsize$[0.32, 0.58]$ & $0.71$ \scriptsize$[0.62, 0.84]$ & $4.47$ \scriptsize$[3.74, 4.98]$ & $9\times$ \\
Weibo2014 (5) & $0.67$ \scriptsize$[0.57, 0.84]$ & $1.39$ \scriptsize$[1.09, 1.63]$ & $10.8$ \scriptsize$[9.77, 12.0]$ & $16\times$ \\
Cho2017 (26) & $1.61$ \scriptsize$[1.38, 2.05]$ & $4.87$ \scriptsize$[2.08, 6.04]$ & $29.3$ \scriptsize$[24.0, 37.8]$ & $18\times$ \\
\bottomrule
\end{tabular}
\end{table}

\subsection{Privacy accounting}\label{app:dp-settings}

The accounting charges the subsampling rate $q = B/m_\star$ of the smallest client and the
$\tau = \lceil m_{\max}/B \rceil$ local steps of the largest, over $T\tau$ steps. Under
privacy, client $i$ runs $\lfloor m_i/B\rfloor\le\tau$ steps on fresh minibatches drawn
without replacement, so $\tau$ is an upper bound. For each target $(\varepsilon,\delta)$,
$\sdp$ is the smallest noise scale whose certified $\varepsilon$ for
$C_q(G_{\mu_0})^{\otimes T\tau}$ meets the target. It is calibrated for each seed, since
$m_\star$ and $m_{\max}$ depend on the split.

PriRFed accounting (Figure~\ref{fig:prirfed-acc}) holds the mechanism $(C, B, q, \tau, T)$ fixed and changes only the
outer accounting. PriRFed~\citep[Thm.~7]{huang2026federated} turns each round into a pair
$(\varepsilon_r, \delta_r)$ and composes the $T$ rounds, optimized over the split $\delta = \hat\delta + T\delta_r$. We feed it our numerically certified per-round privacy profile
$C_q(G_{\mu_0})^{\otimes\tau}$ in place of its published closed form. 
PriRFed also amplifies each round by client sampling~\citep[Lem.~1]{huang2026federated}.
We do not amplify by client subsampling, since the server draws and records $S_t$
(Remark~\ref{rem:hidden-indices}).  PriRFed reuses the learning rate and batch size selected
for our \textsc{RetAvg} at the same budget.

\end{document}